\documentclass[nohyperref,dvipsnames]{article}

\usepackage[accepted]{icml2023}

\usepackage[utf8]{inputenc} 
\usepackage[T1]{fontenc}    
\usepackage{hyperref}       
\usepackage{url}            
\usepackage{booktabs}       
\usepackage{amsfonts}       
\usepackage{nicefrac}       
\usepackage{microtype}      
\usepackage{wrapfig,lipsum}
\usepackage{amsthm,mathrsfs,amsmath}
\usepackage{amssymb}
\usepackage{paralist}
\usepackage{tcolorbox}
\usepackage{pifont}
\usepackage{sectsty}
\usepackage{xargs}
\usepackage{dsfont}
\usepackage{caption}
\usepackage{subcaption} 
\usepackage{upgreek}
\usepackage{mathrsfs}
\usepackage{aliascnt}
\usepackage[noend]{algpseudocode}  
\usepackage{algorithm}
\usepackage{algorithmicx}
\usepackage{bbm}
\usepackage{bm}
\usepackage{stmaryrd}  
\usepackage{easyeqn}
\usepackage[capitalize,noabbrev]{cleveref}
\usepackage[textsize=tiny]{todonotes}
\usepackage{xcolor} 
\usepackage{xargs}  
\usepackage{enumitem} 

\icmltitlerunning{Efficient Conformal Prediction under Data Heterogeneity}

\hypersetup{colorlinks,citecolor = blue}
\graphicspath{{imgs/}}

\theoremstyle{plain}
\newtheorem{theorem}{Theorem}[section]

\newtheorem{lemma}[theorem]{Lemma}

\theoremstyle{definition}

\theoremstyle{remark}

\newtheorem{assumption}{\textbf{H}\hspace{-3pt}}
\crefformat{assumption}{{\textbf{H}}#2#1#3}

\newcommand{\group}[1]{\bgroup\small\sffamily\noindent\color{green}{#1}\egroup}  
\newcommand{\new}[1]{\bgroup\small\sffamily\noindent\color{cobalt}{#1}\egroup}
\newcommand{\old}[1]{\bgroup\small\color{gray}{#1}\egroup}  

\definecolor{darkgreen}{RGB}{0,128,0}
\definecolor{burntumber}{rgb}{0.54, 0.2, 0.14}
\definecolor{aurometalsaurus}{rgb}{0.43, 0.5, 0.5}

\hypersetup{colorlinks = true, urlcolor = blue, bookmarksopen = true}

\newcommand{\R}{\mathbb{R}}

\newcommand{\E}{\mathbb{E}}

\newcommand{\argmin}{\operatornamewithlimits{\arg\min}}

\newcommand{\Oh}{\operatorname{\mathrm{O}}}

\newcommand{\var}{\operatorname{Var}}

\newcommand{\prob}{\mathbb{P}}

\newcommand{\rme}{\mathrm{e}}
\newcommand{\rmd}{\mathrm{d}}
\newcommand{\1}{\mathds{1}}

\newcommand{\pr}[1]{\left({#1}\right)}

\newcommand{\prn}[1]{({#1})}
\newcommand{\prbig}[1]{\big({#1}\big)}

\newcommand{\br}[1]{\left[{#1}\right]}

\newcommand{\brn}[1]{[{#1}]}
\newcommand{\brbig}[1]{\big[{#1}\big]}

\newcommand{\ac}[1]{\left\{{#1}\right\}}

\newcommand{\acn}[1]{\{{#1}\}}

\newcommand{\norm}[1]{\left\|{#1}\right\|}

\newcommand{\normn}[1]{\|{#1}\|}

\newcommand{\abs}[1]{\left\lvert{#1}\right\rvert}

\newcommand{\absn}[1]{|{#1}|}
\newcommand{\absbig}[1]{\big|{#1}\big|}

\newcommand{\nofrac}[2]{{#1}/{#2}}
\newcommand{\gauss}{\mathcal{N}}

\newcommand{\q}[1]{Q_{#1}}

\renewcommand{\algorithmiccomment}[1]{\textbf{\textcolor{darkgreen}{// \texttt{#1}}}}

\newcommand{\tcount}{N}

\newcommand{\ccount}[1]{\tcount^{#1}}
\newcommand{\clcount}[2]{\tcount_{#1}^{#2}}

\newcommand{\nclients}{n}

\newcommand{\target}{\star}

\newcommand{\iweight}[1]{w_{#1}}
\newcommand{\iweightp}[1]{\lambda(#1)}

\newcommand{\qweightb}[2]{\Bar{p}_{#1, #2}}

\newcommand{\qdistrb}{\bar{\mu}}

\newcommand{\Pcal}[1]{P^{#1}}

\newcommandx{\predset}[1]{\mathcal{C}_{#1}}

\newcommandx{\pinballmoreau}[4][1=\alpha,2=\gamma]{S_{#1,#3}^{#2}(#4)}

\newcommand{\xquery}{\mathbf{x}}

\newcommand{\xv}{\mathbf{x}}
\newcommand{\yv}{\mathbf{y}}

\newcommand{\XC}{\mathcal{X}}
\newcommand{\YC}{\mathcal{Y}}
\newcommand{\ZC}{\mathcal{Z}}
\newcommand{\CC}{\mathcal{C}}
\newcommand{\DC}{\mathcal{D}}

\newcommandx{\indi}[2][1=]{\1^{#1}_{#2}}
\newcommand{\indiacc}[1]{\1_{\{#1\}}}

\def\iid{i.i.d.}

\newcommand{\RR}{\mathbb{R}}

\newcommand{\format}[1]{\paragraph{#1}}

\begin{document}


\twocolumn[
\icmltitle{Efficient Conformal Prediction under Data Heterogeneity}

\begin{icmlauthorlist}
    \icmlauthor{Vincent Plassier}{lagrange,ep}
    \icmlauthor{Nikita Kotelevskii}{skt,mbzuai}
    \icmlauthor{Aleksandr Rubashevskii}{skt,mbzuai}
    \icmlauthor{Fedor Noskov}{hse}
    \icmlauthor{Maksim Velikanov}{ep,tii}
    \icmlauthor{Alexander Fishkov}{skt}
    \icmlauthor{Samuel Horvath}{mbzuai}
    \icmlauthor{Martin Tak\'a\v{c}}{mbzuai}
    \icmlauthor{Eric Moulines}{ep,mbzuai}
    \icmlauthor{Maxim Panov}{mbzuai}
\end{icmlauthorlist}

\icmlaffiliation{ep}{CMAP, Ecole Polytechnique, Paris, France}
\icmlaffiliation{tii}{Technology Innovation Institute, Abu Dhabi, UAE}
\icmlaffiliation{skt}{Skolkovo Institute of Science and Technology, Moscow, Russia}
\icmlaffiliation{hse}{HSE University, Moscow, Russia}
\icmlaffiliation{lagrange}{Lagrange Mathematics and Computing Research Center, Paris, France}
\icmlaffiliation{mbzuai}{Mohamed bin Zayed University of Artificial Intelligence, Abu Dhabi, UAE}

\icmlcorrespondingauthor{Vincent Plassier}{vincent.plassier@polytechnique.edu}
\icmlcorrespondingauthor{Maxim Panov}{maxim.panov@mbzuai.ac.ae}

\icmlkeywords{Machine Learning}

\vskip 0.3in
]

\printAffiliationsAndNotice{}  

\begin{abstract}
  Conformal Prediction (CP) stands out as a robust framework for uncertainty quantification, which is crucial for ensuring the reliability of predictions. 
  However, common CP methods heavily rely on the data exchangeability, a condition often violated in practice. Existing approaches for tackling non-exchangeability lead to methods that are not computable beyond the simplest examples. In this work, we introduce a new efficient approach to CP that produces provably valid confidence sets for fairly general non-exchangeable data distributions. 
  We illustrate the general theory with applications to the challenging setting of federated learning under data heterogeneity between agents. Our method allows constructing provably valid personalized prediction sets for agents in a fully federated way. The effectiveness of the proposed method is demonstrated in a series of experiments on real-world datasets.
\end{abstract}

\addtocontents{toc}{\protect\setcounter{tocdepth}{0}}

\section{Introduction}
\label{sec:introduction}

Conformal Prediction (CP) has been shown to be a reliable method for quantifying uncertainty in machine learning models \citep{shafer2008tutorial,angelopoulos2021gentle}. However, the performance of prediction sets generated by CP can significantly decrease when the data is statistically heterogeneous. In particular, the presence of distributional shifts interferes with the assumption of exchangeability -- an essential cornerstone of the methods of CP. The goal of this work is to develop efficient solutions that extend the CP framework to data with heterogeneous distributions.
 

Assume that we have $\tcount$ data samples $\DC_\tcount = \{X_i, Y_i\}_{i=1}^{\tcount}$ with $X_i \in \XC$ and $Y_i \in \mathcal{Y}$.
Given the data $\DC_\tcount$, the goal of CP uncertainty quantification is to construct prediction sets $\CC_{\alpha}(\xv)\subseteq \mathcal{Y}$ for a new unseen object $\xv$ and the desired miscoverage level $\alpha \in [0, 1]$. We say that $\mathcal{C}_{\alpha}$ is a valid (distribution-free) predictive interval if 
it holds that
\begin{equation}\label{eq:average-miscoverage}
  \prob \bigl(Y_{\tcount + 1} \in \CC_\alpha (X_{\tcount + 1})\bigr) \geq 1-\alpha .
\end{equation}
Here the notation $\prob$ denotes that the probability is computed with respect to $\{(X_i, Y_i) \}_{i=1}^N$ and $(X_{N+1}, Y_{N+1})$, which are assumed \iid\ with common distribution $P$. 
The equation~\eqref{eq:average-miscoverage} essentially bounds the miscoverage rate on average over possible sets of calibration data and test points. However, if there is a high variability in the coverage probability as a function of the calibration data, the test coverage probability may be substantially below $1 - \alpha$ for a particular calibration set. Therefore, for a given calibration set $\DC_\tcount$ the reliability of the confidence set can be assessed via the \textit{empirical miscoverage rate}:
\begin{equation}\label{eq:def:miscoverage-rate}
  \alpha(\DC_{\tcount}) 
  = \prob\pr{Y_{\tcount+1} \notin \CC_{\alpha}(X_{\tcount + 1})\,\vert\, \DC_{\tcount}}.
\end{equation}
Understanding the distribution of $\alpha(\mathcal{D}_{\tcount})$ allows us to control the uncertainty.
Its distribution was studied for split conformal prediction (SCP) in~\citep[Proposition 2a]{vovk2012conditional}), where it has been shown that:
\begin{equation}\label{eq:sim:alphaDN}
  \alpha(\mathcal{D}_{\tcount})
  \sim \text{Beta}(\lceil(\tcount+1)\alpha\rceil, \lceil(\tcount+1)(1-\alpha)\rceil),
\end{equation}
therefore ensuring the concentration of $\alpha(\mathcal{D}_{\tcount})$ around the target value: $\alpha(\mathcal{D}_{\tcount}) \simeq \alpha + \Oh_{\prob}(1/\sqrt{\tcount})$, see also \citep{angelopoulos2021gentle}.

While classical CP methods (such as SCP) work well under data exchangeability, there are few established performance guarantees for non-exchangeable data. Among the notable exceptions, \citet{tibshirani2019conformal} in their seminal work developed methods that address the important case of \emph{weighted exchangeable data}.
However, their approach involves calculating certain weighted quantiles with weights that require difficult combinatorial calculations that are not feasible in practice beyond relatively simple examples. Moreover, their approach provides only the  guarantees in expectation;  the variability of coverage around its mean has not been addressed.

One of the most important examples of weighted-exchageable data is the case of distribution shift between calibration and test distributions~\citep{tibshirani2019conformal,podkopaev2021distribution}.
However, computability issues prevent the inclusion of available data from the test distribution in these procedures, which should be advantageous in the case of domain adaptation.
Another important example is federated learning in the presence of statistical heterogeneity between agents. In the first place, such heterogeneity cancels all standard guarantees of CP, since it conflicts with the principle of exchangeability, a fundamental assumption of CP. Again, the \cite{tibshirani2019conformal} approach can potentially be used. However, the problem of computational intractability remains. 

Our work addresses these challenges through the following key contributions:
\begin{itemize}
  \item We present a new method for constructing conformal prediction sets while accounting for heterogeneity in the data. The developed method introduces weights for the empirical distribution that can be computed efficiently, unlike those of \citet{tibshirani2019conformal}. In addition, the validity of the prediction sets is demonstrated with high probability. For more details, see Section~\ref{sec:approach}.
  
  \item We develop the application of the general method to the important practical situation of federated learning in the presence of statistical heterogeneity between agents. In particular, we propose a federated method for estimating importance weights and the resulting weighted quantile; see Section~\ref{sec:federated_conformal}.

  \item Extensive empirical assessments are conducted using synthetic data and various benchmark computer vision datasets. The results of these experiments are discussed in Section~\ref{sec:experiments}.
\end{itemize}

\section{Conformal Prediction under Data Heterogeneity}
\label{sec:approach}

  We consider the independent calibration data $\DC_{\tcount} = \{(X_{k}, Y_{k})\}_{k = 1}^\tcount$, where 
  \begin{equation}
    (X_{k},Y_{k}) \sim P^{k}, \quad k = 1, \dots, \tcount.
  \label{eq:heterogeneous_data}
  \end{equation}
  In this case, $X_{k}$ represents the covariate, $Y_{k}$ is the label and the distributions $P^{k}$ can arbitrarily vary for different $k$ in the general case. We also write by $Z_{k}=(X_{k},Y_{k})$ the pair of covariate and label and denote by $\XC$ and $\YC$ the support of $X$ and $Y$, respectively.
  The set $\XC$ is often a subset of $\R^d$, while $\YC$ can be either finite in classification or a subset of $\R$.
  The goal is to construct a prediction set $\predset{\alpha}(\xv)$ for the input object $\xv$ that conditionally on $\DC_\tcount$ yields a confidence level close to  $1 - \alpha$ with probability at least $1-\delta$, where $\alpha,\delta \in (0, 1)$:
  \begin{equation}\label{eq:target_coverage_guarantee}
    \prob\pr{Y_{\tcount + 1} \in \predset{\alpha}(X_{\tcount + 1}) \,\vert\, \DC_{\tcount}}
    = 1 - \alpha + \tau_{\tcount,\delta},
  \end{equation}
  where the new observation $(X_{\tcount + 1}, Y_{\tcount + 1})$ comes from the distribution $P^{\tcount + 1}$ and $\tau_{\tcount,\delta}$ is a small number. 

\subsection{Basics of Conformal Prediction}

  Let us start from describing \textit{Split Conformal Prediction (SCP)} which is de-facto standard approach in conformal prediction applications. This method assumes that the available data are split into two distinct subsets: a training dataset and a calibration dataset.
  A predictor $\hat{f}(\xv)$ is then trained on the training dataset, while the calibration dataset is used for generating the prediction sets.
  Using the predictor $\hat{f}(\xv)$, the so-called \textit{score function} can be calculated as $V(\xv, \yv)= S(\hat{f}(\xv), \yv)$.
  Here, the function $S$ measures the discrepancy between the predicted label $\hat{f}(\xv)$ and the target $\yv$. For example, in the case of regression, one can take $S(\hat{\yv}, \yv)= |\hat{\yv} - \yv|$, which leads to $V(\xv, \yv) =|\hat{f}(\xv) - \yv|$ \citep{papadopoulos2011regression,kato2023review}.

  The core idea of SCP is to calculate the scores for the available calibration data and then construct the prediction set based on an appropriate quantile estimated from the empirical distribution of the scores. More precisely, the resulting confidence set is given by
\begin{equation}\label{eq:uniform-confidence-set:main}
    \predset{\alpha, \mu}(\xv)
    =
    \{\yv \in \YC \colon V(\xv, \yv) 
    \le \q{1 - \alpha} (\mu )
    \},
  \end{equation}
  where $\mu = \frac{1}{\tcount + 1} \delta_{\infty} + \frac{1}{\tcount + 1} \sum\nolimits_{k = 1}^{\tcount} \delta_{V_{k}}$, $V_k = V(X_k, Y_k)$, and $\delta_{v}$ is the Dirac measure at $v \in \RR \cup \{\infty\}$.
  Such a confidence set allows  strong conformal guarantees under the assumption that all the calibration points $\{(X_k, Y_k)\}_{k=1}^N$ have the same distribution, i.e. $P^k \equiv P, ~ k = 1, \dots, \tcount$:
  \begin{equation} 
    \!\!\!\!\!\!\! 
    0 
    \le \prob\bigl(Y_{\tcount + 1} \in \predset{\alpha, \mu}(X_{\tcount + 1})\bigr) 
    - 1 + \alpha    
    < \tfrac{1}{\tcount + 1},
  \label{eq:uniform_coverage}
  \end{equation}
  for the data point $(X_{\tcount + 1}, Y_{\tcount + 1})$ generated from the same distribution $P$. 
   
  Since the prediction sets are generated using a fixed calibration data set, some of them could lead to suboptimal decisions with test coverage much less than $1 - \alpha$.
  This limitation can be particularly problematic in real-life scenarios where users work with a fixed calibration dataset.
  To enhance prediction reliability, it is essential to validate the coverage under most circumstances, ensuring its validity with a high probability, regardless of the specific calibration dataset \citep{vovk2012conditional,bian2023training,humbert2023one}.
  
  \format{CP applications beyond exchangeable cases.} 
  There exist multiple scenarios when standard SCP approach becomes not applicable as the data becomes non-exchageable. In particular, one can consider:
  \begin{enumerate}
    \item \textbf{Distribution shift between calibration and test data.} In this scenario, the calibration points are drawn as $(X_{k},Y_k) \sim P^{\mathrm{cal}}$, 
    $k \in [\tcount]$, whereas the new test point satisfies $(X_{\tcount+1},Y_{\tcount+1}) \sim P^{\mathrm{test}}$. 
    If one knows the likelihood ratio $\rmd P^{\mathrm{test}} / \rmd P^{\mathrm{cal}}(\xv, \yv)$, then the adaptation to distribution shift can be performed~\citep{tibshirani2019conformal}.

    \item \textbf{Domain adaptation.} In this scenario, in addition to the calibration points from a source distribution $(X_{k}, Y_k)\sim P^{\mathrm{cal}}, k \in [\tcount]$ one also has points from the test distribution $(X_{\tcount + m}, Y_{\tcount + m}) \sim P^{\mathrm{test}}, ~ m \in [M]$ with typically $M \ll \tcount$. If the new test point satisfies $(X_{\tcount+M+1},Y_{\tcount+M+1})\sim P^{\mathrm{test}}$, one can expect that the usage of both sets of calibration points should improve the resulting coverage compared to usage of only one of the sets.

    \item \textbf{Federated learning with statistical heterogeneity between agents.} In a federated learning setting, instead of storing the entire dataset on a centralized node, each agent $i \in [\nclients]$ owns a local calibration set $\DC_{i} = \acn{\prn{X_k^{i}, Y_k^{i}}}_{k = 1}^{\ccount{i}}$, where $\ccount{i}$ is the number of calibration samples for the agent $i$. We further assume that the calibration data are \iid\ within client and that the statistical heterogeneity is due to the difference between client data distribution:
    \begin{equation*}
      \forall k\in[\ccount{i}],\quad \prn{X_k^{i}, Y_k^{i}} \sim \Pcal{i}.
    \end{equation*}
    If one wants to perform personalized CP procedure towards the distribution of one of the agents $\target \in [\nclients]$, the shifts between the distribution of this agent and the others should be taken into account.
  \end{enumerate}
  These examples are the particular case of the independent but not identically distributed data considered in~\eqref{eq:heterogeneous_data}. In what follows we are going to present the general approach to CP that covers all these scenarios and allows for the efficient implementation.

\format{Conformal prediction under covariate shift.}
  In this work we focus on the particular case of covariate shift, while label shift or more general types of shifts can be considered in a similar way.
  The general approach for dealing with heterogeneous data distributions (due to covariate shift) was proposed by~\cite{tibshirani2019conformal}. The key idea is to account for the covariate shift by introducing  \emph{density ratios} $\iweight{\xv}$ that quantify the shift for the particular input point $\xv$. Finally, one can construct valid conformal confidence sets by considering the weighted quantile:
  \begin{equation*}
    \predset{\alpha, \qdistrb}(\xquery)
    =
    \{\yv \in \YC\colon V(\xquery, \yv) 
    \le \q{1 - \alpha}
    (\qdistrb_\xquery
    )
    \},
  \end{equation*}
  where $\qdistrb_\xquery = \qweightb{\xquery}{\xquery} \delta_{\infty} + \sum\nolimits_{k = 1}^{\tcount} \qweightb{X_k}{\xquery} \delta_{V_{k}}$. Here the weights $\qweightb{x}{x'}$ depend on the density ratios $\iweight{x}$ via a complex combinatorial formula that requires summation over $\tcount!$ elements. This makes the method inapplicable in practice and motivates the need for alternative approaches.

\subsection{Efficient Conformal Prediction under Data Shift}

 In this paper, we propose a general and efficiently computable approach for conformal prediction allowing for heterogeneity of the data.
 Consider the calibration measure $P^{\mathrm{cal}} = \frac{1}{\tcount} \sum_{k=1}^{\tcount} P^{k}$ and the test measure $P^{\mathrm{test}} \equiv P^{\tcount + 1}$. We propose to introduce density ratios of the following form:
  \begin{equation}\label{eq:general_importance_weights}
    \lambda{(\xv, \yv)} = \tfrac{\rmd P^{\mathrm{test}}}{\rmd P^{\mathrm{cal}}}(\xv, \yv)
  \end{equation}
  assuming that $P^{\mathrm{test}}$ is absolutely continuous with respect to $P^{\mathrm{cal}}$, i.e., $P^{\mathrm{test}} \ll P^{\mathrm{cal}}$. 
  The density ratios are given by $\lambda_{k} = \lambda{(X_k, Y_k)}$, where $(X_k,Y_k)\sim P^k$.
  In the following, we consider the prediction set defined by
  \begin{equation}\label{eq:weighted_confidence_set}
    \predset{\alpha, \mu}(\xquery)
    =
    \{\yv \in \YC\colon V(\xquery, \yv) 
    \le \q{1 - \alpha}
    (\mu_{\xquery, \yv}
    )
    \},
  \end{equation}
  where the probability measure and the importance weights are given, with $V_k = V(X_k,Y_k)$, by
  \begin{align}
    \label{eq:weighted_distribution}
    &\textstyle \mu_{\xquery, \yv} = p_{\tcount + 1}^{(\xv, \yv)} \delta_{\infty} + \sum_{k = 1}^{\tcount} p_{k}^{(\xv, \yv)} \delta_{V_k},
    \\
    \label{eq:def:weights:main}
    &p_{k}^{(\xv, \yv)} = \tfrac{\lambda_{k}}{\lambda{(\xv, \yv)} + \sum_{l = 1}^{\tcount} \lambda_{l}},
    \,\,p_{\tcount + 1}^{(\xv, \yv)} = \tfrac{\lambda{(\xv, \yv)}}{\lambda{(\xv, \yv)} + \sum_{l = 1}^{\tcount} \lambda_{l}}.
  \end{align}
  Interestingly, although these weights may appear similar to those proposed by~\citet{tibshirani2019conformal}, they are quite different. 
  In fact, the importance weights proposed by~\citet[Lemma~3]{tibshirani2019conformal} rely on intractable sums of permutations when the distributions $\{P_k\}_{k=1}^{\tcount}$ are distinct. At the same time the weights~\eqref{eq:def:weights:main} can be computed in a straightforward way given the density ratios $\lambda{(\xv, \yv)}$. 

\format{Theoretical guarantees.}
  Since controlling the average miscoverage rate is not enough to ensure the validity of the method in most cases, in this part, we show that the miscoverage rate does not deviate much from the target value $\alpha$ for most of the calibration datasets $\DC_{\tcount}$. This is the result of the following theorem; see full proof and additional details in \Cref{sec:conditional-coverage}, while sketch of the proof is given below.
\begin{theorem}\label{theorem:conditional_coverage}
    Assume there are no ties between $\{V_k\}_{k\in[\tcount+1]}\cup\{\infty\}$ almost surely.
    If $\{\lambda_{k}-\E\lambda_k\}_{k \in [\tcount]}$ are sub-Gaussian random variables with parameters $\sigma_{1}, \ldots, \sigma_{\tcount}$, then, for any $\delta \in (0,1)$, with probability at least $1 - \delta$, it holds
    \begin{multline*}
      - \frac{\tau_{\tcount, \delta} + 3\tcount^{-1} \E\lambda_{\tcount+1}}{1 + \tcount^{-1} \E\lambda_{\tcount+1}}
      < \alpha(\DC_{\tcount}) - \alpha
      \\
      < \tau_{\tcount, \delta}
      + \sup_{v \in \RR} \prob\pr{V_{\tcount + 1} = v}
      ,
    \end{multline*}
    where we denote
    \[
      \tau_{\tcount, \delta} = \tcount^{-1} \sqrt{\textstyle 8 \log (\frac{1}{6\delta}) \sum_{k = 1}^{\tcount} \prn{4 \sigma_{k}^{2} + \E\brn{\lambda_k}^2}}.
    \]
  \end{theorem}
  This theorem shows that the confidence level of $\predset{\alpha, \mu}(\xquery)$ is close to $1 - \alpha$, regardless of the calibration set $\DC_{\tcount}$. 
  Specifically, it meets the condition in~\eqref{eq:target_coverage_guarantee} with $\tau_{\tcount, \delta}= \Oh(\tcount^{-1/2})$. 
  Under exchangeable data, this matches the standard deviation of the miscoverage rate distribution exhibited in~\eqref{eq:sim:alphaDN}.
  Therefore, we extend the guarantees of standard CP methods to our approach but significantly broadening the scope of applicability by allowing general distribution shifts.
  The following result concerns the bias of our method; see \Cref{sec:bias} for more details.

  \begin{theorem}\label{theorem:bias}
    Assume there are no ties between $\{V_k\}_{k\in[\tcount+1]}\cup\{\infty\}$ almost surely.
    If for any $v\in\R$, $\lambda(Z) \1_{V(Z)<v} - \E[\lambda(Z) \1_{V(Z)<v}]$ is sub-Gaussian with parameter $\sigma\ge 0$, where $Z\sim P^{\mathrm{cal}}$.
    Then, it holds
    \begin{equation*}
      \absbig{\E[\alpha(\DC_{\tcount})] - \alpha}
      \le 19 \sigma \sqrt{\tfrac{ \log 4\tcount }{ \tcount }} + \tfrac{ 18 \E\lambda^2(Z_{\tcount+1})}{\tcount}
      .
    \end{equation*}
  \end{theorem}
  Interestingly, the upper bound depends on $\lambda$, which does not directly capture the individual data distributions $P^k$, but rather the difference between the test distribution and the average of the local calibration distributions.
  The convergence rate of order $\tcount^{-1/2}$ contrasts with the findings of \citet{plassier2023conformal}, where their subsampling procedure achieved the standard centralized bias of order $\Oh(\tcount^{-1})$.
  However, while subsampling techniques do improve bias, they also amplify the variance of the miscoverage rate, which can become a limiting factor.
  Also in the federated context, \citet{lu2023federated} developed a method to generate prediction sets for $(X_{\tcount+1},Y_{\tcount+1})\sim \sum_{i=1}^{\nclients} \frac{\ccount{i}+1}{\tcount+\nclients} P^{i}$. 
  Their method achieves a bias that is inversely proportional to the proportion of data held by each agent, that scales as $\Oh(\nclients/\tcount)$. 
  However, this result becomes less favorable when the number of agents exceeds $\sqrt{\tcount}$, and additionally, their method does not allow for personalized prediction sets.

  \format{Sketch of Proof.}
  We will briefly outline the main steps of the proof of \Cref{theorem:conditional_coverage}. 
  In particular, we will focus on upper-bounding the miscoverage error $\alpha(\DC_{\tcount}) - \alpha$.
  Detailed proofs can be found in the supplementary paper, where the lower bound is also analyzed.
  All the results are derived from the connection between coverage and the cumulative distribution function, which is defined for any data point $z = (x, y) \in \XC \times \YC$ by
  \begin{equation*}
    F_{\tcount+1}(z) = \E\br{\indiacc{V_{\tcount+1}\le V(x,y)}}.      
  \end{equation*}
  We also introduce $\widehat{F}_{\tcount+1}(z)$, an empirical approximation of $F_{\tcount+1}(z)$:
  \begin{equation}
    \widehat{F}_{\tcount+1}(z) = \sum_{k=1}^{\tcount} p_{k}^{(x,y)} \indiacc{V_{k} \le V(x,y)}.
    \label{eq:def:widehat-F}
  \end{equation}
  For the first step of the proof, remark that
  \begin{multline*}
    \prob\pr{\alpha(\mathcal{D}_{\tcount}) < \alpha + \tau_{\tcount, \delta} + \sup_{v\in\R} \prob\ac{V_{\tcount+1} = v}}
    \\
    \ge \prob\pr{\sup_{z\in\XC\times\YC}\ac{\widehat{F}_{\tcount+1}(z) - F_{\tcount+1}(z)} < \tau_{\tcount, \delta}}.
  \end{multline*}
  As in~\citep{bian2023training}, the problem consists in controlling difference between the empirical and the true cumulative distribution functions. 
  Given that the weights $p_k^{(x,y)}$ depend on the entire calibration dataset, we introduce independent approximated weights $q_{k}=\lambda_{k}/\sum_{l=1}^{\tcount}\E\lambda_l$.
  Based on those weights, consider
  \begin{equation*}
    \widehat{G}_{\tcount}(z)
    = \sum_{k=1}^{\tcount} q_{k} \indiacc{V_{k}\le V(x,y)}.
  \end{equation*}
  It can be shown that $\E[\widehat{G}_{\tcount}(z)] = F_{\tcount+1}(z)$. Therefore, the central part of the proof involves controlling the two right-hand side terms separately as follows:
  \begin{multline*}
    \widehat{F}_{\tcount+1}(z) - F_{\tcount+1}(z)
    = \widehat{F}_{\tcount+1}(z) - \widehat{G}_{\tcount}(z)
    \\
    + \widehat{G}_{\tcount}(z) - F_{\tcount+1}(z)
    .
  \end{multline*}
  Since the approximation weights $q_k$ are chosen such that $q_k \approx p_k$, the term $\sup_{z\in\XC\times\YC}\{\widehat{F}_{\tcount+1}(z) - \widehat{G}_{\tcount}(z)\}$ can be controlled using classical concentration inequalities; refer to \Cref{lem:bound:pk-pkXk:2}. 
  However, controlling $\widehat{G}_{\tcount}(z) - F_{\tcount+1}(z)$ presents a more challenging task.
  As detailed in \Cref{sec:general-dkw}, applying techniques similar to those used in the Dvoretzky--Kiefer--Wolfowitz (DKW) proof~\citep{dvoretzky1956asymptotic,massart1990tight}, it is possible to show that
  \begin{multline}\label{eq:bound:sup-hatG-G}
    \prob\pr{\sup_{z \in \ZC}\ac{\widehat{G}_{\tcount}(z)-F_{\tcount+1}(z)} \ge \tau_{\tcount, \delta}}
    \\
    \le 2 \inf _{\theta>0} \ac{\rme^{- \theta \tau_{\tcount, \delta}} \prod_{k=1}^{\tcount} \E\br{\cosh\pr{\theta q_k}}}.
  \end{multline}
  Notably, using $\cosh(\theta q_k)\le \rme^{(\theta q_k)^2/2}$ with $q_k=1/\tcount$ leads to the standard DKW result. 
  However, utilizing $\E\cosh(\theta q_k) \le \cosh(\theta \E q_k) \exp(2 \theta^{2} \sigma_k^2 / \tcount^{2})$ provides the following result:
  \begin{multline*}
    \inf_{\theta>0} \ac{e^{- \theta \tau_{\tcount, \delta}} \prod_{k=1}^{\tcount} \E\br{ \cosh\pr{\theta q_k} }}
    \\
    \le \exp\pr{- \frac{\tau_{\tcount, \delta}^{2} \prn{\sum_{k=1}^{\tcount} \E\lambda_k}^{2}}{2 \sum_{k=1}^{\tcount} \prn{4 \sigma_{k}^{2} + \E\brn{\lambda_k}^2}}}
    .
  \end{multline*}
  Combining the previous line with~\eqref{eq:bound:sup-hatG-G} enables controlling the deviation of $\widehat{G}_{\tcount}(z) - F_{\tcount+1}(z)$.
  These high-probability bounds for $\widehat{F}_{\tcount+1}(z) - \widehat{G}_{\tcount}(z)$ and $\widehat{G}_{\tcount}(z) - F_{\tcount+1}(z)$ are then unified to obtain the result of \Cref{theorem:conditional_coverage}.

\section{Federated Conformal Prediction under Covariate Shift}
\label{sec:federated_conformal}

  Let us illustrate the general approach with an application to the challenging problem of federated conformal prediction. 
  We consider a system of $\nclients$ agents, and we further assume that the calibration data for these agents are heterogeneous due to the presence of a covariate shift:
  \begin{equation} 
    \prn{X_k^{i}, Y_k^{i}} \sim P^{i} = P_X^{i} \times P_{Y \mid X}, ~ k \in [\ccount{i}], ~ i \in [\nclients],
  \end{equation}
  where $P_{Y \mid X}$ is the conditional distribution of the label given the covariate that is assumed identical among agents, $P_X^{i}$ is the prior covariate density that may differ across agents, and $\ccount{i}$ is the number of calibration point for the agent $i$.

  Given calibration datasets $\DC_{i} = \acn{\prn{X_k^{i}, Y_k^{i}}}_{k = 1}^{\ccount{i}}, i \in [\nclients]$, the goal is to construct a prediction set $\predset{\alpha}^{\target}(\xv)$ for an agent $\target \in [\nclients]$ and input $\xv$ with confidence level $1 - \alpha \in (0, 1)$.
  Ideally, for any new point $(X_{\tcount + 1}^\target, Y_{\tcount + 1}^\target)$ drawn from the distribution $P^\target$, the conditional coverage $\prob\pr{Y_{\tcount + 1}^\target \in \predset{\alpha}^{\target}(X_{\tcount + 1}^\target) \,\vert\, \DC_{\tcount}}$ should be near the confidence level $1-\alpha$ set by the user.
  Importantly, this confidence set should benefit not only the data available on the target agent, but also from the calibration data of other agents.

  In our federated setup, the general density ratios~\eqref{eq:general_importance_weights} become
  \begin{equation}
    \lambda(\xv) = \tfrac{P^{\target}_X(\xv)}{\sum_{i = 1}^{\nclients} \pi_i P^{i}_X(\xv)},
    \label{eq:importance_weights}
  \end{equation}
  where $P_{X}^{i}$ represents the covariate density of agent $i \in [n]$, and $\pi_i = \ccount{i} / \tcount$. 
  Using these ratios~\eqref{eq:importance_weights}, we can compute the general weighted confidence set~\eqref{eq:weighted_confidence_set} to obtain the confidence set $\predset{\alpha}^{\target}(\xv)$. Importantly, the resulting confidence set will fully satisfy the result of Theorem~\ref{theorem:conditional_coverage} and thus allows for the tight coverage guarantees.

\format{Federated inference procedure.} 
  In practice, the confidence set $\predset{\alpha}^{\star}(\xquery)$ from~\eqref{eq:weighted_confidence_set} can be computed in two steps:
  \begin{itemize}[nosep,labelindent=0pt,labelwidth=!,itemsep=.3ex,topsep=-3pt] 
    \item[(i)] estimate the weights $p_{k}^{(\xv, \yv)} \equiv p_{k}^{(\xv)}$ that are involved in the definition of weighted empirical score distribution function in~\eqref{eq:weighted_confidence_set};
    \item[(ii)] compute a quantile of the weighted empirical score distribution function $\mu_{\xquery, \yv} \equiv \mu_{\xquery}$ from~\eqref{eq:weighted_distribution}.
  \end{itemize}

\format{Federated ratios estimation.}
  The computation of importance weights $\{p_{k}^{(\xv, \yv)}\}_{k=1}^{\tcount+1}$ in the empirical CDF relies on the determination of ratios $\iweightp{\xv}$, which are initially unknown and should be estimated.
  For this, we utilize the Gaussian Mixture Model (GMM).
  In this process, each client independently computes its GMM parameters $\{\pi_{y}^{i}, \mu_{y}^{i}, \Sigma_{y}^{i}\}_{y\in\YC}$ using their local data $\DC_{i}$:
  \begin{equation}\label{eq:def:GMM-params}
    \begin{aligned}
      &\pi_y^i = \nofrac{\clcount{y}{i}}{\ccount{i}},
      \qquad 
      m_y^i = (\nofrac{1}{\clcount{y}{i}}) \sum_{k} \phi(X_k^i),
      \\
      &\Sigma_y^i = (\nofrac{1}{\clcount{y}{i}}) \sum_{k} \, (\phi(X_k^i) - m_y^i) (\phi(X_k^i) - m_y^i)^{\top}.
    \end{aligned}
  \end{equation}
  Here, $\phi$ denotes the mapping obtain while keeping the first layers of the trained neural network $\hat{f}$, and $\clcount{y}{i}$ is the number of data classified $y$ on client $i$.
  These parameters are then transmitted to the other agents.
  With this information, the local agents compute the density ratio $\lambda_k^i=\lambda(X_k^i)$ on their local data, based on \eqref{eq:importance_weights} with
  \begin{align*}
    &P_X^{\target}(x)
    = \sum_{y\in\YC} \pi_{y}^{\target} \, \mathcal{N}(\phi(x); \mu_{y}^{\target}, \Sigma_{y}^{\target}),
    \\
    &P_X^{\mathrm{cal}}(x)
    = \sum_{i=1}^{\nclients} \sum_{y\in\YC} \pi_{y}^{i} \, \mathcal{N}(\phi(x); \mu_{y}^{i}, \Sigma_{y}^{i}),
  \end{align*}
  where $\mathcal{N}(\phi(x); \mu_{y}^{i}, \Sigma_{y}^{i})$ denotes the pdf of the Gaussian distribution with mean $\mu_{y}^{i}$ and covariance matrix $\Sigma_{y}^{i}$.

\format{Federated quantile estimation.}
  Once the density ratios are estimated, they can be used in the federated quantile estimation procedure described by~\cite{plassier2023conformal}. 
  This algorithm relies on regularized pinball loss functions, which are smoothed using the Moreau-Yosida inf-convolution~\citep{moreau1963proprietes}. 
  This smoothing enables the application of classical FL stochastic gradient methods.
  Theoretical guarantees on the coverage are provided, along with assurances of differential privacy.
  However, the developed approach demands calculating the quantile for any new query $\xquery$.
  To mitigate this issue, we approximate the importance weights $p_{k}^{(\xv, \yv)}$ by query-independent weights $\widehat{p}_{k} = (\sum_{l = 1}^{\tcount} \lambda_{l})^{-1} \lambda_{k}$.
  As a result, the considered prediction set becomes:
  \begin{equation*}
    \textstyle
    \widehat{\mathcal{C}}_{\alpha, \mu}(\xquery)
    = \ac{\yv \in \YC\colon V(\xquery, \yv) \le \widehat{Q}_{1 - \alpha}^{(\gamma)}\bigl(\sum_{k = 1}^{\tcount} \widehat{p}_{k} \delta_{V_k}\bigr)},
  \end{equation*}
  where $\widehat{Q}_{1 - \alpha}^{(\gamma)}\prn{ \sum_{k = 1}^{\tcount} \widehat{p}_{k} \delta_{V_k} }$ is the approximated quantiles obtained via the FL minimization of the $\gamma$-regularized pinball loss; refer to \Cref{algo:Qgamma} in the Supplementary Material for more details.
  Note that calculating $\widehat{\mathcal{C}}_{\alpha, \mu}(\xquery)$ is straightforward, because the quantile $\widehat{Q}_{1 - \alpha}^{(\gamma)}\prn{ \sum_{k = 1}^{\tcount} \widehat{p}_{k} \delta_{V_k} }$ is independent of the query $\xquery$.

\section{Related Work}
\label{sec:related_work}

The reliability of prediction sets has been extensively studied  in centralized frameworks  \citep{papadopoulos2008inductive}.
However, their marginal capabilities may generate invalid confidence intervals for specific data subgroups.
As a result, analyzing conditional conformal prediction methods is crucial, especially for high heteroscedasticity regions \citep{alaa2023conformalized}.
Exploring flexible coverage guarantees has been investigated by \citet{foygel2021limits}. 
They studied potential relaxations of marginal coverage guarantees, however, achieving such guarantees is not always feasible \citep{vovk2012conditional,lei2014distribution}.
In the case of exchangeable data, \cite{bian2023training} examined theoretical guarantees in the \textit{training-conditional} setting. 
They showed with high probability that the miscoverage rate is upper bounded by the value $2\alpha\pm \Oh_{\prob}(1/\sqrt{\tcount})$ for the K-fold CV+ method.
In addition, they also demonstrated the unattainability of similar results for the full conformal method or the jackknife+ method without further assumptions.

Most conformal prediction methods are guaranteed to be valid under data exchangeability. However, this assumption can be restrictive. 
To alleviate this constraint, \citet{tibshirani2019conformal} proposed a method working with independent data, even when they follow different distributions. 
This approach is beneficial when dealing with data from diverse sources, which tend to produce data from different distributions.
However, this method requires the determination of importance weights based on permutations of density ratios. 
If the distribution of only one datapoint differs from the calibration distribution, \citet{lei2021conformal} demonstrated that approximations of these importance ratios is sufficient to achieve the desired coverage level. 
The error introduced by this ratio estimation can be controlled by the total variation distance between the true ratio and its approximation.
Furthermore, \citet{plassier2023conformal} introduced a subsampling technique to simplify the combinatorial complexity of the importance weights computations.
Based on a multinomial draw of the calibration datapoints, their method achieves a coverage error of order $\Oh(\tcount^{-1} \log\tcount)$, 
almost recovering the standard theoretical guarantee for exchangeable data \citep{vovk2005algorithmic}.
However, we conjecture that the numerical advantage of the proposed method is primarily due to the small deviation of the miscoverage rather than the small bias.

This framework of data heterogeneity is particularly relevant in federated contexts where there may be shifts in data distribution among different agents. The concept of using calibration data from multiple agents to refine prediction sets has been the subject of several recent papers; see, among others, \citep{lu2021distribution,humbert2023one,lu2023federated,plassier2023conformal,zhu2023federated}.
\cite{plassier2023conformal} introduced a federated conformal prediction method to account for label shifts. 
However, their methods require multiple rounds of communication between a central server and the agents.
In contrast, \citet{humbert2023one} have developed a conformal approach with only one communication round.
Their method relies on computing a quantile locally, which is then shared with the central server. Then, the central server computes the quantile of these quantiles to create prediction sets.
Another approach to generating prediction sets when the test point follows the agent's mixture distribution is presented in~\citep{lu2023federated}.





\section{Experiments}
\label{sec:experiments}

\format{Domain Adaptation on Synthetic Data.}
We start with an initial experiment to demonstrate the potential of our proposed approach in addressing domain adaptation.
For this, we consider synthetic 1-D data and a two layer neural network regression model. 
We use 100 calibration data points to build 90\% prediction intervals for 20 test data.
The data are sampled from either the distribution $P^{1} = \mathcal{N}(3, 2)$ or drawn from $P^{2} = \mathcal{N}(5, 2)$. 
The calibration set is sampled from either $P^{1}$ (80 points), $P^{2}$ (20 points) or a mixture of both (80 + 20 points).
The observations are modeled as $Y = (1 + 0.1|X|)\sin(X) + \varepsilon$, where $\varepsilon \sim \mathcal{N}(0, 0.5)$. 
We train a fully-connected two layers neural network with 10 neurons per layer, using the MSE loss function computed on 150 data points.
To illustrate this, Figure~\ref{fig:synth_data} displays an example of the sampled data and the trained model.
When $P^{\mathrm{cal}} = P^{1}$, note that our method aligns with the approach proposed by~\citet[Corollary~1]{tibshirani2019conformal}.
This experiment is replicated 500 times, using the non-conformity score $V(\xv,\yv)=|\hat{f}(\xv)-\yv|$. The results are summarized in Table~\ref{tab:synth_data}. 
Classical conformal method noticeably breaks the 90\% coverage guarantee when the non-exchangeability is not satisfied.
In contrast, our proposed method achieves the desired confidence level, while demonstrating a lower standard deviation than competitors.

\begin{table}[]
  \centering
  \begin{tabular}{ccl}
  \toprule
      {\bf METHOD} & $P^{\mathrm{cal}}$ & {\bf COVERAGE}\\
      \midrule
      \textsc{This work} & Mix & 91.03 \:$\pm$\: 4.86 \%\\
      \citet{plassier2023conformal} & Mix & 92.28 \:$\pm$\: 5.97 \%\\ 
      \citet{tibshirani2019conformal} & Mix & 82.35 \:$\pm$\: 3.94 \% \\ 
      \textsc{Classic conformal} & Mix & 81.75 \:$\pm$\: 5.55 \% \\
      \citet{tibshirani2019conformal} & $P^{1}$ & 92.96 \:$\pm$\: 5.31 \% \\ 
      \textsc{Classic conformal} & $P^{1}$ & 79.16 \:$\pm$\: 6.88 \% \\
      \textsc{Classic conformal} & $P^{2}$ & 90.26 \:$\pm$\: 6.58 \% \\
      \bottomrule
  \end{tabular}
  \caption{Coverage on synthetic data. Mix corresponds to a $80/20$ mixture of data from $P^{1}$ and $P^{2}$.}
  \label{tab:synth_data}
  \vskip-13pt
\end{table}

\begin{figure}[t]
    \centering
    \includegraphics[width=1\linewidth]{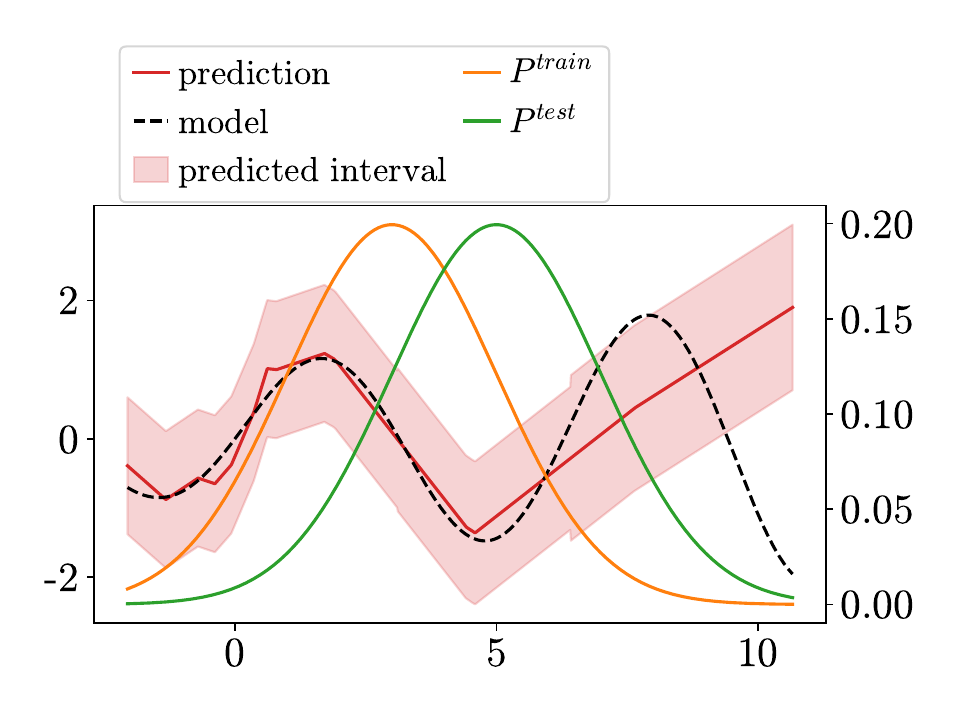}
    \vskip-15pt
    \caption{Example of synthetic data. We show the true dependence as a dotted line, neural network prediction and the predictive confidence interval are in red. Additionally, we present PDFs of the train and test distributions on a secondary vertical axis.}
    \label{fig:synth_data}
\end{figure}

\begin{figure*}[t!]
  \centering
  \begin{subfigure}{0.45\textwidth}
    \includegraphics[width=\textwidth]{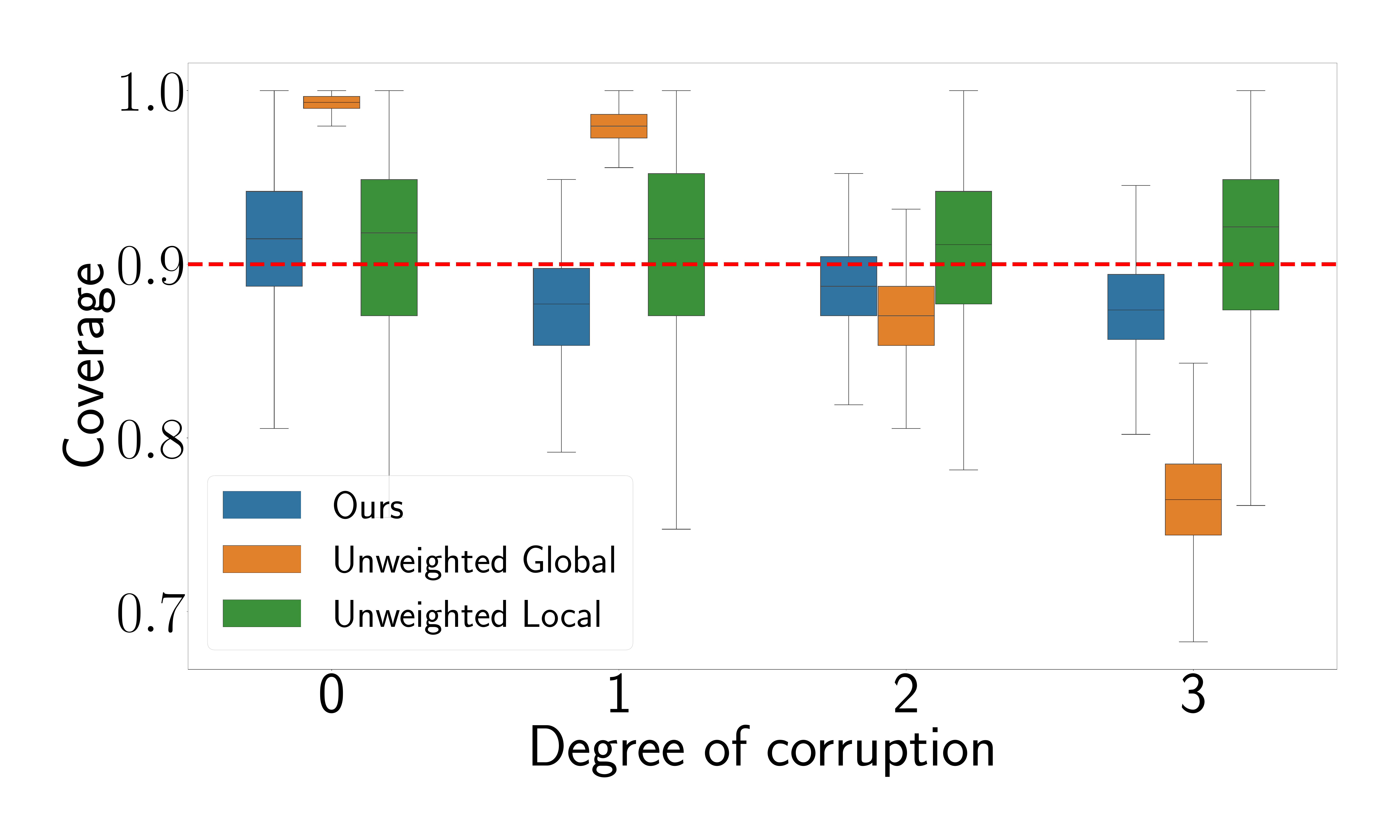}
    \vskip-10pt
    \caption{} \label{fig:cifar_10_coverage_panel}
  \end{subfigure}
  ~~~~
  \begin{subfigure}{0.45\textwidth}
    \includegraphics[width=\textwidth]{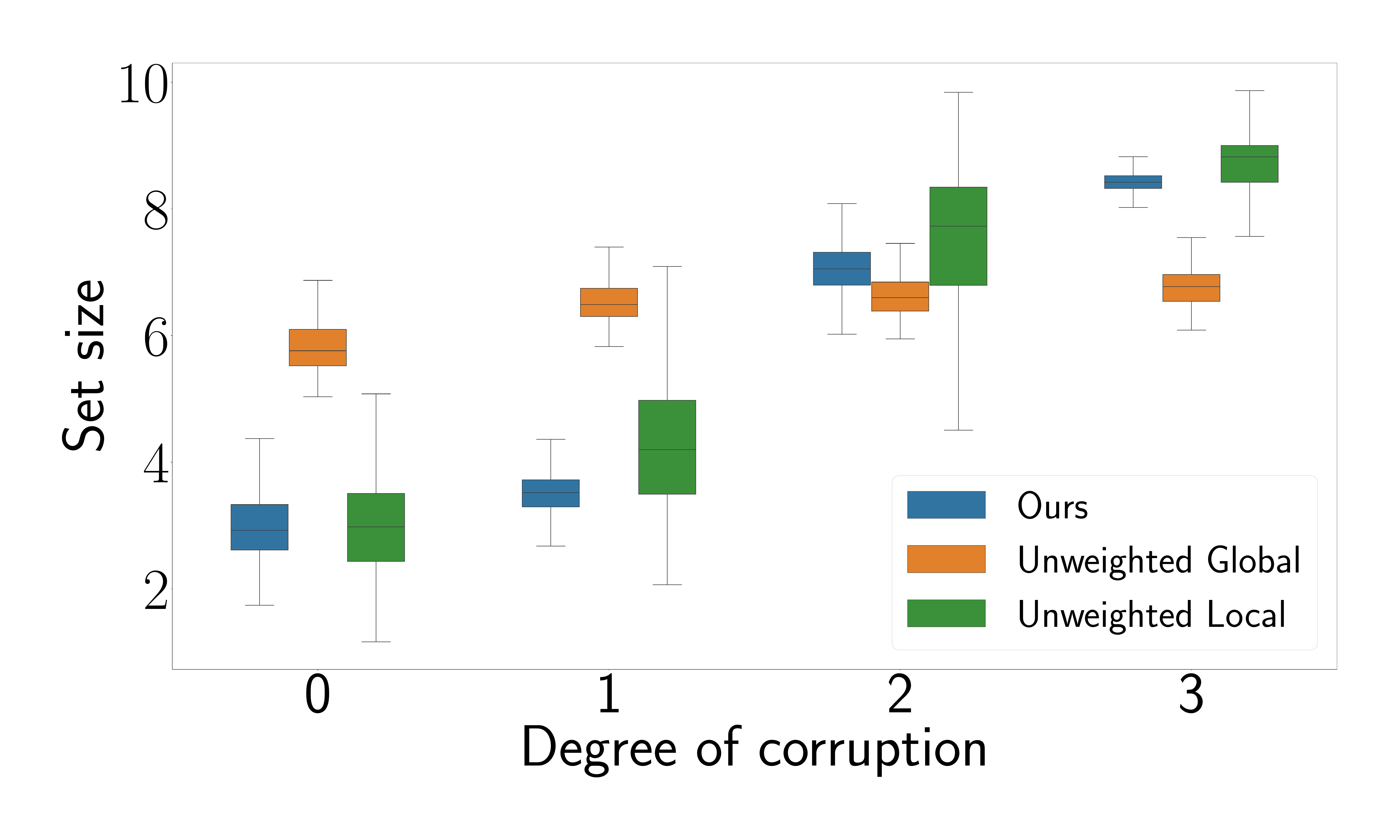}
    \vskip-10pt
    \caption{} \label{fig:cifar_10_setsize_panel}
  \end{subfigure}
  \vskip-5pt
  \caption{CIFAR-10 experimental results: (a) The distribution of coverage percentage for each agent. It shows how closely the predicted values covers the true values across varying degrees of data corruption. (b) The distribution of set sizes for different agents. The plot illustrates the growth in set sizes as data corruption increases, emphasizing the relationship between data integrity and set size.}
  \vskip-10pt
\end{figure*}

\format{Federated Learning on CIFAR-10.}
In this experiment, we follow Section~\ref{sec:federated_conformal} and consider personalized federated learning problem with the CIFAR-10 dataset. To induce a covariate shift between agents, we apply augmentations using Gaussian blur at varying intensities. This strategy modifies the agent's local distributions \(P_X^i\) while keeping \(P_{Y\mid X}^i\) intact.
For this experiment, we introduce four different levels of data corruption.
Level 0 corresponds to the original data, whereas levels 1-3 introduce increasing degrees of corruption; 1 being minimal and 3 being the most severe. We achieve this using square Gaussian kernels of sizes 3, 5, and 7. Additionally, the standard deviation of the elements in these filters is adjusted based on kernel size $\kappa$ using the formula \(\sigma = 0.3 \cdot \bigl((\kappa - 1) \cdot 0.5 - 1\bigr) + 0.8\).
In this experiment, we consider 40 different distributed agents with 10 agents for each corruption level.
Each agent owns a calibration set of 25 objects.

In our testing phase, we set \( \alpha=0.1 \) and, for the non-conformity score, we adopt the adaptive prediction sets (APS) approach developed in~\citep{romano2020classification, angelopoulos2021gentle}:
\begin{EQA}[c]
\label{eq:def:APS}
    \textstyle V(\xv, \yv)= \sum_{\{y\colon \hat{f}(\xv)_y \geq \hat{f}(\xv)_{\yv}\}} \hat{f}(\xv)_{y},
\end{EQA}
where $\hat{f}(\xv)_{y}$ is the predicted probability of label $y \in \YC$.
We consider ResNet-18 model trained on the unaltered CIFAR-10 dataset. We found temperature scaling of the logits beneficial and set the temperature \(T=10\) in our experiments. The weights associated to the nonconformity scores are computed with the GMM approach described above.
We compare our method against three baselines: ``Unweighted Local'', ``Unweighted Global'' and ``Weighted global (resampled)''~\citep{plassier2023conformal}. 
First two use standard conformal prediction with \(\lambda=\prn{\text{Num data} + 1}^{-1}\), while the latter uses specially constructed weights.
``Unweighted Local'' only uses the local agent's calibration data, while ``Unweighted Global'' combines calibration data from all agents.
Our results are obtained from 100 independent runs, each using different random splits of an agent's data for calibration, testing, and for density model fitting. 
For each iteration, we evaluate the local test set of an agent, calculating both average coverage and average set size. The metrics are collected across 100 runs and visualized using boxplots.

Results for \(\alpha=0.1\) are shown in Figure~\ref{fig:cifar_10_coverage_panel} and Figure~\ref{fig:cifar_10_setsize_panel}. It is worth noting that while we had 40 agents, we aggregated results for clarity, presenting them by the level of data corruption.
The figures show that except for ``Unweighted Global'', all methods achieve the target coverage. ``Weighted global (resampled)'' has a coverage of $92.28\pm5.97\%$ which is perfectly aligned with expectation to have increased variance compared to our method (see Table~\ref{tab:synth_data}).
In contrast, our method consistently shows less variation in set sizes than the others. This consistency arises from the federated process, when calibration scores from other agents, once appropriately weighted, contribute to the conformal procedure for a given agent.

It is worth mentioning, that we do not present~\citep{tibshirani2019conformal} here among competitors. For this method, the only feasible approach is to sample a certain number of permutations and compute approximate weights based on them. Unfortunately, it leads to the very high variance of the weights which led us even to the bias in the mean coverage ($82.35$ on the ``Mix'' calibration data in our experiment on domain adaptation). When replicating the CIFAR10 experiment, sampling the permutations led us to even worse results.

\format{Federated Learning on CIFAR-100.}
For CIFAR-100 dataset we performed a similar experiment but with slightly different hyperparameters and another model -- ResNet50. However, here we study how the size of calibration dataset influences the results. Basically, we explore how mean and standard deviation of empirical coverage changes with the size of the calibration dataset. The results are presented in Figure~\ref{fig:cifar_100_avg_mean} and Figure~\ref{fig:cifar_100_avg_std}. In each setup, we have 100 distributed agents which are split between 4 different levels of corruption.
From these plots we see that our method has the smallest variance among others, and converges fast to the right mean with the increase of the calibration dataset. In contrast, other methods are much slower to converge with the ``Unweighted Global'' method having extremely high variance as it either significantly undercovers or overcovers depending on the level of data corruption.

\begin{figure*}[t!]
  \centering
  \begin{subfigure}{0.48\textwidth}
    \includegraphics[width=\textwidth]{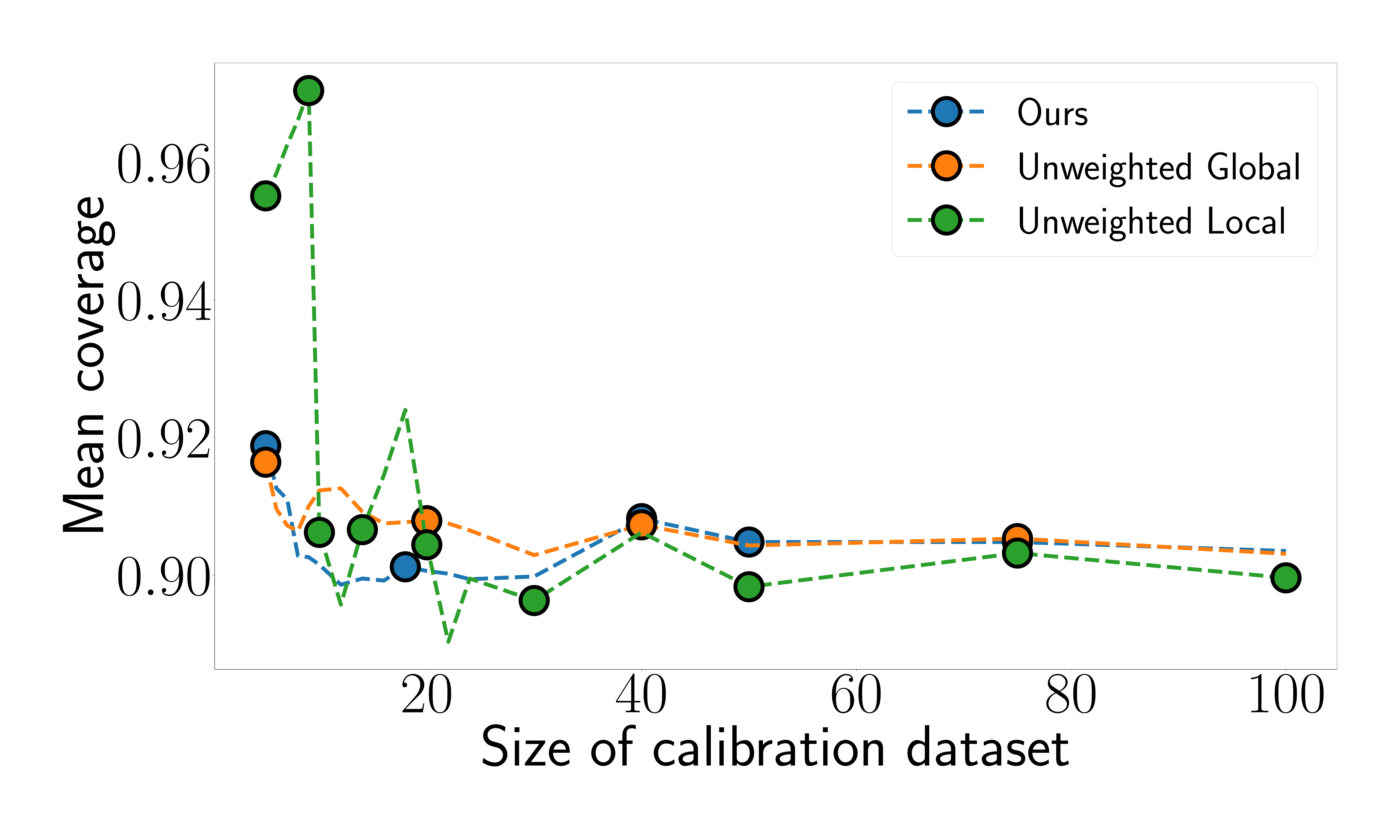}
    \caption{} \label{fig:cifar_100_avg_mean}
  \end{subfigure}
  ~~~~
  \begin{subfigure}{0.48\textwidth}
    \includegraphics[width=\textwidth]{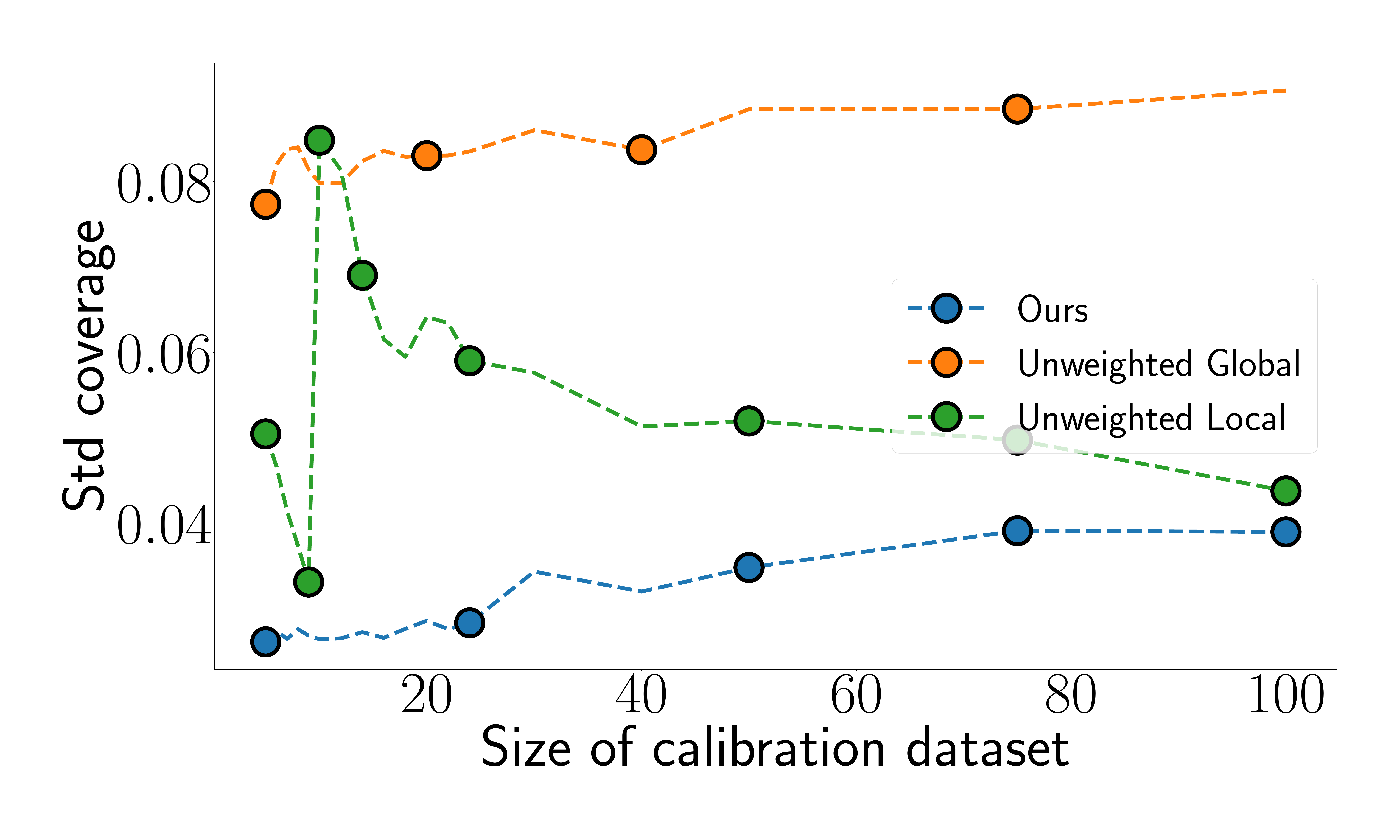}
    \caption{} \label{fig:cifar_100_avg_std}
  \end{subfigure}
  \caption{CIFAR-100 experimental results: (a) Mean empirical coverage changes as function of the calibration dataset size. (b) Standard deviation of empirical coverage as function of the calibration dataset size.}
\end{figure*}

\begin{figure*}[t!]
  \centering
  \begin{subfigure}{0.48\textwidth}
    \includegraphics[width=\textwidth]{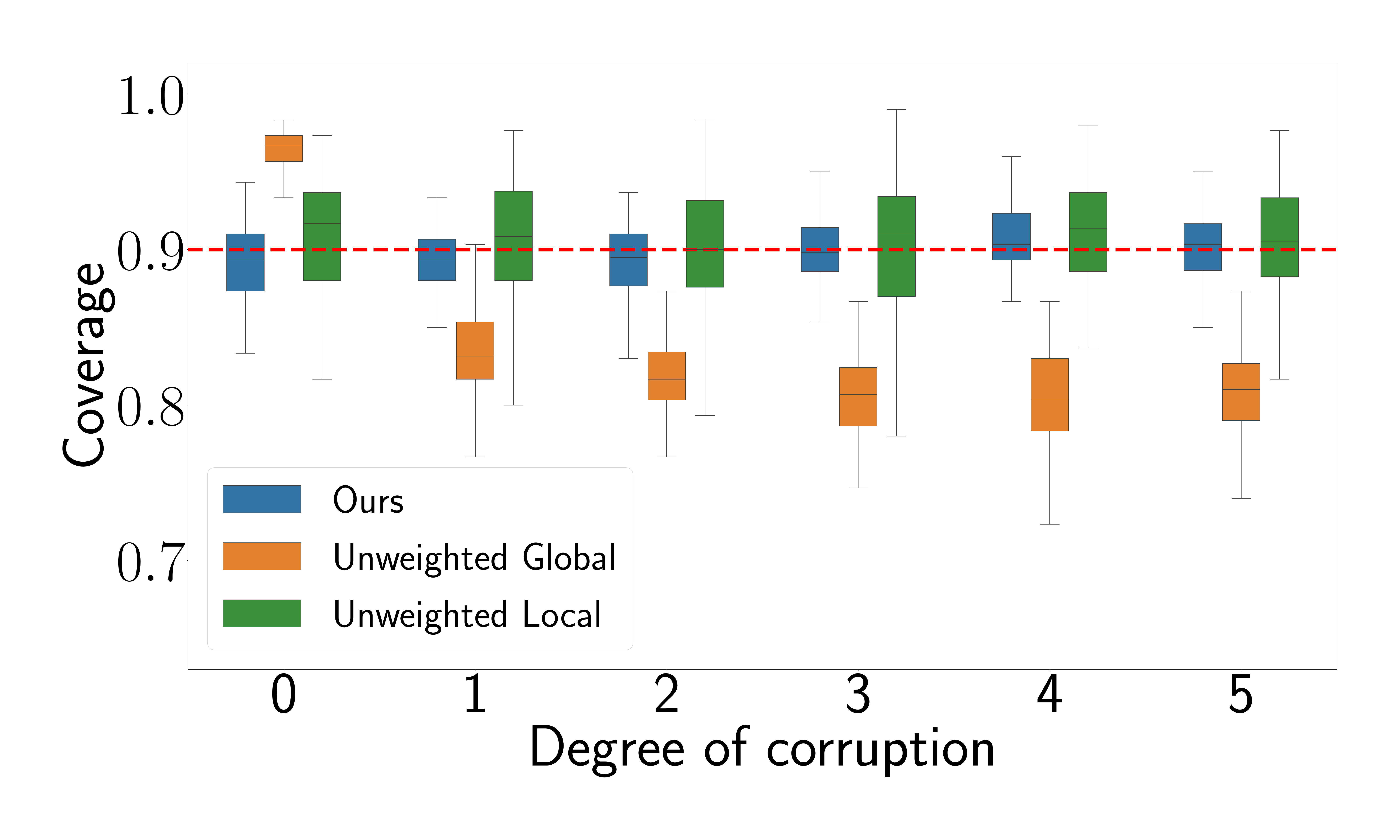}
    \caption{} \label{fig:imagenet_coverage}
  \end{subfigure}
  ~~~~
  \begin{subfigure}{0.48\textwidth}
    \includegraphics[width=\textwidth]{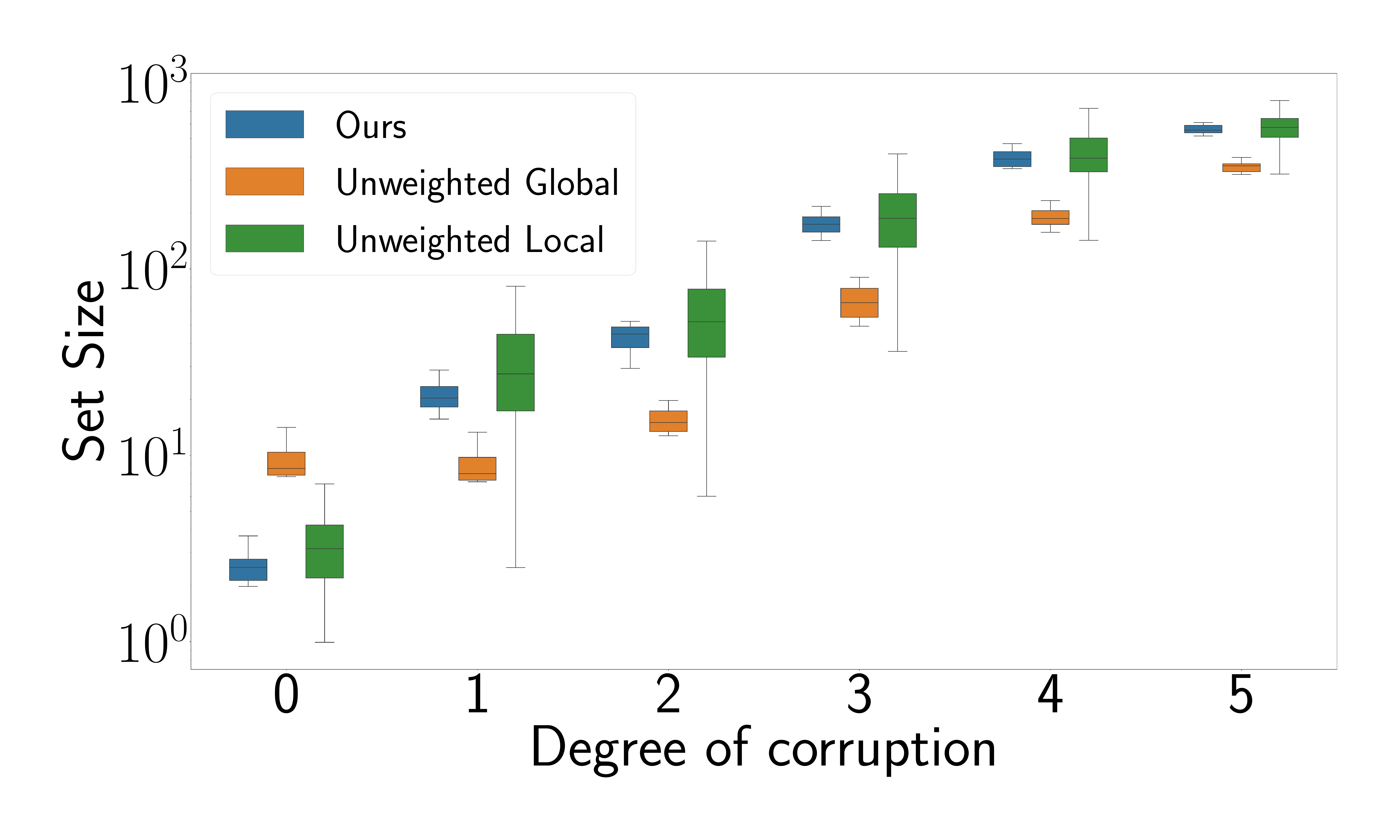}
    \caption{} \label{fig:imagenet_set_size}
  \end{subfigure}
  \caption{ImageNet experimental results: (a) Empirical coverage of conformal prediction sets as a function of data corruption level. It shows how accurately do conformal sets capture the true classes of the data. (b) Average set size of conformal prediction sets as a function of data corruption level. The size of the sets increases with the level of corruption due to the increasing uncertainty of the model based on the corrupted data.}
\end{figure*}

\format{Federated Learning on ImageNet.}
With the ImageNet dataset~\citep{deng2009imagenet}, we perform a similar set of experiments as with CIFAR-$10$ and CIFAR-$100$ datasets, but slightly modifying the experimental setup.
We use ImageNet-C~\citep{hendrycks2019} as the corrupted data with $5$ different corruption levels from $1$ to $5$. Regarding the type of corruption, we select the ``defocus blur'' option. We additionally consider the $0$ level as corresponding to the non-corrupted data, namely the original ImageNet data.
To illustrate the severity of the corruption, the model's accuracy drops from $79\%$ (non-corrupted data, level $0$) to $9\%$ (corrupted data, level $5$).
We use ResNet-$50$ pre-trained on the original ImageNet data.
For the conformal procedure, we set the importance level $\alpha=0.1$ and the temperature scaling parameter is chosen as $T=10$. The APS non-conformity measure~\eqref{eq:def:APS} is used.

We compare the performance of the introduced algorithm with local baselines at $5$ different corruption levels in terms of coverage and average prediction set size.
Data allocation scheme is as follows.
ImageNet and ImageNet-C validation data are distributed among $20$ agents without overlap, such that half of the agents contain clean, non-corrupted data, and the other half consists of corrupted images of the same level of corruption.
Each agent contains $850$ data samples, of which $50$ are calibration samples, $300$ test samples and $500$ density model fitting samples.
We conduct separate experiments with the above-mentioned setup for all levels of corruption.
The results are presented as box plots with means and standard deviations obtained from $10$ random splits of the data for each of these experiments with varying levels of corruption.
The results of the algorithms at each corruption level are obtained by averaging the results from all agents within the same level of corruption.

The Figure~$\ref{fig:imagenet_coverage}$ shows the empirical coverage of the prediction sets at different levels of corruption.
The resulting coverage of the introduced method is valid in comparison with ``Unweighted Global'' baseline, has lower variance and better accuracy compare to the ``Unweighted Local'' one.
The Figure~$\ref{fig:imagenet_set_size}$ demonstrates the average prediction set size of conformal sets as a function of the level of corruption.
The presented method shows on average slightly smaller set sizes and much less variance of the results compared to an ``Unweighted Local'' baseline. 
By using information from other agents through importance weights, our method overcomes the problems of data heterogeneity compared to unweighted alternatives.

\section{Conclusion}
\label{sec:conclusions}

Our work introduces a new method for tailoring prediction sets on non-exchangeable data. 
By leveraging density ratio, the coverage validity is ensured for most calibration datasets. 
This approach opens up new possibilities for applications within the federated framework, where distributional shifts among agents are common.
While our numerical experiments have primarily focused on covariate shift, it is important to emphasize the method's potential applicability to a much broader range of scenarios.

\subsubsection*{Acknowledgements}

The authors gratefully thank Pierre Humbert and Sylvain Arlot for their valuable feedback. The research on experimental evaluation of proposed methods
(Section 5) was supported by the Russian
Science Foundation grant 20-71-10135.
 Part of this work has been carried out under the auspice of the Lagrange Mathematics and Computing Research Center. The work of F. Noskov was prepared within the framework of HSE basic research program.

\bibliography{conformal}

\clearpage
\newpage
\appendix
\onecolumn
\allowdisplaybreaks

\addtocontents{toc}{\protect\setcounter{tocdepth}{2}}
{
  \hypersetup{linkcolor=burntumber}
  \tableofcontents
}

\section{Proof of the conditional coverage}
\label{sec:conditional-coverage}

We assume that the calibration data $\acn{(X_{k},Y_{k})}_{k\in[\tcount+1]}$ are independent and 
\begin{align*}
\forall k\in[\tcount+1], \qquad (X_{k},Y_{k}) \sim P^{k}.
\end{align*}
In this context, $X_{k}$ denotes the covariate, and $Y_{k}$ represents the label. Additionally, we define $Z_{k}=(X_{k},Y_{k})$ to represent the pair consisting of a covariate and a label. The supports of $X_k$ and $Y_k$ are symbolized by $\XC$ and $\YC$ respectively.
Typically, the set $\XC$ is a subset of $\R^d$, whereas $\YC$ may either be finite in the case of classification or represent a subset of $\R^{d}$.
We denote by  $P^{\mathrm{cal}}=(1/\tcount)\sum_{k=1}^{\tcount} P^{k}$ the calibration measure. 
\begin{assumption}\label{ass:Pcal-lambdak}
    The data $\acn{Z_{k}}_{k\in[\tcount+1]}$ are pairwise mutually independent and $P^{\tcount+1}\ll P^{\mathrm{cal}}$.
\end{assumption}
Moreover, we set
\begin{equation}
    \label{eq:def:lambda}
    \begin{aligned}
        &\forall (x,y)\in\XC\times\YC, \quad\lambda(x,y)
        = (\nofrac{\rmd P^{\tcount+1}}{\rmd P^{\mathrm{cal}}})(x,y),
        \\
        &\forall k\in[\tcount+1], \quad \lambda_{k} 
        = \lambda(Z_{k}).
    \end{aligned}
\end{equation}
For $\alpha\in(0,1)$, we denote by $1-\alpha$ the confidence level and $\mathcal{D}_{\tcount}=\{Z_{k}\}_{k\in[\tcount]} \in \prn{\XC\times\YC}^{\tcount}$ the calibration dataset of size $\tcount$.
We denote by $V: \XC \times \YC \to \R^+$ a non-conformity score, and for $\beta \in [0,1]$, $\q{\beta}\prn{\mu}$ represents the $\beta$-quantile of the probability measure $\mu$, and  $\delta_{v}$  the Dirac measure at $v\in\R\cup\{\infty\}$. We define by $p^{(x,y)}_{k}$  the importance weights define by
\begin{align*}
    \forall k\in[\tcount],\quad
    p_{k}^{(x,y)}
    = \frac{\lambda_{k}}{\lambda(x,y) + \sum_{l=1}^{\tcount} \lambda_{l}},&
    &p_{\tcount+1}^{(x,y)}
    = \frac{\lambda(x,y)}{\lambda(x,y) + \sum_{l=1}^{\tcount} \lambda_{l}}.
\end{align*}
In this paper, we consider the prediction set at $x\in\XC$:
\begin{equation*}
    \mathcal{C}_{\alpha,\mu}(x)
    = \ac{y\in\YC\colon V(x, y) \le \q{1-\alpha}\prn{\mu^{(x,y)}}},
\end{equation*}
where 
\begin{equation*}
    \mu^{(x,y)} = \sum_{k=1}^{\tcount} p_{k}^{(x,y)} \delta_{V(Z_{k})} + p_{\tcount+1}^{(x,y)} \delta_{\infty},
\end{equation*}
Given the calibration dataset $\mathcal{D}_{\tcount}$, consider the conditional miscoverage rate
\begin{equation*}
    \alpha(\mathcal{D}_{\tcount})
    = \prob\pr{Y_{\tcount+1} \notin \mathcal{C}_{\alpha,\mu}(X_{\tcount+1}) \mid \mathcal{D}_{\tcount}}.
\end{equation*}
For $z\in\XC\times\YC$, we denote
\begin{align*}
    \widehat{F}_{\tcount+1}(z)
    = \sum_{k=1}^{\tcount} p_{k}^{z} \indiacc{V(Z_{k}) \le V(z)},&
    &F_{\tcount+1}(z) = \E\br{\indiacc{V(Z_{\tcount+1})\le V(z)}}.
\end{align*}
Define by $G_{\tcount+1}$ the cumulative density function of $V(Z_{\tcount+1})$, and consider the associated quantile function:
\begin{equation*}
    \forall \gamma\in[0,1], \qquad
    G_{\tcount+1}^{+}(\gamma)
    = \inf\ac{v\in\R \colon G_{\tcount+1}(v)\ge \gamma}.
\end{equation*}
In addition, define the supremum of the difference between the empirical and the true cumulative distribution functions:
\begin{align*}
    &\Delta(\mathcal{D}_{\tcount}) 
    = \sup_{z\in\XC\times\YC}\ac{\widehat{F}_{\tcount+1}(z) - F_{\tcount+1}(z)},
    \\
    &\tilde{\Delta}(\mathcal{D}_{\tcount}) 
    = \sup_{z\in\XC\times\YC}\ac{F_{\tcount+1}(z) - \widehat{F}_{\tcount+1}(z)}.
\end{align*}

\begin{lemma}\label{lem:validity-delta}
    Assume that there are no ties between the $\{V(Z_k) \colon k\in[\tcount+1]\}$ and $V(Z_{\tcount+1})<\infty$ almost surely.
    For any $\delta>0$, we have
    \begin{align*}
        &\prob\pr{\alpha(\mathcal{D}_{\tcount}) \ge \alpha + \delta + \sup_{v\in\R} \prob\ac{V(Z_{\tcount+1}) = v}}
        \le \prob\pr{\Delta(\mathcal{D}_{\tcount}) \ge \delta}
        ,
        \\
        &\prob\prbig{\alpha(\mathcal{D}_{\tcount}) \le \alpha - \delta}
        \le \prob\prbig{\tilde{\Delta}(\mathcal{D}_{\tcount}) \ge \delta}
        .
    \end{align*}
\end{lemma}

\begin{proof}
    For any $x\in\XC$, we have
    \begin{align*}
        \mathcal{C}_{\alpha,\mu}(x)
        &= \ac{y\in\YC\colon V(x, y) \le \q{1-\alpha}\prn{\mu^{(x,y)}}}
        \\
        &= \ac{y\in\YC\colon V(x, y) < \q{1-\alpha}\prn{\mu^{(x,y)}}}
        \cup \ac{y\in\YC\colon V(x, y) = \q{1-\alpha}\prn{\mu^{(x,y)}}}
        .
    \end{align*}

    Moreover, note that
    \begin{align*}
        \ac{y\in\YC\colon V(x, y) < \q{1-\alpha}\prn{\mu^{(x,y)}}}
        &= \ac{y\in\YC\colon \sum_{k=1}^{\tcount} p_{k}^{(x,y)} \1_{V(Z_{k}) \le V(x, y)} < 1 - \alpha}
        \\
        &= \ac{y\in\YC\colon \widehat{F}_{\tcount+1}(x, y) < 1 - \alpha},
    \end{align*}
    and remark that
    \begin{align*}
        \E\br{\prob\pr{V(Z_{\tcount+1}) = \q{1-\alpha}\prn{\mu^{(Z_{\tcount+1})}} \mid \mathcal{D}_{\tcount}}}
        &\le \prob\pr{V(Z_{\tcount+1}) \in \{V(Z_k)\}_{k\in [\tcount]} \cup \{\infty\}}.
    \end{align*}
    Hence, the assumptions on $ \{V(Z_k)\colon k\in [\tcount+1]\}$ imply that almost surely:
    \[
        \prob\pr{V(Z_{\tcount+1}) = \q{1-\alpha}\prn{\mu^{(Z_{\tcount+1})}} \mid \mathcal{D}_{\tcount}}
        = 0.
    \]
    Therefore, we deduce that
    \begin{align}
        \nonumber
        \alpha(\mathcal{D}_{\tcount})
        &= \prob\pr{Y_{\tcount+1} \notin \mathcal{C}_{\alpha,\mu}(X_{\tcount+1}) \mid \mathcal{D}_{\tcount}}
        \\
        \nonumber
        &= 1 - \prob\pr{V(Z_{\tcount+1}) < \q{1-\alpha}\prn{\mu^{(Z_{\tcount+1})}} \mid \mathcal{D}_{\tcount}}
        - \prob\pr{V(Z_{\tcount+1}) = \q{1-\alpha}\prn{\mu^{(Z_{\tcount+1})}} \mid \mathcal{D}_{\tcount}}
        \\
        \nonumber
        &= \prob\pr{\widehat{F}_{\tcount+1}(Z_{\tcount+1}) \ge 1 - \alpha \mid \mathcal{D}_{\tcount}}
        \\
        \label{eq:eq:miscoverage}
        &= \prob\pr{F_{\tcount+1}(Z_{\tcount+1}) - \pr{F_{\tcount+1}(Z_{\tcount+1}) - \widehat{F}_{\tcount+1}(Z_{\tcount+1})} \ge 1 - \alpha \mid \mathcal{D}_{\tcount}}.
    \end{align}
    For any $\beta\in(0,1)$ and $v\in\R$, remark that $G_{\tcount+1}(v) \ge \beta$ if and only if $v\ge G_{\tcount+1}^{+}(\beta)$. Therefore, using the shorthand notation $V_{N+1}= V(Z_{N+1})$, it follows that
    \begin{align*}
        \prob\pr{G_{\tcount+1}(V_{\tcount+1}) < \beta}
        &= \prob\pr{V_{\tcount+1} < G_{\tcount+1}^{+}(\beta)}
        \\
        &= G_{\tcount+1}\pr{G_{\tcount+1}^{+}(\beta)} - \prob\pr{V_{\tcount+1} = G_{\tcount+1}^{+}(\beta)}.
    \end{align*}
    Therefore, we deduce that
    \begin{align}
        \nonumber
        \prob\pr{F_{\tcount+1}(Z_{\tcount+1}) \ge \beta}
        &= 1 - \prob\pr{G_{\tcount+1}(V_{\tcount+1}) < \beta}
        \\
        \nonumber
        &= 1 - \beta + \ac{\beta - \prob\pr{G_{\tcount+1}(V_{\tcount+1}) < \beta}}
        \\
        \nonumber
        &= 1 - \beta + \ac{\beta - G_{\tcount+1}\pr{G_{\tcount+1}^{+}(\beta)} + \prob\pr{V_{\tcount+1} = G_{\tcount+1}^{+}(\beta)}}
        \\
        \label{eq:eq:G-utility}
        &\in [1 - \beta, 1 - \beta + \prob\pr{V_{\tcount+1} = G_{\tcount+1}^{+}(\beta)}].
    \end{align}
    The previous upper bound also holds for $\beta \le 0$ or $\beta = 1$, whereas the previous lower bound holds for $\beta\ge 0$.
    From \eqref{eq:eq:miscoverage}, we obtain
    \begin{equation*}
        \alpha(\mathcal{D}_{\tcount})
        \le \prob\pr{F_{\tcount+1}(Z_{\tcount+1}) \ge 1 - \alpha - \Delta(\mathcal{D}_{\tcount}) \mid \mathcal{D}_{\tcount}}.
    \end{equation*}
    Therefore, applying \eqref{eq:eq:G-utility} with $\beta= 1 - \alpha - \Delta(\mathcal{D}_{\tcount})$ concludes the first part of the proof since $\beta\in (-\infty, 1]$.
    The second part of the proof follows from \eqref{eq:eq:miscoverage} and \eqref{eq:eq:G-utility}, with $\beta=1 - \alpha + \tilde{\Delta}(\mathcal{D}_{\tcount})$, since $\beta \ge 0$.
\end{proof}

\begin{theorem}\label{thm:coverage-conditional}
    Assume that \Cref{ass:Pcal-lambdak} holds, and suppose there are no ties between the $\{V(Z_k) \colon k\in[\tcount+1]\}$ and $V(Z_{\tcount+1})<\infty$ almost surely.
    If we suppose that $\acn{\lambda_{k}-\E\lambda_k}_{k\in[\tcount]}$ are sub-Gaussian with parameters $\sigma_{1},\ldots,\sigma_{\tcount}\ge 0$, then, for any $\delta>0$ it follows
    \begin{align*}
        &\prob\pr{
            \alpha(\mathcal{D}_{\tcount}) 
            < \alpha
            + \frac{\sqrt{8 \log (\frac{1}{3\delta}) \sum_{k=1}^{\tcount} \prn{4 \sigma_{k}^{2} + \E\brn{\lambda_k}^2}}}{\tcount}
            + \sup_{v\in\R} \prob\ac{V(Z_{\tcount+1}) = v}
        }
        \ge 1 - \delta,
        \\
        &\prob\pr{
            \alpha(\mathcal{D}_{\tcount}) 
            > \alpha
            - \frac{3\E\lambda_{\tcount+1} + \sqrt{8 \log (\frac{1}{3\delta}) \sum_{k=1}^{\tcount} \prn{4 \sigma_{k}^{2} + \E\brn{\lambda_k}^2}}}{\tcount + \E\lambda_{\tcount+1}}
        }
        \ge 1-\delta
        .
    \end{align*}
\end{theorem}

\begin{proof}
    This proof is split in two part. In part one, we demonstrate the first inequality. In part two, we prove the second inequality.

    \paragraph{Proof of part 1.}

    For any $k\in[\tcount+1]$, introduce the weights $q_{k}=\lambda_{k}/\sum_{l=1}^{\tcount}\E\lambda_l$.
    The random variables $\acn{q_{k}}_{k\in[\tcount+1]}$ are mutually pairwise independent approximations of the importance weights $\{p_{k}^{z}\}_{k\in[\tcount+1]}$. 
    Note that
    \begin{multline}\label{eq:bound:Delta:1}
        \Delta(\mathcal{D}_{\tcount}) 
        \le \sup_{z\in\XC\times\YC}\ac{\sum_{k=1}^{\tcount} \pr{p_{k}^{z} - q_{k}} \indiacc{V(Z_{k})\le V(z)}}
        + \sup_{z\in\XC\times\YC}\ac{\sum_{k=1}^{\tcount} \pr{q_{k}\indiacc{V(Z_{k})\le V(z)} - \E\br{q_{k}\indiacc{V(Z_{k})\le V(z)}}}}
        \\
        + \sup_{z\in\XC\times\YC} \ac{\sum_{k=1}^{\tcount} \E\br{q_{k} \indiacc{V(Z_{k}) \le V(z)}} - F_{\tcount+1}(z)}.
    \end{multline}
    Using \Cref{lem:eq:eq:hatG-unbiaised} we obtain
    \begin{equation}\label{eq:bound:Delta:5}
        \sup_{z\in\XC\times\YC} \ac{\sum_{k=1}^{\tcount} \E\br{q_{k} \indiacc{V(Z_{k}) \le V(z)}} - F_{\tcount+1}(z)}
        = 0
        .
    \end{equation}
    Let $\delta$ denote a positive real. 
    Applying \Cref{thm:bound:DKW-revisited} shows that
    \begin{equation}\label{eq:bound:Delta:2}
        \prob\pr{\sup_{z\in\XC\times\YC}\ac{\sum_{k=1}^{\tcount} \pr{q_{k}\indiacc{V(Z_{k})\le V(z)} - \E\br{q_{k}\indiacc{V(Z_{k})\le V(z)}}}} \ge \delta}
        \le 2 \inf_{\theta>0} \ac{e^{- \theta \delta} \prod_{k=1}^{\tcount} \E\br{\cosh\pr{\theta q_{k}}}}.
    \end{equation}
    Let $\theta>0$, and denote by $\acn{\epsilon_k}_{k\in[\tcount]}$ a sequence of i.i.d. Rademacher random variables.
    The independence of $\acn{\lambda_{k}}_{k\in[\tcount]}$ implies that
    \begin{equation*}
        \prod_{k=1}^{\tcount} \E\br{\cosh\pr{\theta q_{k}}}
        = \prod_{k=1}^{\tcount} \pr{2^{-1} \E\br{\exp\pr{\theta q_{k}}} + 2^{-1} \E\br{\exp\pr{-\theta q_{k}}}}
        = \prod_{k=1}^{\tcount} \E\br{\exp\pr{\theta \epsilon_k q_{k}}}.
    \end{equation*}
    For all $x\in\R$, note that $\cosh(x)\le\exp(x^2/2)$. Thus, we deduce that
    \begin{equation*}
        \E\br{\exp\pr{\theta \epsilon_{k} \E{q_{k}}}}
        \le \exp\pr{2^{-1} \theta^{2} \E\brn{q_{k}}^{2}}.
    \end{equation*}
    Using the $\sigma_{k}$-sub-Gaussianity of $\lambda_{k}-\E\lambda_k$, it yields that
    \begin{align}
        \nonumber
        \prod_{k=1}^{\tcount} \E\br{\cosh\pr{\theta q_{k}}}
        &= \E\br{
            \exp\pr{\theta \sum_{k=1}^{\tcount} \epsilon_k \E q_k}
            \E\br{\exp\pr{\theta \sum_{k=1}^{\tcount} \epsilon_{k} \br{q_{k} - \E{q_{k}}}} \,\bigg\vert\, \ac{\epsilon_k}_{k\in[\tcount]} } 
        }
        \\
        \nonumber
        &\le \E\br{
            \exp\pr{\theta \sum_{k=1}^{\tcount} \epsilon_k \E q_k}
        }
        \exp\pr{\frac{2\theta^{2}}{\prn{\sum_{l=1}^{\tcount} \E\lambda_l}^2} \sum_{k=1}^{\tcount} \sigma_k^{2}}
        \\
        \label{eq:bound:Delta:3}
        &\le \exp\pr{\frac{\theta^{2}}{2 \prn{\sum_{l=1}^{\tcount} \E\lambda_l}^2} \sum_{k=1}^{\tcount} \E\brn{\lambda_k}^{2} 
        + \frac{2\theta^{2}}{\prn{\sum_{l=1}^{\tcount} \E\lambda_l}^2} \sum_{k=1}^{\tcount} \sigma_k^{2}}
        .
    \end{align}
    Now, consider the specific choice of $\theta_{\tcount}$ given by
    \[
        \theta_{\tcount}
        = \frac{\delta \prn{\sum_{k=1}^{\tcount} \E\lambda_k}^{2}}{\sum_{k=1}^{\tcount} \pr{4 \sigma_{\tcount}^{2} + \E\brn{\lambda_k}^2}}.
    \]
    Combining \eqref{eq:bound:Delta:3} with the expression of $\theta_{\tcount}$, it follows that
    \begin{align*}
        \inf_{\theta>0} \ac{e^{- \theta \delta} \prod_{k=1}^{\tcount} \E\br{ \cosh\pr{\theta q_k} }}
        &\le \exp\pr{- \delta \theta_{\tcount} + \frac{\theta_{\tcount}^{2} \sum_{k=1}^{\tcount} \prn{4 \sigma_{k}^{2} + \E\brn{\lambda_k}^2}}{2 \prn{\sum_{l=1}^{\tcount} \E\lambda_l}^{2}}}
        \\
        &= \exp\pr{- \frac{\delta^{2} \prn{\sum_{k=1}^{\tcount} \E\lambda_k}^{2}}{2 \sum_{k=1}^{\tcount} \prn{4 \sigma_{k}^{2} + \E\brn{\lambda_k}^2}}}
        .
    \end{align*}
    Furthermore, since we assumed the $\sigma_{k}$-sub-Gaussianity of $\lambda_{k}-\E\lambda_k$, applying \Cref{lem:bound:pk-pkXk:2} with $\tilde{\sigma}_{\tcount}^2=\sum_{k=1}^{\tcount} \sigma_k^2$ implies that
    \begin{equation}\label{eq:bound:Delta:4}
        \prob\pr{\sup_{z\in\XC\times\YC}\ac{\sum_{k=1}^{\tcount} \pr{p_{k}^{z} - q_{k}} \indiacc{V(Z_{k})\le V(z)}} \ge \delta}
        \\
        \le \exp\pr{- \frac{\delta^{2} \prn{\sum_{k=1}^{\tcount} \E\lambda_{k}}^{2}}{2 \sum_{k=1}^{\tcount} \sigma_{k}^{2}}}
        .
    \end{equation}
    Plugging \eqref{eq:bound:Delta:5}-\eqref{eq:bound:Delta:2}-\eqref{eq:bound:Delta:4} into \eqref{eq:bound:Delta:1} gives 
    \begin{equation}\label{eq:bound:Delta:6}
        \prob\pr{\Delta(\mathcal{D}_{\tcount}) \ge \delta}
        \le 2 \exp\pr{- \frac{\delta^{2} \prn{\sum_{k=1}^{\tcount} \E\lambda_k}^{2}}{8 \sum_{k=1}^{\tcount} \prn{4 \sigma_{k}^{2} + \E\brn{\lambda_k}^2}}}
        + \exp\pr{- \frac{\delta^{2} \prn{\sum_{k=1}^{\tcount} \E\lambda_{k}}^{2}}{8 \sum_{k=1}^{\tcount} \sigma_{k}^{2}}}
        .
    \end{equation}
    Let $\gamma >0$, and denote by $\delta_0(\gamma), \delta_1(\gamma)$ the positive real numbers defined by
    \begin{align*}
        \delta_0(\gamma) = \frac{\sqrt{8 \log (\frac{1}{\gamma}) \sum_{k=1}^{\tcount} \prn{4 \sigma_{k}^{2} + \E\brn{\lambda_k}^2}}}{\sum_{k=1}^{\tcount} \E\lambda_{k}},&
        &\delta_1(\gamma) = \frac{\sqrt{8 \log (\frac{1}{\gamma}) \sum_{k=1}^{\tcount} \sigma_{k}^{2}}}{\sum_{k=1}^{\tcount} \E\lambda_{k}}.
    \end{align*}
    With this notation, remark that
    \begin{align*}
        &\forall \delta \ge \delta_0(\gamma), \qquad
        \exp\pr{- \frac{\delta^{2} \prn{\sum_{k=1}^{\tcount} \E\lambda_k}^{2}}{8 \sum_{k=1}^{\tcount} \prn{4 \sigma_{k}^{2} + \E\brn{\lambda_k}^2}}}
        \le \gamma,
        \\
        &\forall \delta \ge \delta_1(\gamma), \qquad
        \exp\pr{- \frac{\delta^{2} \prn{\sum_{k=1}^{\tcount} \E\lambda_k}^{2}}{8 \sum_{k=1}^{\tcount} \sigma_{k}^{2}}}
        \le \gamma.
    \end{align*}
    Therefore, for any $\delta\ge \delta_0(\gamma) \vee \delta_1(\gamma)$, \Cref{eq:bound:Delta:6} shows
    \begin{equation*}
        \prob\pr{\Delta(\mathcal{D}_{\tcount}) \ge \delta}
        \le 3\gamma.
    \end{equation*}
    Finally, combining $\sum_{k=1}^{\tcount} \E\lambda_{k}=\tcount$ with \Cref{lem:validity-delta} concludes the first part of the proof.

    \paragraph{Proof of part 2.}

    For the second part of the proof, consider the weights $\tilde{q}_{k}=\lambda_{k}/\sum_{l=1}^{\tcount+1}\E\lambda_l$.
    Furthermore, remark that
    \begin{multline}\label{eq:bound:Delta-tilde:1}
        \tilde{\Delta}(\mathcal{D}_{\tcount}) 
        \le \sup_{z\in\XC\times\YC}\ac{ \sum_{k=1}^{\tcount} \pr{\tilde{q}_k - p_{k}^{z}} \indiacc{V(Z_{k})\le V(z)} }
        + \sup_{z\in\XC\times\YC}\ac{\sum_{k=1}^{\tcount} \pr{ \E\br{\tilde{q}_{k}\indiacc{V(Z_{k})\le V(z)}} - \tilde{q}_{k}\indiacc{V(Z_{k})\le V(z)} }}
        \\
        + \sup_{z\in\XC\times\YC} \ac{F_{\tcount+1}(z) - \sum_{k=1}^{\tcount} \E\br{\tilde{q}_{k} \indiacc{V(Z_{k}) \le V(z)}}}.
    \end{multline}
    First, using \Cref{lem:bound:pk-pkXk:2}, $\forall \delta\ge \E\lambda_{\tcount+1}/\sum_{k=1}^{\tcount+1} \E\lambda_{k}$ it follows
    \begin{equation}\label{eq:bound:Delta-tilde:5}
        \sup_{z\in\XC\times\YC}\ac{ \sum_{k=1}^{\tcount} \pr{\tilde{q}_k - p_{k}^{z}} \indiacc{V(Z_{k})\le V(z)} }
        \le \exp\pr{- \frac{\prn{ \delta\sum_{k=1}^{\tcount} \E\lambda_{k} + (1-\delta) \E\lambda_{\tcount+1} }^{2}}{2 \sum_{k=1}^{\tcount} \sigma_{k}^{2}}}.
    \end{equation}
    Moreover, \Cref{lem:eq:eq:hatG-unbiaised} shows that
    \begin{align*}
        F_{\tcount+1}(z) - \sum_{k=1}^{\tcount} \E\br{\tilde{q}_{k} \indiacc{V(Z_{k}) \le V(z)}}
        &= \sum_{k=1}^{\tcount} \E\br{\pr{\frac{1}{\sum_{l=1}^{\tcount} \E\lambda_l} - \frac{1}{\sum_{l=1}^{\tcount+1} \E\lambda_l}} \lambda_k \indiacc{V(Z_{k}) \le V(z)}}
        \\
        &= \frac{\E\lambda_{\tcount+1}}{\sum_{l=1}^{\tcount+1} \E\lambda_l} \sum_{k=1}^{\tcount} \frac{\E\br{ \lambda_k \indiacc{V(Z_{k}) \le V(z)} }}{\sum_{l=1}^{\tcount} \E\lambda_l}.
    \end{align*}
    Therefore, we deduce that
    \begin{equation}\label{eq:bound:Delta-tilde:2}
        \sup_{z\in\XC\times\YC} \ac{F_{\tcount+1}(z) - \sum_{k=1}^{\tcount} \E\br{\tilde{q}_{k} \indiacc{V(Z_{k}) \le V(z)}}}
        \le \frac{\E\lambda_{\tcount+1}}{\sum_{l=1}^{\tcount+1} \E\lambda_l}.
    \end{equation}
    Furthermore, taking back \eqref{eq:bound:Delta:3} with $\tilde{q}_k$ instead of $q_k$ demonstrates that 
    \begin{equation}\label{eq:bound:Delta-tilde:3}
        \prod_{k=1}^{\tcount} \E\br{\cosh\pr{\theta \tilde{q}_{k}}}
        \le \exp\pr{\frac{\theta^{2}}{2 \prn{\sum_{l=1}^{\tcount+1} \E\lambda_l}^2} \sum_{k=1}^{\tcount} \E\brn{\lambda_k}^{2} + \frac{2\theta^{2}}{\prn{\sum_{l=1}^{\tcount+1} \E\lambda_l}^2} \sum_{k=1}^{\tcount} \sigma_k^{2}}
        .
    \end{equation}
    Consider the specific choice of $\theta_{\tcount+1}$ given by
    \[
        \theta_{\tcount+1}
        = \frac{\delta \prn{\sum_{k=1}^{\tcount+1} \E\lambda_k}^{2}}{\sum_{k=1}^{\tcount} \pr{4 \sigma_{\tcount}^{2} + \E\brn{\lambda_k}^2}}.
    \]
    Setting $\theta_{\tcount+1}$ into \eqref{eq:bound:Delta-tilde:3}, it yields that
    \begin{align}
        \nonumber
        \inf_{\theta>0} \ac{e^{- \theta \delta} \prod_{k=1}^{\tcount} \E\br{ \cosh\pr{\theta \tilde{q}_k} }}
        &\le \exp\pr{- \delta \theta_{\tcount+1} + \frac{\theta_{\tcount+1}^{2} \sum_{k=1}^{\tcount} \prn{4 \sigma_{k}^{2} + \E\brn{\lambda_k}^2}}{2 \prn{\sum_{l=1}^{\tcount+1} \E\lambda_l}^{2}}}
        \\
        \label{eq:bound:Delta-tilde:6}
        &= \exp\pr{- \frac{\delta^{2} \prn{\sum_{k=1}^{\tcount+1} \E\lambda_k}^{2}}{2 \sum_{k=1}^{\tcount} \prn{4 \sigma_{k}^{2} + \E\brn{\lambda_k}^2}}}
        .
    \end{align}
    Therefore, applying \Cref{thm:bound:DKW-revisited} implies that
    \begin{equation}\label{eq:bound:Delta-tilde:4}
        \prob\pr{\sup_{z\in\XC\times\YC}\ac{\sum_{k=1}^{\tcount} \pr{\tilde{q}_{k}\indiacc{V(Z_{k})\le V(z)} - \E\br{\tilde{q}_{k}\indiacc{V(Z_{k})\le V(z)}}}} \ge \delta}
        \le 2 \exp\pr{- \frac{\delta^{2} \prn{\sum_{k=1}^{\tcount+1} \E\lambda_k}^{2}}{2 \sum_{k=1}^{\tcount} \prn{4 \sigma_{k}^{2} + \E\brn{\lambda_k}^2}}}.
    \end{equation}
    Eventually, plugging \eqref{eq:bound:Delta-tilde:5}-\eqref{eq:bound:Delta-tilde:2} and \eqref{eq:bound:Delta-tilde:6} into \eqref{eq:bound:Delta-tilde:1} gives
    \begin{multline}\label{eq:bound:Delta-tilde:7}
        \prob\pr{\tilde{\Delta}(\mathcal{D}_{\tcount}) \ge \delta}
        \le 2 \exp\pr{- \frac{\prn{\delta - \nofrac{\E\lambda_{\tcount+1}}{\sum_{l=1}^{\tcount+1} \E\lambda_l}}^{2} \prn{\sum_{k=1}^{\tcount+1} \E\lambda_k}^{2}}{8 \sum_{k=1}^{\tcount} \prn{4 \sigma_{k}^{2} + \E\brn{\lambda_k}^2}}}
        \\
        + \exp\pr{- \frac{\acn{
                \prn{\delta - \nofrac{\E\lambda_{\tcount+1}}{\sum_{l=1}^{\tcount+1} \E\lambda_l}} \sum_{k=1}^{\tcount} \E\lambda_{k}
                + \prn{\delta - \nofrac{\E\lambda_{\tcount+1}}{\sum_{l=1}^{\tcount+1} \E\lambda_l} - 2} \E\lambda_{\tcount+1}
            }^{2}}{8 \sum_{k=1}^{\tcount} \sigma_{k}^{2}}}
        + \1_{\frac{\E\lambda_{\tcount+1}}{\sum_{l=1}^{\tcount+1} \E\lambda_l} > \delta}
        .
    \end{multline}
    Let $\gamma >0$, and denote by $\tilde{\delta}_0(\gamma), \tilde{\delta}_1(\gamma)$ the positive real numbers defined by
    \begin{align*}
        \tilde{\delta}_0(\gamma)
        = \frac{\E\lambda_{\tcount+1} + \sqrt{8 \log (\frac{1}{\gamma}) \sum_{k=1}^{\tcount} \prn{4 \sigma_{k}^{2} + \E\brn{\lambda_k}^2}}}{\sum_{k=1}^{\tcount+1} \E\lambda_{k}},&
        &\tilde{\delta}_1(\gamma) 
        = \frac{3\E\lambda_{\tcount+1} + \sqrt{8 \log (\frac{1}{\gamma}) \sum_{k=1}^{\tcount} \sigma_{k}^{2}}}{\sum_{k=1}^{\tcount+1} \E\lambda_{k}}.
    \end{align*}
    With this notation, remark that
    \begin{align*}
        &\forall \delta \ge \tilde{\delta}_0(\gamma), \qquad
        \exp\pr{- \frac{\prn{\delta - \nofrac{\E\lambda_{\tcount+1}}{\sum_{l=1}^{\tcount+1} \E\lambda_l}}^{2} \prn{\sum_{k=1}^{\tcount+1} \E\lambda_k}^{2}}{8 \sum_{k=1}^{\tcount} \prn{4 \sigma_{k}^{2} + \E\brn{\lambda_k}^2}}}
        \le \gamma,
        \\
        &\forall \delta \ge \tilde{\delta}_1(\gamma), \qquad
        \exp\pr{- \frac{\ac{
                \pr{\delta - \frac{\E\lambda_{\tcount+1}}{\sum_{l=1}^{\tcount+1} \E\lambda_l}} \sum_{k=1}^{\tcount} \E\lambda_{k}
                + \pr{\delta - \frac{\E\lambda_{\tcount+1}}{\sum_{l=1}^{\tcount+1} \E\lambda_l} - 2} \E\lambda_{\tcount+1}
            }^{2}}{8 \sum_{k=1}^{\tcount} \sigma_{k}^{2}}}
        \le \gamma,
        \\
        &\forall \delta \ge  \tilde{\delta}_0(\gamma) \wedge \tilde{\delta}_1(\gamma), \qquad
        \1_{\frac{\E\lambda_{\tcount+1}}{\sum_{l=1}^{\tcount+1} \E\lambda_l} > \delta} = 0.
    \end{align*}
    Therefore, for any $\delta\ge \tilde{\delta}_0(\gamma) \vee \tilde{\delta}_1(\gamma)$, \eqref{eq:bound:Delta-tilde:7} shows
    \begin{equation*}
        \prob\pr{\tilde{\Delta}(\mathcal{D}_{\tcount}) \ge \delta} 
        \le 3 \gamma.
    \end{equation*}
    Eventually, combining $\sum_{k=1}^{\tcount} \E\lambda_{k}=\tcount$ with \Cref{lem:validity-delta} concludes the second part of the proof.
\end{proof}

\begin{lemma}\label{lem:eq:eq:hatG-unbiaised}
    For any $k\in[\tcount]$, set $q_{k} = \lambda_{k} / \sum_{l=1}^{\tcount} \E\lambda_l$.
    If \Cref{ass:Pcal-lambdak} holds, then, we have
    \[
        F_{\tcount+1}(z) 
        = \sum_{k=1}^{\tcount} \E\br{q_{k} \indiacc{V(Z_{k}) \le V(z)}}.
    \]
\end{lemma}

\begin{proof}
    Let $g:\XC\times\YC\to \R$ be a bounded measurable function and set $z\in\XC\times\YC$.
    We get
    \begin{align*}
        \E\br{g(Z_{\tcount+1})}
        &= \E\br{\frac{1}{\tcount}\sum_{k=1}^{\tcount} \frac{\rmd P^{k}}{ \rmd P^{\mathrm{cal}}} (Z_{\tcount+1})g(Z_{\tcount+1})}
        \\
        &= \frac{1}{\tcount} \sum_{k=1}^{\tcount} \E\br{\frac{\rmd P^{k}}{\rmd P^{\mathrm{cal}}}(Z_{\tcount+1}) g(Z_{\tcount+1})}
        \\
        &= \frac{1}{\tcount} \sum_{k=1}^{\tcount} \E\br{\frac{\rmd P^{\tcount+1}}{\rmd P^{\mathrm{cal}}}(Z_{k}) g(Z_k)}
        \\
        &= \frac{1}{\tcount} \sum_{k=1}^{\tcount} \E\br{\lambda_{k} g(Z_k)}.
    \end{align*}
    Setting $g\colon \tilde{z} \mapsto 1$ gives $\sum_{k=1}^{\tcount} \E\lambda_k=\tcount$ and setting $g\colon\tilde{z}\mapsto \indiacc{V(\tilde{z})\le V(z)}$ shows that $F_{\tcount+1}(z) = \sum_{k=1}^{\tcount} \E\br{q_{k} \indiacc{V(Z_{k}) \le V(z)}}$.
\end{proof}

\begin{lemma}\label{lem:bound:pk-pkXk:2}
    Assume that $\sum_{k=1}^{\tcount} (\lambda_{k}-\E\lambda_k)$ is $\tilde{\sigma}_{\tcount}$-sub-Gaussian with parameter $\tilde{\sigma}_{\tcount}\ge 0$ and denote $q_{k}=\lambda_{k}/\sum_{l=1}^{\tcount}\E\lambda_l$.
    Then, for any $\delta>0$, it follows
    \begin{equation*}
        \prob\pr{\sup_{z\in\XC\times\YC}\ac{ \sum_{k=1}^{\tcount} \pr{p_{k}^{z} - q_{k}} \indiacc{V(Z_{k})\le V(z)}} \ge \delta}
        \le \exp\pr{- \frac{\delta^{2} \prn{\sum_{k=1}^{\tcount} \E\lambda_{k}}^{2}}{2 \tilde{\sigma}_{\tcount}^{2}}}.
    \end{equation*}
    Moreover, if we consider $\tilde{q}_{k}=\lambda_{k}/\sum_{l=1}^{\tcount+1}\E\lambda_l$. Then, for $\delta\ge \E\lambda_{\tcount+1}/\sum_{l=1}^{\tcount+1}\E\lambda_l$, it holds
    \begin{equation*}
        \prob\pr{\sup_{z\in\XC\times\YC}\ac{\sum_{k=1}^{\tcount} \pr{\tilde{q}_{k} - p_{k}^{z}} \indiacc{V(Z_{k})\le V(z)} } \ge \delta}
        \le \exp\pr{- \frac{\prn{\delta \sum_{k=1}^{\tcount}\E\lambda_k + (\delta-1)\E\lambda_{\tcount+1}}^{2}}{2 \tilde{\sigma}_{\tcount}^{2}}}
        .
    \end{equation*}
\end{lemma}

\begin{proof}
    Let $z\in\XC\times\YC$, and to simplify notation consider $\Lambda_{\tcount}(z):= \sum_{l=1}^{\tcount} \lambda_l + \lambda(z)$. By calculation, we obtain
    \begin{align}
        \nonumber
        \sum_{k=1}^{\tcount} \pr{p_{k}^{z} - q_{k}} \indiacc{V(Z_{k})\le V(z)}
        &= \sum_{k=1}^{\tcount} \frac{\sum_{l=1}^{\tcount} \E\lambda_l - \Lambda_{\tcount}(z)}{\prn{\sum_{l=1}^{\tcount} \E\lambda_l} \Lambda_{\tcount}(z)} \lambda_{k} \indiacc{V(Z_{k})\le V(z)}
        \\
        \nonumber 
        &= \pr{ \frac{\sum_{l=1}^{\tcount} (\E\lambda_l - \lambda_l)}{\sum_{l=1}^{\tcount} \E\lambda_l} - \frac{\lambda(z)}{\sum_{l=1}^{\tcount} \E\lambda_l}} \times \frac{\sum_{k=1}^{\tcount} \lambda_{k} \indiacc{V(Z_{k})\le V(z)}}{\Lambda_{\tcount}(z)}
        \\
        \label{eq:bound:pk-pkXk:3} 
        &= \pr{B - A} C,
    \end{align}
    where we denote 
    \begin{align*}
        A = \frac{\lambda(z)}{\sum_{k=1}^{\tcount} \E\lambda_k},&
        &B = \frac{\sum_{k=1}^{\tcount} (\E\lambda_k - \lambda_k)}{\sum_{k=1}^{\tcount} \E\lambda_k},&
        &C = \frac{\sum_{k=1}^{\tcount} \lambda_{k} \indiacc{V(Z_{k})\le V(z)}}{\Lambda_{\tcount}(z)}.
    \end{align*}
    Since $A\ge 0$ and $C \in[0,1]$, it follows that
    \begin{align*}
        (B - A) C
        &\le \max\ac{0, B-A}
        \\
        &\le \max\ac{0, B}
        = \max\ac{0, 1 - \frac{\sum_{k=1}^{\tcount} \lambda_{k}}{\sum_{k=1}^{\tcount} \E\lambda_k}}.
    \end{align*}
    Therefore, using that $\sum_{k=1}^{\tcount} (\lambda_{k}-\E\lambda_k)$ is $\tilde{\sigma}_{\tcount}$-sub-Gaussian, it yields
    \begin{align*}
        \prob\pr{\sup_{z\in\XC\times\YC}\ac{\sum_{k=1}^{\tcount} \pr{p_{k}^{z} - q_{k}} \indiacc{V(Z_{k})\le V(z)}} \ge \delta}
        &\le \prob\pr{\max\ac{0, 1 - \frac{\sum_{k=1}^{\tcount} \lambda_{k}}{\sum_{k=1}^{\tcount} \E\lambda_k}} \ge \delta}
        \\
        &\le \prob\pr{\sum_{k=1}^{\tcount} \lambda_{k} \le (1-\delta) \sum_{k=1}^{\tcount} \E\lambda_k} 
        \\
        &\le \exp\pr{- \frac{\delta^{2} \prn{\sum_{k=1}^{\tcount} \E\lambda_{k}}^{2}}{2 \tilde{\sigma}_{\tcount}^{2}}}.
    \end{align*}
    This proves the first part of the lemma.
    Now, we will demonstrate the second part of the lemma.
    Consider $S_{\tcount}=\sum_{k=1}^{\tcount}\lambda_k$, $E_{\tcount+1}=\sum_{k=1}^{\tcount+1}\E\lambda_k$ and remark that
    \begin{align*}
        \sum_{k=1}^{\tcount} \pr{\tilde{q}_{k}^{z} - p_{k}} \indiacc{V(Z_{k})\le V(z)}
        &= \sum_{k=1}^{\tcount} \frac{E_{\tcount+1} - S_{\tcount} - \lambda(z)}{E_{\tcount+1} \prn{S_{\tcount} + \lambda(z)}} \lambda_{k} \indiacc{V(Z_{k})\le V(z)}
        \\ 
        &= \frac{E_{\tcount+1}-S_{\tcount}-\lambda(z)}{E_{\tcount+1}} \times \frac{\sum_{k=1}^{\tcount} \lambda_{k} \indiacc{V(Z_{k})\le V(z)}}{S_{\tcount} + \lambda(z)}
        \\
        &\le \max\ac{0, 
            \frac{E_{\tcount+1}-S_{\tcount}-\lambda(z)}{E_{\tcount+1}} \times \frac{S_{\tcount}}{S_{\tcount} + \lambda(z)}
        }
        .
    \end{align*}
    Since the function $f: \lambda\in\R_+ \mapsto \frac{E_{\tcount+1}-S_{\tcount}-\lambda}{E_{\tcount+1}} \times \frac{S_{\tcount}}{S_{\tcount} + \lambda}\in\R$ is non-increasing, we deduce that
    \[
        \frac{E_{\tcount+1}-S_{\tcount}-\lambda}{E_{\tcount+1}} \times \frac{S_{\tcount}}{S_{\tcount} + \lambda} \le 1 - \frac{S_{\tcount}}{E_{\tcount+1}}
    \]
    Thus, we obtain that
    \begin{align*}
        \prob\pr{\sup_{z\in\XC\times\YC}\ac{\sum_{k=1}^{\tcount} \pr{\tilde{q}_{k} - p_{k}^{z}} \indiacc{V(Z_{k})\le V(z)}} \ge \delta}
        &\le \prob\pr{ \max\ac{0, 
            \frac{E_{\tcount+1}-S_{\tcount}-\lambda(z)}{E_{\tcount+1}} \times \frac{S_{\tcount}}{S_{\tcount} + \lambda(z)}
        } \ge \delta }
        \\
        &= \prob\pr{ \frac{E_{\tcount+1}-S_{\tcount}-\lambda(z)}{E_{\tcount+1}} \times \frac{S_{\tcount}}{S_{\tcount} + \lambda(z)} \ge \delta }
        \\
        &\le \prob\pr{ 1 - \frac{S_{\tcount}}{E_{\tcount+1}} \ge \delta }
        \\
        &= \prob\pr{ \sum_{k=1}^{\tcount} \prn{\E\lambda_k - \lambda_k} \ge \delta \sum_{k=1}^{\tcount}\E\lambda_k + (\delta-1)\E\lambda_{\tcount+1} }
        .
    \end{align*}
    For $\delta \ge \E\lambda_{\tcount+1}/\sum_{k=1}^{\tcount+1} \E\lambda_k$, using the Hoeffding's inequality combined with the previous inequality concludes the proof: 
    \begin{equation*}
        \prob\pr{ \sum_{k=1}^{\tcount} \prn{\E\lambda_k - \lambda_k} \ge \delta \sum_{k=1}^{\tcount}\E\lambda_k + (\delta-1)\E\lambda_{\tcount+1} }
        \le \exp\pr{- \frac{\prn{\delta \sum_{k=1}^{\tcount}\E\lambda_k + (\delta-1)\E\lambda_{\tcount+1}}^{2}}{2 \tilde{\sigma}_{\tcount}^{2}}}
        .
    \end{equation*}
\end{proof}

\section{Extension of the DKW theorem}
\label{sec:general-dkw}

Consider $\{g_{k}\colon z\mapsto g_{k}(z)\}_{k\in[\tcount]}$ a family of real valued functions such that the random variables $\{g_{k}(Z_{k})\}_{k\in[\tcount]}$ are mutually pairwise independent, and define the empirical cumulative function given for $z\in\ZC$, by
\begin{align*}
    \hat{G}(z) = \sum_{k=1}^{\tcount} g_{k}(Z_{k}) \1_{V(Z_{k}) \le V(z)},&
    &G(z) = \E\brbig{\hat{G}(z)}. 
\end{align*}
In this paragraph we control $\prob\prn{\sup_{z \in \ZC}\acn{\hat{G}_{\tcount}(z)-G(z)} \ge \epsilon}$. The tools utilized in this proof are closed to the ones developed in~\citep{dvoretzky1956asymptotic,massart1990tight,bartl2023variance}.

\begin{theorem}\label{thm:bound:DKW-revisited}
    For any  $\epsilon>0$, the following inequality holds
    \begin{equation*}
        \prob\pr{\sup_{z \in \ZC}\ac{\hat{G}_{\tcount}(z)-G(z)} \ge \epsilon}
        \le 2 \inf _{\theta>0} \ac{\rme^{- \theta \epsilon} \prod_{k=1}^{\tcount} \E\br{\cosh\pr{\theta g_{k}(Z_{k})}}}.
    \end{equation*}
\end{theorem}

\begin{proof}
    First, for any $\theta>0$, applying Markov's inequality gives
    \begin{equation}\label{eq:bound:markov-theta}
        \prob\pr{\sup_{z \in \ZC} \ac{\hat{G}_{\tcount}(z)-G(z)} \ge \epsilon}
        \le \rme^{-\theta \epsilon} \E\br{\exp \pr{\theta \sup_{z \in \ZC} \ac{\hat{G}_{\tcount}(z)-G(z)}}}.
    \end{equation}
    Moreover, \Cref{lem:bound:symmetrization} shows that
    \begin{equation*}
        \E\br{\exp \pr{\theta \sup_{z \in \ZC}\ac{\hat{G}_{\tcount}(z)-G(z)}}} 
        \le 2 \prod_{k=1}^{\tcount} \E\br{\cosh\pr{\theta g_{k}(Z_{k})}}.
    \end{equation*}
    Plugging the previous inequality into \eqref{eq:bound:markov-theta}, and minimizing the resulting expression with respect to $\theta$ yields:
    \begin{equation*}
        \prob\pr{\sup_{z \in \ZC}\ac{\hat{G}_{\tcount}(z)-G(z)} \ge \epsilon}
        \le 2 \inf _{\theta>0} \ac{\rme^{- \theta \epsilon} \prod_{k=1}^{\tcount} \E\br{\cosh\pr{\theta g_{k}(Z_{k})}}}.
    \end{equation*}
\end{proof}

\begin{lemma}\label{lem:bound:weighted-rademacher}
    Let $\acn{\epsilon_{i}}_{i\in[n]}$ be i.i.d Rademacher random variables taking values in $\{-1,1\}$, then for any $\theta>0$ and $\{p_{j}\}_{j\in[\tcount]}\in\R^{\tcount}$, we have
    \[
        \E\br{\exp \pr{\theta \sup_{0\le i \le \tcount} \sum_{j=1}^i p_{j} \epsilon_{j}}} \le 2 \prod_{k=1}^{\tcount} \cosh\pr{\theta p_{k}}.
    \]
    By convention, we consider $\sum_{j=1}^0 p_{j} \epsilon_{j}=0$.
\end{lemma}

\begin{proof}
    First we will show that
    \begin{align}\label{eq:bound:prob-rademacher:1}
        \forall t \in \R_{+},&
        &\prob\pr{\max_{i=0}^{\tcount} \sum_{j=1}^i p_{j} \epsilon_{j} \ge t} \le 2 \prob\pr{\sum_{j=1}^{\tcount} p_{j} \epsilon_{j} \ge t}.
    \end{align}
    To prove the previous inequality, for any $i\in\{1,\ldots,\tcount\}$ consider the following set
    \begin{equation*}
        E_{i}
        = \ac{\sum_{j=1}^i p_{j} \epsilon_{j} \ge t} \bigcap_{l=1}^{i-1} \ac{\sum_{j=1}^l p_{j} \epsilon_{j} < t},
    \end{equation*}
    and write
    \begin{equation*}
        E_0 =
        \begin{cases}
            \emptyset & \text{if } t > 0
            \\
            \Omega & \text{if } t\le 0
        \end{cases}.
    \end{equation*}
    Observe that $\{E_{i}\}_{i=0}^{\tcount}$ are pairwise disjoint and also that
    \begin{equation}\label{eq:eq:cup-Ei}
        \ac{\max_{i=0}^{\tcount} \sum_{j=1}^i p_{j} \epsilon_{j} \ge t}
        = \bigcup_{i=0}^{\tcount} E_{i}.
    \end{equation}
    Secondly, note that
    \begin{equation}\label{eq:bound:inclusion-Ei}
        \bigcup_{i=0}^{\tcount} \pr{E_{i} \cap\ac{\sum_{j=i+1}^{\tcount} p_{j} \epsilon_{j} \ge 0}} \subset\ac{\sum_{j=1}^{\tcount} p_{j} \epsilon_{j} \ge t}.
    \end{equation}
    Since $\sum_{j=i+1}^{\tcount} p_{j} \epsilon_{j}$ is symmetric around $0$, we have
    \begin{equation}\label{eq:bound:sum-eps-demi}
        \prob\pr{\sum_{j=i+1}^{\tcount} p_{j} \epsilon_{j} \ge 0} \ge \frac{1}{2}.
    \end{equation}
    Moreover, by independence, we get
    \begin{equation}\label{eq:bound:prob-Ei}
            \prob\pr{E_{i} \cap\ac{\sum_{j=i+1}^{\tcount} p_{j} \epsilon_{j} \ge 0}} 
            = \prob\pr{E_{i}} \prob\pr{\sum_{j=i+1}^{\tcount} p_{j} \epsilon_{j} \ge 0} 
            \ge \frac{\prob\pr{E_{i}}}{2}.
    \end{equation}
    Thus, combining \eqref{eq:eq:cup-Ei} and \eqref{eq:bound:inclusion-Ei} with \eqref{eq:bound:prob-Ei} implies that
    \begin{equation*}
        \begin{aligned}
            \prob\pr{\sum_{j=1}^{\tcount} p_{j} \epsilon_{j} \ge t}
            &\ge \sum_{i=0}^{\tcount} \prob\pr{E_{i} \cap\ac{\sum_{j=i+1}^{\tcount} p_{j} \epsilon_{j} \ge 0}} 
            \\
            &\ge \sum_{i=0}^{\tcount} \frac{\prob\pr{E_{i}}}{2}
            = \frac{1}{2} \prob\pr{\cup_{i=0}^{\tcount} E_{i}}
            \\
            &= \frac{1}{2} \prob\ac{\sup_{i=0}^{\tcount} \sum_{j=1}^i p_{j} \epsilon_{j} \ge t}.
        \end{aligned}
    \end{equation*}
    Therefore, the last line concludes the proof of \eqref{eq:bound:prob-rademacher:1}.
    Note that for any differentiable function $f$ such that $f'\ge 0$, the Fubini-Tonelli's theorem shows that
    \begin{equation}\label{eq:bound:prob-rademacher:2}
        \begin{aligned}
            \E\br{f(U) \1_{\{U \ge 0\}}} 
            &= \E\br{\pr{f(0)+\int_0^{U} f^{\prime}(t) \rmd t} \1_{\{U \ge 0\}}} 
            \\
            &= f(0) \E\br{\1_{\{U \ge 0\}}}+\E\br{\int_0^{\infty} f^{\prime}(t) \1_{\{U \ge t\}} \rmd t \1_{\{U \ge 0\}}} 
            \\
            &= f(0) \prob\pr{U \ge 0}+\int_0^{\infty} f^{\prime}(t) \prob\pr{U \ge t} \rmd t.
        \end{aligned}
    \end{equation}
    Now, we apply this previous formula to the function $f(t)=\rme^{\theta t}$ and the random variable $U = \max_{i=0}^{\tcount} \sum_{j=1}^i p_{j} \epsilon_{j}$. We get
    \begin{equation}\label{eq:bound:E-rademacher:1}
        \begin{aligned}
            \E\br{f\pr{U}} 
            &= \E\br{f(U) \1_{U < 0}} + \E\br{f(U) \1_{U \ge 0}}
            \\
            &= \theta \int_0^{\infty} \rme^{\theta t} \prob\pr{U \ge t} \rmd t 
            \\
            &\overset{\eqref{eq:bound:prob-rademacher:1}}{\le} 1+2 \theta \int_0^{\infty} \rme^{\theta t} \prob\pr{\sum_{j=1}^{\tcount} p_{j} \epsilon_{j} \ge t} \rmd t.
        \end{aligned}
    \end{equation}
    Once again, applying \eqref{eq:bound:prob-rademacher:2} with $V=\sum_{j=1}^{\tcount} p_{j} \epsilon_{j}$, it yields
    \begin{equation}\label{eq:bound:E-rademacher:2}
        \begin{aligned}
            \theta \int_0^{\infty} \rme^{\theta t} \prob\pr{V \ge t} \rmd t
            &= \E\br{f(V) \1_{\ac{V \ge 0}}} - \prob\pr{V \ge 0}
            \\
            &\overset{\eqref{eq:bound:sum-eps-demi}}{\le} \E\br{f(V)} - \frac{1}{2}.
        \end{aligned}
    \end{equation}
    Hence, plugging \eqref{eq:bound:E-rademacher:2} inside \eqref{eq:bound:E-rademacher:1} gives
    \begin{equation*}
        \E\br{\exp\pr{\theta \max_{i=0}^{\tcount} \sum_{j=1}^i p_{j} \epsilon_{j}}}
        \le 2 \E\br{\exp\pr{\theta \sum_{j=1}^{\tcount} p_{j} \epsilon_{j}}} 
        = 2\prod_{j=1}^{\tcount} \E\br{\exp\pr{\theta p_{j} \epsilon_{j}}}.
    \end{equation*}
    The proof is finished using that $\E\brn{\exp\prn{\theta p_{j} \epsilon_{j}}} = \cosh (\theta p_{j})$.
\end{proof}

\begin{lemma}\label{lem:bound:symmetrization}
    Let $\theta>0$, we have
    \begin{equation*}
        \E\br{\exp\pr{\theta \sup_{z \in \ZC}\ac{\hat{G}_{\tcount}(z)-G(z)}}}
        \le 2 \prod_{k=1}^{\tcount} \E\br{\cosh\pr{\theta g_{k}(Z_{k})}}.
    \end{equation*}
\end{lemma}

\begin{proof}
    Let $\theta>0$ be fixed, since $t\mapsto \rme^{\theta t}$ is increasing, the supremum can be inverted with the exponential:
    \begin{equation*}
        \E\br{\exp\pr{\theta \sup_{z \in \ZC}\ac{\hat{G}_{\tcount}(z)-G(z)}}}
        = \E\br{\sup_{z \in \ZC} \exp\pr{\theta \ac{\hat{G}_{\tcount}(z)-G(z)}}}.
    \end{equation*}
    For any $k\in[\tcount]$, consider $\tilde{Z}_{k}$ an independent copy of the random variable $Z_{k}$. The linearity of the expectation gives
    \begin{multline*}
        \sum_{k=1}^{\tcount} \pr{g_{k}(Z_{k}) \1_{V(Z_{k}) \le V(z)} - \E\br{g_{k}(Z_{k}) \1_{V(Z_{k}) \le V(z)}}}
        \\
        = \E\br{\sum_{k=1}^{\tcount} \pr{g_{k}(Z_{k}) \1_{V(Z_{k}) \le V(z)} - g_{k}(\tilde{Z}_{k}) \1_{V(\tilde{Z}_{k}) \le V(z)}} \,\Big\vert\, \acn{Z_{k}}_{k=1}^{\tcount}}.
    \end{multline*}
    Therefore, the Jensen's inequality implies
    \begin{align*}
        &\E\br{\exp\pr{\theta \sup_{z \in \ZC}\ac{\hat{G}_{\tcount}(z)-G(z)}}}
        \\
        &= \E\br{\sup_{z \in \ZC} \exp\pr{\theta \E\br{\sum_{k=1}^{\tcount} \ac{g_{k}(Z_{k}) \1_{V(Z_{k}) \le V(z)} - g_{k}(\tilde{Z}_{k}) \1_{V(\tilde{Z}_{k}) \le V(z)}} \,\bigg\vert\, \acn{Z_{k}}_{k=1}^{\tcount}}}}
        \\
        &\le \E\br{\sup_{z \in \ZC} \exp\pr{\theta \sum_{k=1}^{\tcount} \ac{g_{k}(Z_{k}) \1_{V(Z_{k}) \le V(z)} - g_{k}(\tilde{Z}_{k}) \1_{V(\tilde{Z}_{k}) \le V(z)}}}}.
    \end{align*}
    Let $\{\epsilon_{k}\}_{k\in[\tcount]}$ be i.i.d. random Rademacher variables independent of $\{(Z_{k},\tilde{Z}_{k})\}_{k=1}^{\tcount}$, we have
    \begin{multline*}
        \E\br{\sup_{z \in \ZC} \exp\pr{\theta \sum_{k=1}^{\tcount} \ac{g_{k}(Z_{k}) \1_{V(Z_{k}) \le V(z)} - g_{k}(\tilde{Z}_{k}) \1_{V(\tilde{Z}_{k}) \le V(z)}}}}
        \\
        = \E\br{\sup_{z \in \ZC} \exp\pr{\theta \sum_{k=1}^{\tcount} \epsilon_{k} \ac{g_{k}(Z_{k}) \1_{V(Z_{k}) \le V(z)} - g_{k}(\tilde{Z}_{k}) \1_{V(\tilde{Z}_{k}) \le V(z)}}}}.
    \end{multline*}
    Using the Cauchy-Schwarz's inequality, we deduce that
    \begin{equation*}
        \E\br{\exp \pr{\theta \sup_{z \in \ZC}\ac{\hat{G}_{\tcount}(z)-G(z)}}} 
        \le \E\br{\sup_{z \in \ZC} \exp\pr{2 \theta \sum_{k=1}^{\tcount} \epsilon_{k} g_{k}(Z_{k}) \1_{V(Z_{k}) \le V(z)}}}.
    \end{equation*}
    Given the random variables $\{V(Z_{k})\}_{k=1}^{\tcount}$, denote by $\sigma$ the permutation of $[\tcount]$ such that $V(Z_{\sigma(1)})\le\cdots\le V(Z_{\sigma(\tcount)})$. In particular, it holds
    \begin{equation*}
        \sum_{k=1}^{\tcount} \epsilon_{k} g_{k}(Z_{k}) \1_{V(Z_{k}) \le V(z)}
        = 
        \begin{cases}
            0 & \text { if } V(z) < V(Z_{\sigma(1)}) 
            \\ 
            \sum_{j=1}^i  \epsilon_{\sigma(j)} g_{\sigma(j)}(Z_{(j)}) & \text { if } V(Z_{\sigma(i)}) \le V(z) < V(Z_{\sigma(i+1)})
            \\ 
            \sum_{j=1}^{\tcount} \epsilon_{\sigma(j)} g_{\sigma(j)}(Z_{(j)}) & \text { if } V(z) \ge V(Z_{\sigma(\tcount)})
        \end{cases}.
    \end{equation*}
    Thus, can rewrite the supremum as
    \begin{equation*}
        \sup_{z \in \ZC} \exp\pr{2 \theta \sum_{k=1}^{\tcount} \epsilon_{k} g_{k}(Z_{k}) \1_{V(Z_{k}) \le V(z)}}
        \le \sup_{0 \le i \le n} \exp \pr{2 \theta \sum_{j=1}^i \epsilon_{\sigma(j)} g_{\sigma(j)}(Z_{(j)})}.
    \end{equation*}
    Applying \Cref{lem:bound:weighted-rademacher}, we finally obtain that
    \begin{multline*}
        \E\br{\sup_{z \in \ZC} \exp\pr{2 \theta \sum_{k=1}^{\tcount} \epsilon_{k} g_{k}(Z_{k}) \1_{V(Z_{k}) \le V(z)}} \,\bigg\vert\, \acn{Z_{k}}_{k=1}^{\tcount}} 
        \\
        \le \E\br{\sup_{0 \le i \le \tcount} \exp \pr{2 \theta \sum_{j=1}^i \epsilon_{\sigma(j)} g_{\sigma(j)}(Z_{(j)})} \,\bigg\vert\, \acn{Z_{k}}_{k=1}^{\tcount}}
        \le 2 \prod_{k=1}^{\tcount} \cosh\pr{\theta g_{k}(Z_{k})}.
    \end{multline*}
\end{proof}

\section{Proof of \Cref{theorem:bias}}
\label{sec:bias}

The objective of this section is to investigate the bias introduced by the conformal procedure developed in \Cref{sec:approach}. 
Due to distribution shift, studying this bias is different from the case of exchangeable data. 
To address this, we introduce a sequence $\{\tilde{Z}_k\}_{k\in[\tcount]}$ of independent and identically distributed random variables following the distribution $P^{\mathrm{cal}}$. 
To simplify notation, we consider $\tilde{Z}_{\tcount+1}=Z_{\tcount+1}$.
Note that, unlike the other data points, $\tilde{Z}_{\tcount+1}$ is drawn from $P^{\tcount+1}$.
For all $k\in[\tcount+1]$, we define:
\begin{align}\label{eq:def:qweighthat}
    &p_{k} = \frac{\lambda\prn{Z_k}}{\sum_{l=1}^{\tcount+1} \lambda\prn{Z_{l}}},&
    &q_k = \frac{\lambda\prn{\tilde{Z}_k}}{\sum_{l=1}^{\tcount+1} \lambda\prn{\tilde{Z}_{l}}}.
\end{align}
Furthermore, let's introduce the following two random variables:
\begin{equation}\label{eq:def:X-delta}
    \begin{aligned}
        &\Gamma 
        = \sum_{k=1}^{\tcount} q_{k} \1_{\tilde{V}_k<V_{\tcount+1}}, 
        &\Delta 
        = \sum_{k=1}^{\tcount} p_k \1_{V_k<V_{\tcount+1}}
        - \sum_{k=1}^{\tcount} q_{k} \1_{\tilde{V}_k<V_{\tcount+1}},
    \end{aligned}
\end{equation}
where $\tilde{V}_k=V(\tilde{Z}_k)$, $V_k=V(Z_k)$ and $Z_k=(X_k,Y_k)\sim P^k$.

\begin{lemma}\label{lem:probVN-infQuantile}
    Assume there are no ties between there are no ties between $\{V_k\}_{k\in[\tcount+1]}\cup\{\infty\}$ almost surely.
    For any $\epsilon\ge 0$, it holds
    \begin{multline*}\label{eq:bounds:probX-alpha-eps:2}
        \textstyle
        - \epsilon - \prob\pr{\Delta \le -\epsilon}
        \le \prob\pr{V_{\tcount+1}\le \q{1 - \alpha}\pr{\textstyle\sum_{k=1}^{\tcount} p_{k} \delta_{V_k} + p_{\tcount+1} \delta_{\infty}}} - 1 + \alpha 
        \le \epsilon + \prob\pr{\Delta \ge \epsilon} + \E\br{\max_{k=1}^{\tcount+1}\acn{q_{k}}}.
    \end{multline*}
\end{lemma}
\begin{proof}
    First, using the definition of the quantile combined with~\eqref{eq:def:X-delta}, we have
    \begin{equation*}
        \textstyle
        \pr{V_{\tcount+1}< \q{1 - \alpha}\pr{\textstyle\sum_{k=1}^{\tcount} p_{k} \delta_{V_k} + p_{\tcount+1} \delta_{\infty}}}
        \iff \pr{\Gamma+\Delta > \alpha}.
    \end{equation*}
    Therefore, using the no ties assumption on $\{V_k \colon k\in[\tcount+1]\}\cup \{\infty\}$, it holds that
    \begin{align}
        \nonumber
        \textstyle
        \prob\pr{V_{\tcount+1}\le \q{1 - \alpha}\pr{\textstyle\sum_{k=1}^{\tcount} p_{k} \delta_{V_k} + p_{\tcount+1} \delta_{\infty}}}
        &= \prob\pr{V_{\tcount+1}< \q{1 - \alpha}\pr{\textstyle\sum_{k=1}^{\tcount} p_{k} \delta_{V_k} + p_{\tcount+1} \delta_{\infty}}}
        \\
        \label{eq:probVN-indic}
        &= \E\br{\1_{\Gamma>\alpha}}
        + \E\br{\1_{\Gamma+\Delta>\alpha} - \1_{\Gamma>\alpha}}.
    \end{align}
    Moreover, for any $\epsilon\ge 0$, remark that
    \begin{equation}\label{eq:bounds:indic-X-delta}
        \1_{\Gamma > \alpha+\epsilon} - \1_{\Delta\le -\epsilon}
        \le \1_{\Gamma+\Delta>\alpha}
        \le \1_{\Gamma > \alpha-\epsilon} + \1_{\Delta\ge \epsilon}.
    \end{equation}
    Thus, combining~\eqref{eq:probVN-indic} with~\eqref{eq:bounds:indic-X-delta} gives
    \begin{equation}\label{eq:bounds:probX-alpha-eps}
        \textstyle
        \prob\pr{\Gamma>\alpha+\epsilon} - \prob\pr{\Delta \le -\epsilon}
        \le \prob\pr{V_{\tcount+1}\le \q{1 - \alpha}\pr{\textstyle\sum_{k=1}^{\tcount} p_{k} \delta_{V_k} + p_{\tcount+1} \delta_{\infty}}}
        \le \prob\pr{\Gamma>\alpha-\epsilon} + \prob\pr{\Delta \ge \epsilon}.
    \end{equation}
    Consider $\tilde{\alpha}\in (0,1)$, it holds
    \begin{equation*}
        (X > \tilde{\alpha})
        \iff \pr{V_{\tcount+1}< \q{1 - \tilde{\alpha}}\pr{\textstyle\sum_{k=1}^{\tcount} q_{k} \delta_{V_k} + q_{\tcount+1} \delta_{\infty}}}.
    \end{equation*}
    Once again, the no ties assumption implies that
    \begin{align*}
        \prob\prn{X> \tilde{\alpha}}
        &= \prob\prn{V_{\tcount+1}< \q{1 - \tilde{\alpha}}\prn{\textstyle\sum_{k=1}^{\tcount} q_{k} \delta_{V_k} + q_{\tcount+1} \delta_{V_{\infty}}}}
        \\
        &= \prob\prn{V_{\tcount+1}\le \q{1 - \tilde{\alpha}}\prn{\textstyle\sum_{k=1}^{\tcount} q_{k} \delta_{V_k} + q_{\tcount+1} \delta_{V_{\infty}}}}
        \\
        &= \prob\prn{V_{\tcount+1}\le \q{1 - \tilde{\alpha}}\prn{\textstyle\sum_{k=1}^{\tcount+1} q_{k} \delta_{V_k}}}.
    \end{align*}
    Applying the result from~\citep[Lemma~3]{tibshirani2019conformal}, it gives that
    \begin{equation*}
        0\le \prob\pr{V_{\tcount+1}\le \q{1 - \tilde{\alpha}}\pr{\textstyle\sum_{k=1}^{\tcount+1} q_{k} \delta_{V_k}}} - 1 + \tilde{\alpha}
        \le \E\br{\max_{k=1}^{\tcount+1}\acn{q_{k}}}.
    \end{equation*}
    Therefore, setting $\tilde{\alpha}$ as follows:
    \begin{equation*}
        \begin{cases}
            \tilde{\alpha} = \alpha + \epsilon, & \forall \epsilon\in [0,1-\alpha)
            \\
            \tilde{\alpha} = \alpha - \epsilon, & \forall \epsilon\in [0,\alpha)
        \end{cases};
    \end{equation*}
    using~\eqref{eq:bounds:probX-alpha-eps} we obtain
    \begin{equation*}
        \begin{cases}
            \textstyle 1 - (\alpha + \epsilon) - \prob\pr{\Delta\le - \epsilon} \le \prob\pr{V_{\tcount+1}\le \q{1 - \alpha}\prbig{\textstyle\sum_{k=1}^{\tcount} p_{k} \delta_{V_k} + p_{\tcount+1} \delta_{\infty}}}, & \forall \epsilon\in [0,1-\alpha)
            \\
            \textstyle \prob\pr{V_{\tcount+1}\le \q{1 - \alpha}\prbig{\textstyle\sum_{k=1}^{\tcount} p_{k} \delta_{V_k} + p_{\tcount+1} \delta_{\infty}}} \le 1 - \alpha + \epsilon + \prob\pr{\Delta\ge \epsilon}, & \forall \epsilon\in [0,\alpha)
        \end{cases},
    \end{equation*}
    The proof is concluded since the previous inequalities hold $\forall \epsilon\ge 0$.
\end{proof}

Let $\sigma_{\mathrm{cal},\1}, \sigma_{\mathrm{cal}}>0$ and $\forall k\in[\tcount+1]$ consider $(\sigma_{k,\1}, \sigma_{k})\in (\R_+)^2$.
Moreover, define
\begin{equation}\label{eq:def:epsilon-count}
    \epsilon_{\tcount}
    = 8 \sqrt{ \frac{ 2 \log 4\tcount }{ \tcount } }
    \pr{ 
        \sigma_{\mathrm{cal}, \1}^2
        \vee \sigma_{\1} 
        \vee \sqrt{ \sigma_{\mathrm{cal}}^2 + \sigma^2}
    }.
\end{equation}

\begin{lemma}\label{lem:upperbound-probdelta}
    Assume that $\frac{1}{\tcount}\sum_{k=1}^{\tcount} \{ \lambda\prn{\tilde{Z}_k} \1_{\tilde{V}_k<v} - \E\brn{\lambda\prn{\tilde{Z}_k} \1_{\tilde{V}_k<v}} \}$, $\frac{1}{\tcount}\sum_{k=1}^{\tcount} \{ \lambda\prn{Z_k} \1_{V_k<v} - \E\brn{\lambda\prn{Z_k} \1_{V_k<v}} \}$, $\frac{1}{\tcount}\sum_{k=1}^{\tcount} \{ \lambda\prn{\tilde{Z}_k} - \E \lambda\prn{\tilde{Z}_k} \}$ and $\frac{1}{\tcount}\sum_{k=1}^{\tcount} \{ \lambda\prn{Z_k} - \E \lambda\prn{Z_k} \}$ are sub-Gaussian with parameters $\{\sigma_{\1}, \sigma_{\mathrm{cal},\1}, \sigma, \sigma_{\mathrm{cal}}\}$ respectively, $\forall v\in\R$.
    It holds
    \begin{equation*}
        \prob\pr{ \Delta > \epsilon_{\tcount} } \vee \prob\pr{ \Delta < - \epsilon_{\tcount} }
        \le \frac{1 + 4 \var \lambda\prn{\tilde{Z}_{1}}}{\tcount}
        + \frac{4 \sum_{l=1}^{\tcount} \var \lambda\prn{Z_{l}} + 8 \var \lambda\prn{Z_{\tcount+1}}}{\tcount^2}
        ,
    \end{equation*}
    where $\Delta$ is defined in~\eqref{eq:def:X-delta}.
\end{lemma}
\begin{proof}
    First, recall that $\lambda= \rmd P^{\tcount+1}/\rmd P^k$ and $P^{\mathrm{cal}}=(1/\tcount)\sum_{k=1}^{\tcount} P^k$. Since $Z_k\sim P^k$ and $\tilde{Z}_{k}\sim P^{\mathrm{cal}}$, we have
    \begin{align}
        \nonumber
        \E\br{ \sum_{k=1}^{\tcount} \lambda\prn{Z_k} \1_{V_k<V_{\tcount+1}} \,\bigg\vert\, V_{\tcount+1}}
        &= \E\br{ \sum_{k=1}^{\tcount} \int \lambda\prn{z} \1_{V(z)<V_{\tcount+1}} \rmd P^k(z) \,\bigg\vert\, V_{\tcount+1} }
        \\
        \nonumber
        &= \E\br{ \tcount \int \lambda\prn{z} \1_{V(z)<V_{\tcount+1}} \rmd P^{\mathrm{cal}}(z) \,\bigg\vert\, V_{\tcount+1} }
        \\
        \label{eq:eq:probdelta:1}
        &= \E\br{ \sum_{k=1}^{\tcount} \lambda\prn{\tilde{Z}_k} \1_{\tilde{V}_k<V_{\tcount+1}} \,\bigg\vert\, V_{\tcount+1} }.
    \end{align}
    Using definition of the importance weights $p_k$ given in~\eqref{eq:def:qweighthat}, note that
    \begin{equation*}
        \begin{aligned}
            \sum_{k=1}^{\tcount} p_k \1_{V_k<V_{\tcount+1}}
            = \frac{\sum_{k=1}^{\tcount} \lambda\prn{Z_k} \1_{V_k<V_{\tcount+1}}}{\sum_{l=1}^{\tcount+1} \lambda\prn{Z_{l}}},&
            &\sum_{k=1}^{\tcount} q_{k} \1_{V_k<V_{\tcount+1}}
            = \frac{\sum_{k=1}^{\tcount} \lambda\prn{\tilde{Z}_k} \1_{\tilde{V}_k<V_{\tcount+1}}}{\sum_{l=1}^{\tcount+1} \lambda\prn{\tilde{Z}_{l}}}
            .
        \end{aligned}
    \end{equation*}
    Therefore,~\eqref{eq:eq:probdelta:1} implies that
    \begin{align}
        \nonumber
        &\Delta
        = \frac{ \pr{\sum_{l=1}^{\tcount+1} \lambda\prn{\tilde{Z}_{l}}} \sum_{k=1}^{\tcount} \lambda\prn{Z_k} \1_{V_k<V_{\tcount+1}} - \pr{\sum_{l=1}^{\tcount+1} \lambda\prn{Z_{l}}} \sum_{k=1}^{\tcount} \lambda\prn{\tilde{Z}_k} \1_{\tilde{V}_k<V_{\tcount+1}} }{\pr{\sum_{l=1}^{\tcount+1} \lambda\prn{Z_{l}}} \pr{\sum_{l=1}^{\tcount+1} \lambda\prn{\tilde{Z}_{l}}}}
        \\
        \nonumber
        &= \frac{ \sum_{k=1}^{\tcount} \acn{\lambda\prn{Z_k} \1_{V_k<V_{\tcount+1}} - \E\brn{\lambda\prn{Z_k} \1_{V_k<V_{\tcount+1}} \,\vert\, V_{\tcount+1} }} }{ \sum_{l=1}^{\tcount+1} \lambda\prn{Z_{l}} }
        \\
        \nonumber
        &\qquad+ \frac{ \sum_{l=1}^{\tcount+1} \acn{\lambda\prn{Z_{l}} - \lambda\prn{\tilde{Z}_{l}}} }{ \pr{ \sum_{l=1}^{\tcount+1} \lambda\prn{Z_{l}} } \pr{ \sum_{l=1}^{\tcount+1} \lambda\prn{\tilde{Z}_{l}} } } \E\br{ \sum_{k=1}^{\tcount} \lambda\prn{Z_k} \1_{V_k<V_{\tcount+1}} \,\bigg\vert\, V_{\tcount+1} }
        \\
        \label{eq:eq:probdelta:2}
        &\qquad - \frac{ \sum_{k=1}^{\tcount} \acn{\lambda\prn{\tilde{Z}_k} \1_{\tilde{V}_k<V_{\tcount+1}} - \E\brn{\lambda\prn{\tilde{Z}_k} \1_{\tilde{V}_k<V_{\tcount+1}} \,\vert\, V_{\tcount+1} }} }{ \sum_{l=1}^{\tcount+1} \lambda\prn{\tilde{Z}_{l}} }
        .
    \end{align}
    Moreover, define the following set:
    \begin{equation*}
        A_{\tcount}
        = \ac{\sum_{l=1}^{\tcount+1} \lambda\prn{Z_{l}} \le \frac{1}{2} \E\br{\sum_{l=1}^{\tcount+1} \lambda\prn{Z_{l}}}}
        \bigcup \ac{\sum_{l=1}^{\tcount+1} \lambda\prn{\tilde{Z}_{l}} \le \frac{1}{2} \E\br{\sum_{l=1}^{\tcount+1} \lambda\prn{\tilde{Z}_{l}}}}.
    \end{equation*}
    Hence, the Bienaymé-Tchebytchev's inequality shows that
    \begin{align*}
        &\prob\pr{\sum_{l=1}^{\tcount+1} \lambda\prn{Z_{l}} \le \frac{1}{2} \E\br{\sum_{l=1}^{\tcount+1} \lambda\prn{Z_{l}}}}
        \le \prob\pr{\sum_{l=1}^{\tcount+1} \ac{ \E \lambda\prn{Z_{l}} - \lambda\prn{Z_{l}} } \ge \frac{1}{2} \sum_{l=1}^{\tcount+1} \E \lambda\prn{Z_{l}}}
        \le \frac{4 \sum_{l=1}^{\tcount+1}  \var \lambda\prn{Z_{l}}}{\prn{\sum_{l=1}^{\tcount+1} \E \lambda\prn{Z_{l}}}^2},
        \\
        &\prob\pr{\sum_{l=1}^{\tcount+1} \lambda\prn{\tilde{Z}_{l}} \le \frac{1}{2} \E\br{\sum_{l=1}^{\tcount+1} \lambda\prn{\tilde{Z}_{l}}}}
        \le \frac{4 \sum_{l=1}^{\tcount+1}  \var \lambda\prn{\tilde{Z}_{l}}}{\prn{\sum_{l=1}^{\tcount+1} \E \lambda\prn{\tilde{Z}_{l}}}^2}.
    \end{align*}
    Thus, summing these two inequalities gives that    
    \begin{equation}\label{eq:eq:probdelta:3}
        \prob\pr{A_{\tcount}}
        \le \frac{4 \sum_{l=1}^{\tcount+1} \acn{\var \lambda\prn{Z_{l}} + \var \lambda\prn{\tilde{Z}_{l}}} }{\prn{\sum_{l=1}^{\tcount+1} \E \lambda\prn{Z_{l}}}^2}.
    \end{equation}
    Let $\epsilon>0$ be fixed.
    Since the $\{\lambda\prn{Z_k} \1_{V_k<V_{\tcount+1}} - \E\brn{\lambda\prn{Z_k} \1_{V_k<V_{\tcount+1}} \,\vert\, V_{\tcount+1} }\}$ are assumed $\sigma_{k,\1}$-sub-Gaussian, it holds:
    \begin{equation}\label{eq:eq:probdelta:4}
        \prob\pr{\frac{ \sum_{k=1}^{\tcount} \acn{\lambda\prn{Z_k} \1_{V_k<V_{\tcount+1}} - \E\brn{\lambda\prn{Z_k} \1_{V_k<V_{\tcount+1}} \,\vert\, V_{\tcount+1} }} }{ \sum_{l=1}^{\tcount+1} \E \lambda\prn{Z_{l}} } \ge \epsilon}
        \le \exp\pr{- \frac{\epsilon^2 ( \sum_{l=1}^{\tcount+1} \E \lambda\prn{Z_{l}} )^2}{2 \tcount \sigma_{\mathrm{cal},\1}^2 }}.
    \end{equation}
    Similarly, the Hoeffding's inequality implies that
    \begin{equation}\label{eq:eq:probdelta:5}
        \prob\pr{ \frac{ \sum_{k=1}^{\tcount} \acn{ \E\brn{\lambda\prn{\tilde{Z}_k} \1_{\tilde{V}_k<V_{\tcount+1}} \,\vert\, V_{\tcount+1} } - \lambda\prn{\tilde{Z}_k} \1_{\tilde{V}_k<V_{\tcount+1}} } }{ \sum_{l=1}^{\tcount+1} \E \lambda\prn{\tilde{Z}_{l}} } \ge \epsilon }
        \le \exp\pr{- \frac{\epsilon^2 ( \sum_{l=1}^{\tcount+1} \E \lambda\prn{Z_{l}} )^2}{2 \tcount \sigma_{\1}^2 }}.
    \end{equation}
    Moreover, define $\tau\in [0,1]$ by: 
    \begin{equation*}
        \tau
        = \frac{ \E\brn{ \sum_{k=1}^{\tcount} \lambda\prn{Z_k} \1_{V_k<V_{\tcount+1}} \,\vert\, V_{\tcount+1} } }{ \sum_{l=1}^{\tcount+1} \E \lambda\prn{Z_{l}} }
        .
    \end{equation*}
    Using that $\sum_{l=1}^{\tcount} \E\lambda\prn{Z_{l}} = \sum_{l=1}^{\tcount} \E\lambda\prn{\tilde{Z}_{l}}$, we obtain
    \begin{equation}\label{eq:eq:probdelta:6}
        \prob\pr{ \frac{ \sum_{l=1}^{\tcount} \acn{\lambda\prn{Z_{l}} - \lambda\prn{\tilde{Z}_{l}}} }{ \sum_{l=1}^{\tcount+1} \E \lambda\prn{Z_{l}} } \ge \frac{\epsilon}{\tau} }
        \le \exp\pr{- \frac{\epsilon^2 ( \sum_{l=1}^{\tcount+1} \E \lambda\prn{Z_{l}} )^2}{2 \tcount \sigma_{\mathrm{cal}}^2 + 2 \tcount \sigma^2 }}.
    \end{equation}
    Therefore, combining~\eqref{eq:eq:probdelta:2}-\eqref{eq:eq:probdelta:3}-\eqref{eq:eq:probdelta:4}-\eqref{eq:eq:probdelta:5}-\eqref{eq:eq:probdelta:6} with $\epsilon=\epsilon_{\tcount}$, we obtain
    \begin{multline*}
        \prob\pr{\Delta \ge \frac{\epsilon_{\tcount}}{4} + \frac{\epsilon_{\tcount}}{4} + \frac{\epsilon_{\tcount}}{2}}
        \le \prob\pr{A_{\tcount}}
        + \prob\pr{\frac{ \sum_{k=1}^{\tcount} \acn{\lambda\prn{Z_k} \1_{V_k<V_{\tcount+1}} - \E\brn{\lambda\prn{Z_k} \1_{V_k<V_{\tcount+1}} \,\vert\, V_{\tcount+1} }} }{ \sum_{l=1}^{\tcount+1} \E \lambda\prn{Z_{l}} } \ge \frac{\epsilon_{\tcount}}{8}}
        \\
        + \prob\pr{ \frac{ \sum_{k=1}^{\tcount} \acn{ \E\brn{\lambda\prn{\tilde{Z}_k} \1_{\tilde{V}_k<V_{\tcount+1}} \,\vert\, V_{\tcount+1} } - \lambda\prn{\tilde{Z}_k} \1_{\tilde{V}_k<V_{\tcount+1}} } }{ \sum_{l=1}^{\tcount+1} \E \lambda\prn{\tilde{Z}_{l}} } \ge \frac{\epsilon_{\tcount}}{8} }
        + \prob\pr{ \frac{ \sum_{l=1}^{\tcount} \acn{\lambda\prn{Z_{l}} - \lambda\prn{\tilde{Z}_{l}}} }{ \sum_{l=1}^{\tcount+1} \E \lambda\prn{Z_{l}} } \ge \frac{\epsilon_{\tcount}}{8} }.
    \end{multline*}
    Finally, a similar reasoning for $\prob\prn{\Delta \le - \epsilon_{\tcount}}$ and using $\sum_{l=1}^{\tcount} \E \lambda\prn{Z_{l}} = \tcount$ concludes the proof.
\end{proof}

\begin{lemma}\label{lem:sub-Gaussian}
    Assume for $v\in\R$, that $\lambda(\tilde{Z}_1)\1_{\tilde{V}_1< v} - \E[\lambda(\tilde{Z}_1)\1_{\tilde{V}_1< v}]$ is $\sigma$-sub-Gaussian with parameter $\sigma\ge 0$.
    Then, $\frac{1}{\tcount}\sum_{k=1}^{\tcount} \{ \lambda\prn{\tilde{Z}_k} \1_{\tilde{V}_k<v} - \E\brn{\lambda\prn{\tilde{Z}_k} \1_{\tilde{V}_k<v}} \}$, $\frac{1}{\tcount}\sum_{k=1}^{\tcount} \{ \lambda\prn{Z_k} \1_{V_k<v} - \E\brn{\lambda\prn{Z_k} \1_{V_k<v}} \}$, $\frac{1}{\tcount}\sum_{k=1}^{\tcount} \{ \lambda\prn{\tilde{Z}_k} - \E \lambda\prn{\tilde{Z}_k} \}$ and $\frac{1}{\tcount}\sum_{k=1}^{\tcount} \{ \lambda\prn{Z_k} - \E \lambda\prn{Z_k} \}$ are $\sigma$-sub-Gaussian.
\end{lemma}

\begin{proof}
    Let $v\in\R$ and $\gamma\in\R$. By concavity of the logarithm, the Jensen's inequality implies
    \begin{equation*}
        \frac{1}{\tcount} \sum_{k=1}^{\tcount} \log \E\br{\exp\ac{\gamma \lambda\prn{Z_k}\1_{V_k< v}}}
        \le \log \E\br{ \frac{1}{\tcount} \sum_{k=1}^{\tcount} \exp\ac{\gamma \lambda\prn{Z_k}\1_{V_k< v}}}.
    \end{equation*}
    Therefore, by  the increasing property of the exponential, we obtain
    \begin{equation*}
        \prod_{k=1}^{\tcount} \E\br{\exp\ac{\gamma \lambda\prn{Z_k}\1_{V_k< v}}}
        \le \E\br{ \frac{1}{\tcount} \sum_{k=1}^{\tcount} \exp\ac{\gamma \lambda\prn{Z_k}\1_{V_k< v}}}^{\tcount}
        = \prod_{k=1}^{\tcount} \E\br{\exp\ac{\gamma \lambda(\tilde{Z}_k)\1_{\tilde{V}_k< v}}}.
    \end{equation*}
    Since $\sum_{k=1}^{\tcount} \E[\lambda\prn{Z_k}\1_{V_k< v}] = \sum_{k=1}^{\tcount} \E[\lambda(\tilde{Z}_k)\1_{\tilde{V}_k< v}]$, multiplying by $\exp(- \gamma \sum_{k=1}^{\tcount} \E[\lambda\prn{Z_k}\1_{V_k< v}])$ the previous inequality shows that $\frac{1}{\tcount}\sum_{k=1}^{\tcount} \{ \lambda\prn{Z_k} \1_{V_k<v} - \E\brn{\lambda\prn{Z_k} \1_{V_k<v}} \}$ is $\sigma$-sub-Gaussian.
    Moreover, applying Fatou's lemma gives that
    \begin{align*}
        \E\br{ \liminf_v \rme^{\frac{\gamma }{\tcount}\sum_{k=1}^{\tcount} \{ \lambda\prn{Z_k} \1_{V_k<v} - \E\brn{\lambda\prn{Z_k} \1_{V_k<v}} \} } }
        &\le \liminf_v \E\br{ \rme^{\frac{\gamma }{\tcount}\sum_{k=1}^{\tcount} \{ \lambda\prn{Z_k} \1_{V_k<v} - \E\brn{\lambda\prn{Z_k} \1_{V_k<v}} \} } }
        \\
        &\le \exp\pr{\frac{\gamma^2 \sigma^2}{2}}.
    \end{align*}
    For $k\in[\tcount]$, since $V_k<\infty$ almost surely, the dominated convergence theorem combined with the continuity of the exponential function imply that
    \begin{equation*}
        \liminf_v \exp\pr{\frac{\gamma }{\tcount}\sum_{k=1}^{\tcount} \{ \lambda\prn{Z_k} \1_{V_k<v} - \E\brn{\lambda\prn{Z_k} \1_{V_k<v}} \} }
        = \exp\pr{\frac{\gamma }{\tcount} \sum_{k=1}^{\tcount} \{ \lambda\prn{Z_k} - \E\brn{\lambda\prn{Z_k}} \} }.
    \end{equation*}
    Thus, we deduce that $\frac{1}{\tcount} \sum_{k=1}^{\tcount} \{ \lambda\prn{Z_k} - \E\brn{\lambda\prn{Z_k}} \}$ is $\sigma$-sub-Gaussian.
\end{proof}

In order to simplify the calculation, we also provide the statement when assuming that $\lambda$ is bounded by $\norm{\lambda}_{\infty}$.

\begin{theorem}
    Assume there are no ties between $\{V_k\}_{k\in[\tcount+1]}\cup\{\infty\}$ almost surely.
    If $\exists\sigma\ge0$, such that $\forall v\in\R$, $\lambda\prn{\tilde{Z}_1} \1_{\tilde{V}_1<v} - \E\brn{\lambda\prn{\tilde{Z}_1} \1_{\tilde{V}_1<v}}$ is sub-Gaussian with parameter $\sigma$.
    Then, it holds
    \begin{equation*}
        \abs{\prob\pr{V_{\tcount+1} \le \q{1 - \alpha}\pr{\textstyle\sum_{k=1}^{\tcount} p_{k} \delta_{V_k} + p_{\tcount+1} \delta_{\infty}}}
        - 1 + \alpha} 
        \le
        \begin{cases}
            35 \norm{\lambda}_{\infty}^2 \sqrt{\frac{ \log 4\tcount }{ \tcount }} 
            & \text{if $\lambda$ is bounded}
            \\ 
            \frac{ 18 \E\lambda^2(Z_{\tcount+1})}{\tcount} + 19 \sigma \sqrt{\frac{ \log 4\tcount }{ \tcount }}
            & \text{otherwise}
        \end{cases}
        .
    \end{equation*}
\end{theorem}

\begin{proof}
    Let's start by using \Cref{lem:probVN-infQuantile} with $\epsilon=\epsilon_{\tcount}$ defined in~\eqref{eq:def:epsilon-count}, it gives
    \begin{multline}\label{eq:bound:thm:bias:1}
        \textstyle
        - \epsilon_{\tcount} - \prob\pr{\Delta \le -\epsilon_{\tcount}}
        \le \prob\pr{V_{\tcount+1}\le \q{1 - \alpha}\pr{\textstyle\sum_{k=1}^{\tcount} p_{k} \delta_{V_k} + p_{\tcount+1} \delta_{\infty}}} - 1 + \alpha \\
        \le \epsilon_{\tcount} + \prob\pr{\Delta \ge \epsilon_{\tcount}} + \E\br{\max_{k=1}^{\tcount+1}\acn{q_{k}}}.
    \end{multline}
    Therefore, the previous inequalities combined with \Cref{lem:upperbound-probdelta} and \Cref{lem:sub-Gaussian} imply that
    \begin{multline}\label{eq:bound:thm:bias:2}
        \abs{\prob\pr{V_{\tcount+1}\le \q{1 - \alpha}\pr{\textstyle\sum_{k=1}^{\tcount} p_{k} \delta_{V_k} + p_{\tcount+1} \delta_{\infty}}} - 1 + \alpha} 
        \\
        \le \frac{1 + 4 \var \lambda\prn{\tilde{Z}_{1}}}{\tcount}
        + \frac{4 \sum_{l=1}^{\tcount} \var \lambda\prn{Z_{l}} + 8 \var \lambda\prn{Z_{\tcount+1}}}{\tcount^2}
        + 16 \sigma \sqrt{\frac{ \log 4\tcount }{ \tcount }}
        + \E\br{\max_{k=1}^{\tcount+1}\acn{q_{k}}}.
    \end{multline}
    Furthermore, applying the result from \citep[Lemma C.17]{plassier2023conformal} implies that
    \begin{equation}\label{eq:bound:thm:bias:3}
        \E\br{\max_{k=1}^{\tcount+1} q_k}
        \le 
        \begin{cases}
            \frac{2 \normn{\lambda}_{\infty}}{\tcount}
            + \frac{4}{\tcount} \pr{\var\lambda\prn{\tilde{Z}_1} + \frac{\var \lambda\prn{Z_{\tcount+1}} }{\tcount}} 
            \le \frac{10 \normn{\lambda}_{\infty}^2}{\tcount}
            & \text{if $\lambda$ is bounded}
            \\
            \frac{\sigma \sqrt{8 \log (\tcount+1)}}{\tcount}
            + \frac{2}{\tcount} \pr{ \E\lambda(Z_{\tcount+1}) + 2 \var \lambda(\tilde{Z}_1) + \frac{2 \var \lambda(Z_{\tcount+1})}{\tcount}}
            & \text{otherwise}
        \end{cases}
        .
    \end{equation}
    Therefore, if $\lambda$ is bounded by $\norm{\lambda}_{\infty}$, then the sub-Gaussian parameters are controled as follows:
    \begin{equation*}
        \sigma \le \norm{\lambda}_{\infty}/2;
    \end{equation*}
    see \citep[Lemma 2.2]{boucheron2003concentration} for more details.
    Hence, plugging~\eqref{eq:bound:thm:bias:3} into~\eqref{eq:bound:thm:bias:2} concludes the first part of the proof.
    Now, let's proof the second part of the theorem.
    For that, remark
    \begin{align}\label{eq:bound:thm:bias:4}
        \E{\lambda(\tilde{Z}_{1})} = 1,&
        &\var{\lambda(\tilde{Z}_{1})}
        = \E\brn{\lambda^2(\tilde{Z}_{1})} - 1,
    \end{align}
    and note that
    \begin{align}\label{eq:bound:thm:bias:5}
        \E{\lambda(Z_{\tcount+1})}
        = \int \pr{\frac{\rmd P^{\tcount+1}}{\rmd P^{\mathrm{cal}}}}^2 (z) \, \rmd P^{\mathrm{cal}}(z)
        = \E\brn{\lambda^2(\tilde{Z}_{1})}.
    \end{align}
    Therefore, we obtain that
    \begin{align}\label{eq:bound:thm:bias:6}
        \var{\lambda(Z_{\tcount+1})}
        = \E\brn{\lambda^2(Z_{\tcount+1})} - \E\br{\lambda(Z_{\tcount+1})}^2
        = \E\brn{\lambda^3(\tilde{Z}_{1})} - \E\brn{\lambda^2(\tilde{Z}_{1})}^2.
    \end{align}
    Combining with~\eqref{eq:bound:thm:bias:4}-\eqref{eq:bound:thm:bias:5}-\eqref{eq:bound:thm:bias:6} with~\eqref{eq:bound:thm:bias:3} shows that
    \begin{equation}\label{eq:bound:thm:bias:7}
        \E\br{\max_{k=1}^{\tcount+1} q_k}
        \le \frac{\sigma \sqrt{8 \log (\tcount+1)}}{\tcount}
        + \frac{2}{\tcount} \pr{ \ac{3-\frac{2}{\tcount}} \E\lambda^2(\tilde{Z}_{1}) + \frac{2\E\lambda^3(\tilde{Z}_1)}{\tcount} - 2}
        .
    \end{equation}
    Using the Cauchy-Schwarz's inequality, it holds
    \begin{align*}
        \sum_{k=1}^{\tcount} \var \lambda(Z_{k})
        = \sum_{k=1}^{\tcount} \int \pr{\frac{\rmd P^{\tcount+1}}{\rmd P^{\mathrm{cal}}}}^2 (z) \, \rmd P^{k}(z) - \sum_{k=1}^{\tcount} \E\br{\lambda(Z_{k})}^2
        \le \tcount \E\brn{\lambda^2(\tilde{Z}_{1})} - \tcount.
    \end{align*}
    The previous inequality combined with~\eqref{eq:bound:thm:bias:2} and~\eqref{eq:bound:thm:bias:7} implies that
    \begin{multline}\label{eq:bound:thm:bias:8}
        \abs{\prob\pr{V_{\tcount+1}\le \q{1 - \alpha}\pr{\textstyle\sum_{k=1}^{\tcount} p_{k} \delta_{V_k} + p_{\tcount+1} \delta_{\infty}}} - 1 + \alpha} 
        \\
        \le \frac{ 8 \E\lambda^3(\tilde{Z}_1) + 8 \E\lambda^2(\tilde{Z}_1) - 8 \prn{\E\lambda^2(\tilde{Z}_1)}^2 - 7 }{\tcount}
        + 16 \sigma \sqrt{\frac{ \log 4\tcount }{ \tcount }}
        + \E\br{\max_{k=1}^{\tcount+1}\acn{q_{k}}}
        \\
        \le \frac{ 18 \E\lambda^3(\tilde{Z}_1)}{\tcount}
        + 19 \sigma \sqrt{\frac{ \log 4\tcount }{ \tcount }}
        .
    \end{multline}
\end{proof}

\section{Details on the federated quantile computation}
\label{sec:numerical-results}

\begin{algorithm}
  \caption{Federated Quantile Estimation}
  \label{algo:Qgamma}
  \begin{algorithmic}[]
      \State {\bfseries Input:} significance level $\alpha$, number of rounds $T$, learning rate $\eta$, Moreau regularization parameter $\gamma$, parameter of the DP noise $\sigma$, number of local iteration $K$.
      \State \Comment{In parallel on the local agents}
      \For{each agent $i=0$ {\bfseries to} $\nclients$}
        \State Estimate and transmit the GMM parameters $\{\pi_y^i,m_y^i,\Sigma_y^i\}_{y\in\YC}$ as in \eqref{eq:def:GMM-params}
        \State Compute $\forall (X_k^i,Y_k^i)$, $V_k^i=V(X_k^i,Y_k^i)$ and $\lambda_k^i=(\rmd P_X^{\target}/\rmd P_X^{\mathrm{cal}})(X_k^i)$ using \eqref{eq:importance_weights}
        \State Transfer $\Lambda^i = \sum_{k=1}^{\ccount{i}} \lambda_k^i$ to the central server
      \EndFor
      \For{$t=0$ {\bfseries to} $T-1$}
          \State \Comment{On the central server}
          \State $S_{t + 1} \gets$ random subset of $[\nclients]$ 
          \State \Comment{In parallel on the local agents}
          \For{each agent $i \in S_{t + 1}$}
              \State Initialize quantile $q_{t, 0}^i\gets q_t$
              \For{$k=0$ {\bfseries to} $K - 1$}
                  \State \Comment{Gradient with DP noise}
                  \State $g_{t,k}^i \gets \nabla \mathrm{loss}_i^{(\gamma)} \prn{q_{t,k}^i} + z_{t, k}^i$, $z_{t, k}^i\sim\gauss(0,\sigma^2)$
                  \State \Comment{Update local quantile}
                  \State $q_{t, k + 1}^i \gets q_{t, k}^i - \eta g_{t, k}^i$
              \EndFor
          \State $(\Delta q_{t + 1}^i, \Delta \bar{q}_{t + 1}^{i}) \gets (q_{t, K}^i - q_{t, 0}^i, \sum_{k \in [K]}\frac{q_{t, k}^i}{K})$
          \EndFor
          \State \Comment{On the central server}
          \State $q_{t + 1} \gets q_t + \frac{\nclients}{\absn{S_{t + 1}}} \sum_{i \in S_{t + 1}} (\frac{\Lambda^i}{\sum_{j=1}^{\nclients} \Lambda^j}) \Delta q_{t + 1}^i$
          \State $\bar{q}_{t + 1} \gets \frac{t}{t + 1} \bar{q}_{t} + \frac{\nclients}{\absn{S_{t + 1}}} \sum_{i \in S_{t + 1}} (\frac{\Lambda^i}{\sum_{j=1}^{\nclients} \Lambda^j}) \frac{\Delta \bar{q}_{t + 1}^{i}}{t + 1}$
      \EndFor
      \State {\bfseries Output:}  $\widehat{Q}_{1-\alpha}^{(\gamma)} \gets \bar{q}_T$.
  \end{algorithmic}
\end{algorithm}

In this section, we detail the algorithm implemented for \Cref{sec:experiments}. 
This \Cref{algo:Qgamma} is divided into two parts.
In the first part, the density ratios $\lambda(X_k^i)$ are estimated from the local data of each agent. 
Then, the quantile $\widehat{Q}_{1 - \alpha}^{(\gamma)}$ is determined using the procedure developed by \citet{plassier2023conformal}. 
This method is based on the FedAvg algorithm with Gaussian noise to ensure the differential privacy \citep{mcmahan2017communication,dwork2006differential,ha2019differential}.
Given a regularization parameter $\gamma>0$, each agent uses its local data to calculate its loss function $\mathrm{loss}_i^{(\gamma)}$ whose gradient is:
\begin{align*}
  &
  \begin{aligned}
    \nabla \mathrm{loss}_i^{(\gamma)} (q)
    = \frac{1}{\Lambda^i} \sum_{k=1}^{\ccount{i}} \lambda(X_k^i) \nabla S_{\alpha,V_k^i}^{(\gamma)} (q),& \qquad\qquad\qquad
    &\Lambda^i = \sum_{k=1}^{\ccount{i}} \lambda(X_k^i)  
  \end{aligned}
  \\
  &\nabla \pinballmoreau{v}{q}
  = - (1 - \alpha) \indiacc{q < v - \gamma (1 - \alpha)} + \alpha \indiacc{q > v + \gamma \alpha} + \frac{1}{\gamma}(q - v)\indiacc{v - \gamma (1 - \alpha) < q < v + \gamma \alpha}.
\end{align*}
The quantile $Q_{1 - \alpha}^{(\gamma)}$ is obtained solving
\begin{equation*}
  Q_{1 - \alpha}^{(\gamma)}
  \in \argmin \ac{ \textstyle \sum_{i=1}^{\nclients} \Lambda^i \, \nabla \mathrm{loss}_i^{(\gamma)} }.
\end{equation*}
Based on an approximation $\widehat{Q}_{1 - \alpha}^{(\gamma)}$, we generate the following prediction set:
\begin{equation*}
  \widehat{\mathcal{C}}_{\alpha, \mu}(\xquery)
  = \ac{\yv \in \YC\colon V(\xquery, \yv) \le \widehat{Q}_{1 - \alpha}^{(\gamma)}}.
\end{equation*}

\section{Additional Experiments}
\label{sec:additional-experiments}
\subsection{Additional experiment on ImageNet}
We also experiment with ImageNet-R~\citep{hendrycks2021} as a corrupted data source.
Using the same data partitioning scheme as with ImageNet-C, we evaluate the performance of the presented algorithm under real distribution shifts.
Specifically, we have $20$ clients, half of them contain original ImageNet data, and the other half consist of ImageNet-R data samples.
Each client has $850$ data samples, consisting of $500$ model fitting samples, $50$ calibration samples and $300$ test samples.

The results obtained with ImageNet-R are consistent with those obtained using the ImageNet-C dataset.
In particular, in Figure~\ref{fig:imagenet_r_coverage} the coverage of our method best approaches the nominal value of $1-\alpha$ for both non-shifted data and shifted data.
Also, in the Figure~\ref{fig:imagenet_r_set_size}, the average prediction set sizes of the introduced method have less variance than the local baseline.

\begin{figure*}[t!]
  \centering
  \begin{subfigure}{0.45\textwidth}
    \includegraphics[width=\textwidth]{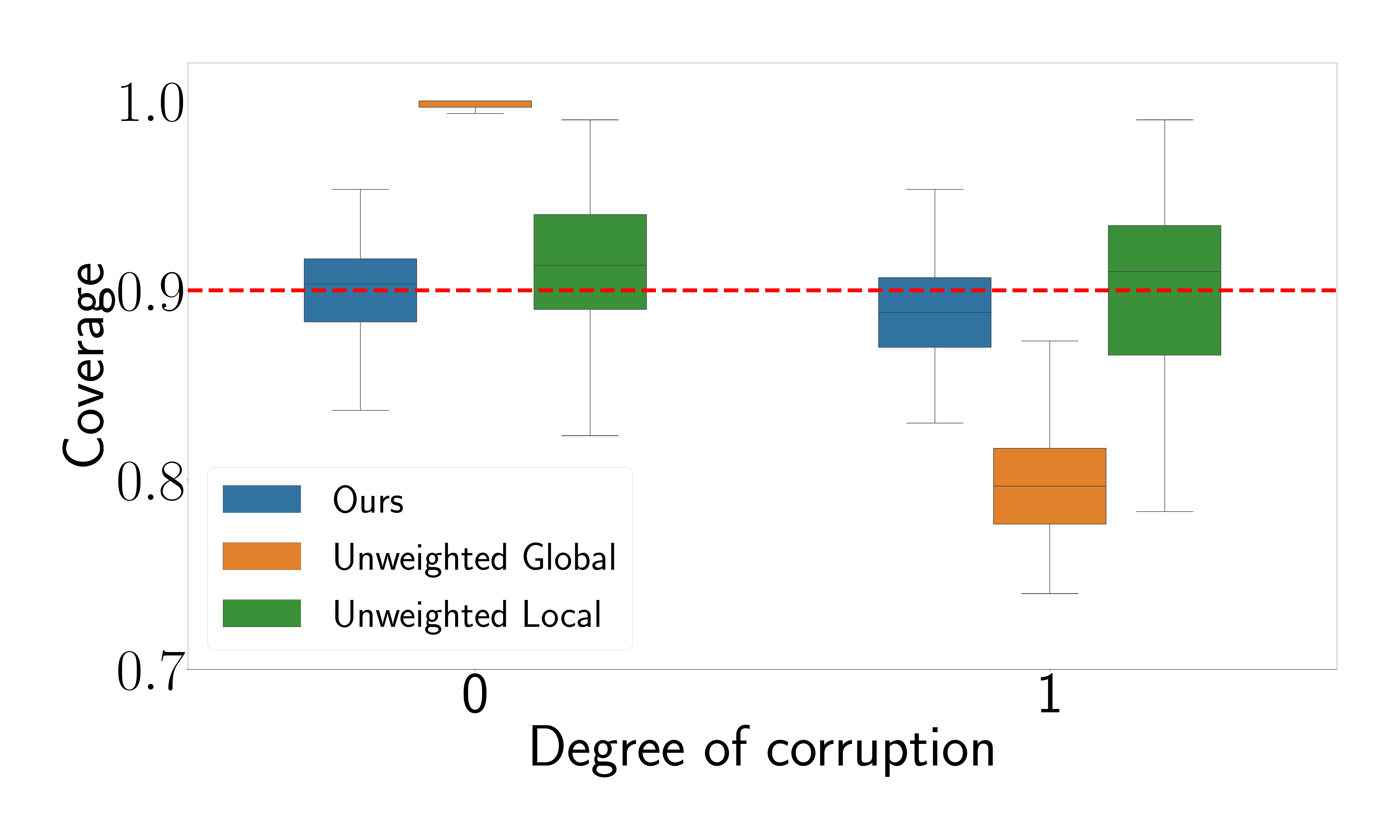}
    \vskip-10pt
    \caption{} \label{fig:imagenet_r_coverage}
  \end{subfigure}
  ~~~~
  \begin{subfigure}{0.45\textwidth}
    \includegraphics[width=\textwidth]{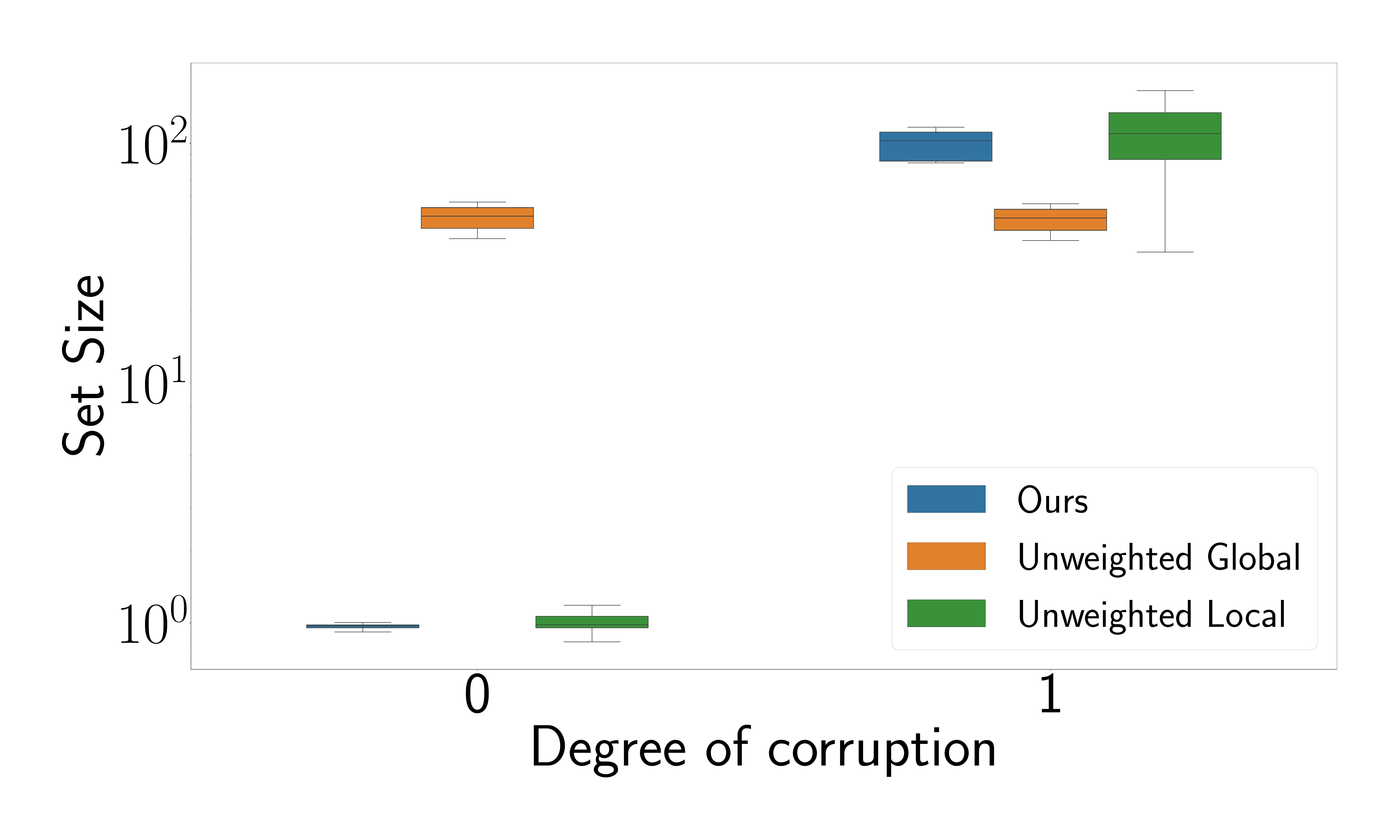}
    \vskip-10pt
    \caption{} \label{fig:imagenet_r_set_size}
  \end{subfigure}
  \vskip-5pt
  \caption{ImageNet \& ImageNet-R experimental results: (a) Empirical coverage of conformal prediction sets for non-corrupted and corrupted data. It shows how accurately do conformal sets capture the true classes of the data. (b) Average set size of conformal prediction sets for non-corrupted and corrupted data. The size of the sets increases with the level of corruption due to the increasing uncertainty of the model based on the corrupted data.}
  \vskip-10pt
\end{figure*}

\subsection{Additional experiment on CIFAR-10}

Following the same setup as for the CIFAR-100 in Figure~\ref{fig:cifar_100_avg_mean} and Figure~\ref{fig:cifar_100_avg_std}, we conduct the same experiment for the CIFAR-10. To recap, we have 100 clients, 100 data points were allocated for test, while the size of calibration dataset varied. Results for this experiment are presented in Figure~\ref{fig:cifar10_avg_mean} and Figure~\ref{fig:cifar10_avg_std}. As before for CIFAR-100, we see that our approach benefits from the federated collaborative procedure, which results in almost perfect coverage with smaller variance.

\begin{figure*}[t!]
  \centering
  \begin{subfigure}{0.45\textwidth}
    \includegraphics[width=\textwidth]{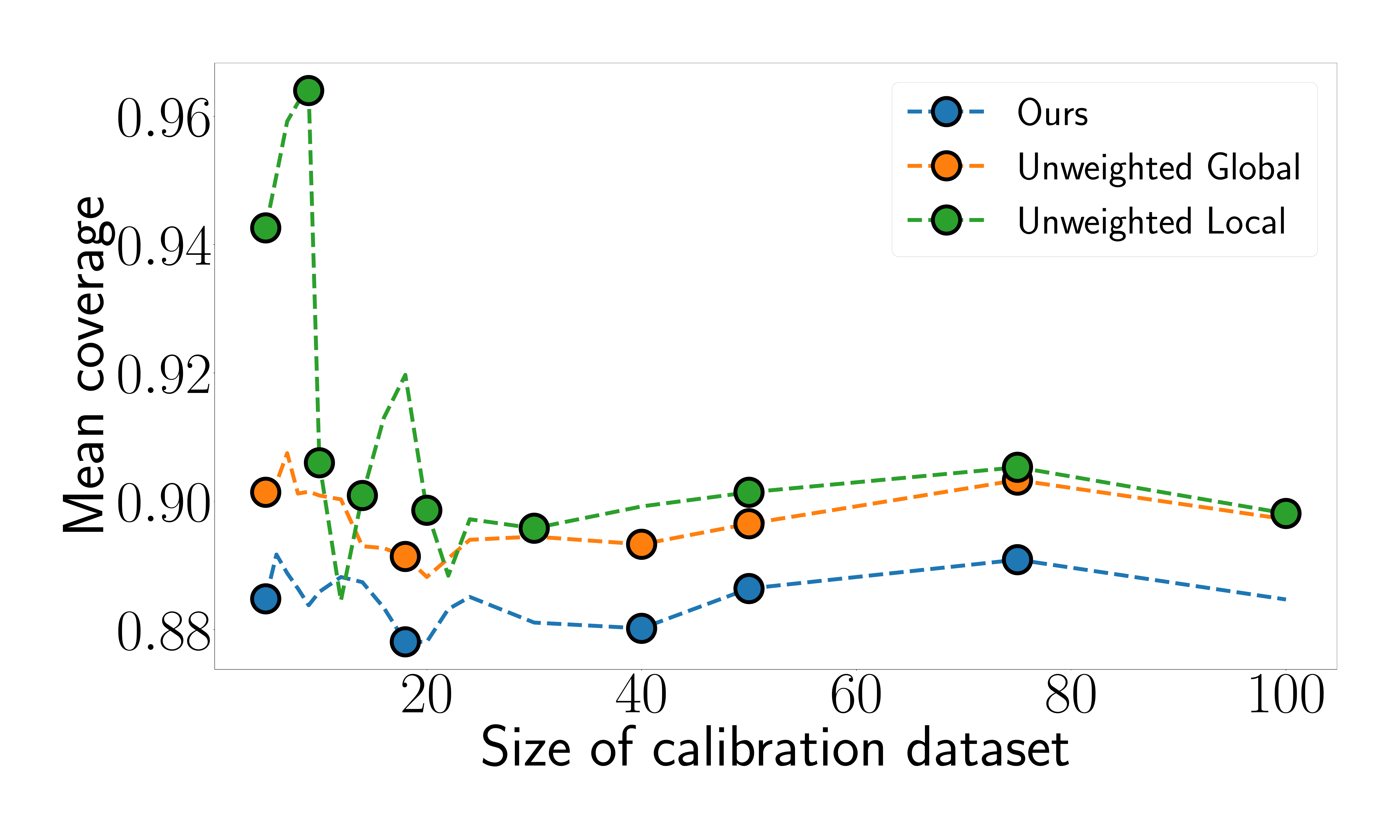}
    \vskip-10pt
    \caption{} \label{fig:cifar10_avg_mean}
  \end{subfigure}
  ~~~~
  \begin{subfigure}{0.45\textwidth}
    \includegraphics[width=\textwidth]{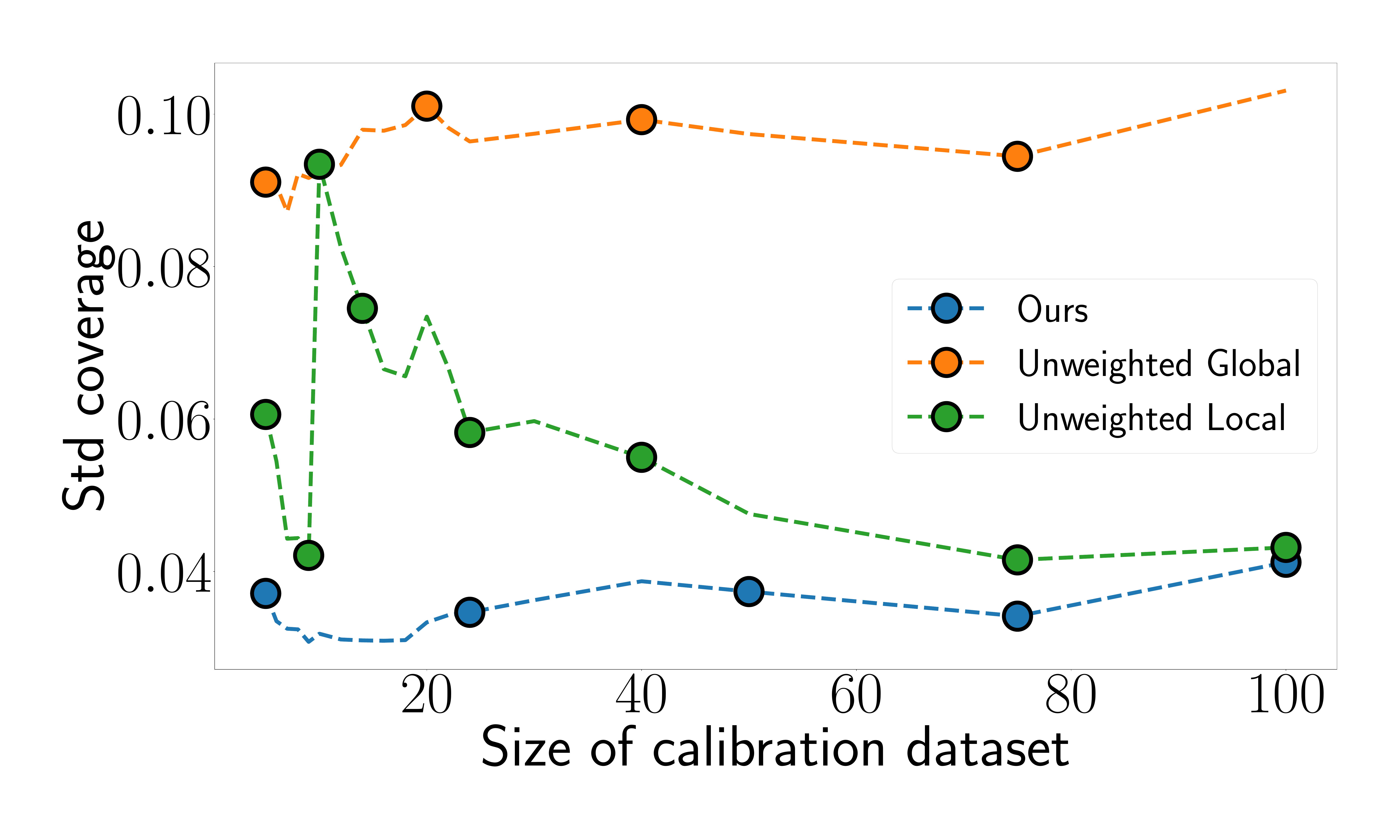}
    \vskip-10pt
    \caption{} \label{fig:cifar10_avg_std}
  \end{subfigure}
  \vskip-5pt
  \caption{CIFAR-10 experimental results. (a) Mean empirical coverage changes as a function of the calibration dataset size. (b) Standard deviation of empirical coverage as function of the calibration dataset size.}
  \vskip-10pt
\end{figure*}

\subsection{Additional experiment on CIFAR-100}

In Figure~\ref{fig:cifar_100_avg_std} from the main text, it may seem that variance of our approach constantly increases. We decided to continue the experiment just to check if the variances will intersect at some point. We present the additional plot in Figure~\ref{fig:cifar100_avg_mean} and Figure~\ref{fig:cifar100_avg_std}. We see, that this effect was rather random, and our approach still optimal from the variance point of view.

\begin{figure*}[t!]
  \centering
  \begin{subfigure}{0.45\textwidth}
    \includegraphics[width=\textwidth]{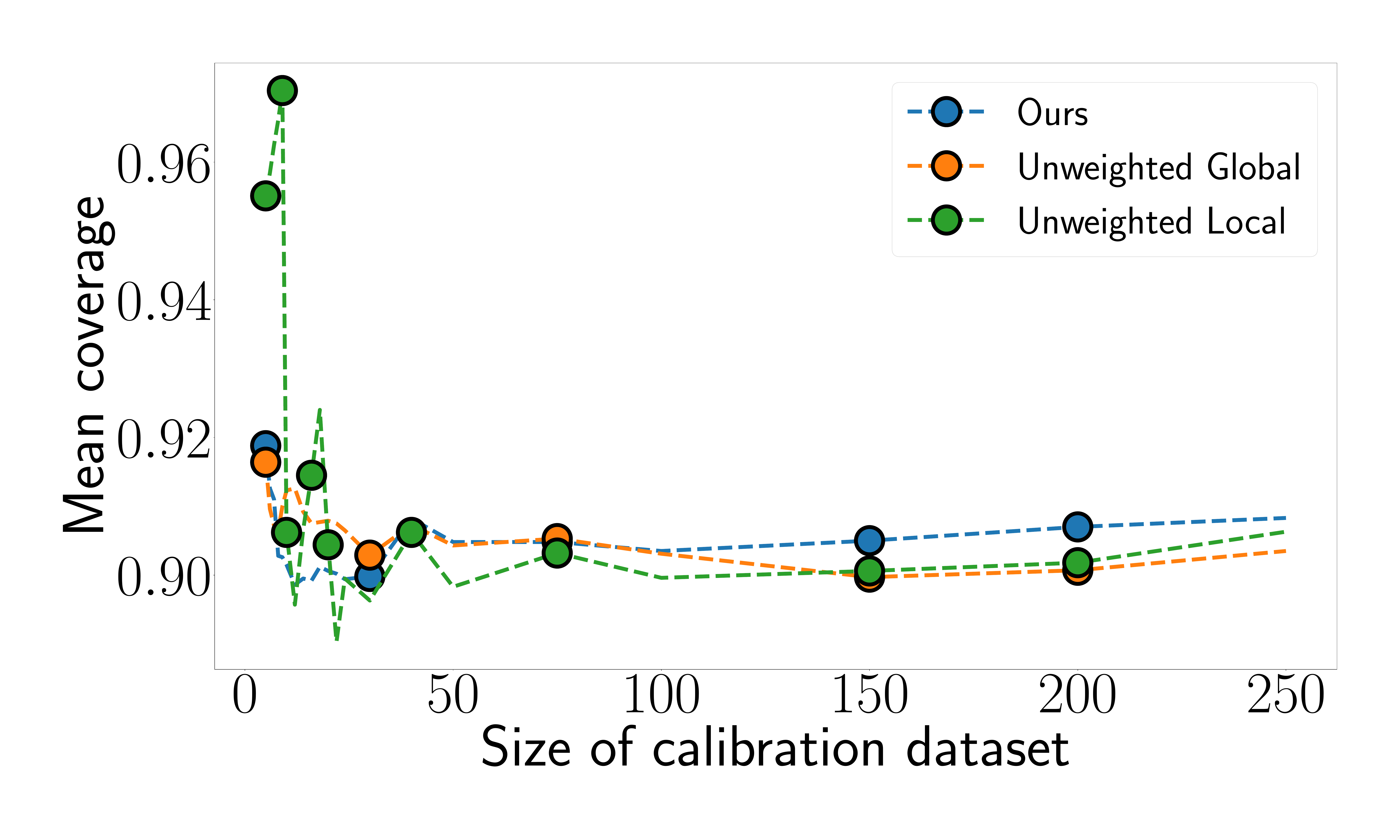}
    \vskip-10pt
    \caption{} \label{fig:cifar100_avg_mean}
  \end{subfigure}
  ~~~~
  \begin{subfigure}{0.45\textwidth}
    \includegraphics[width=\textwidth]{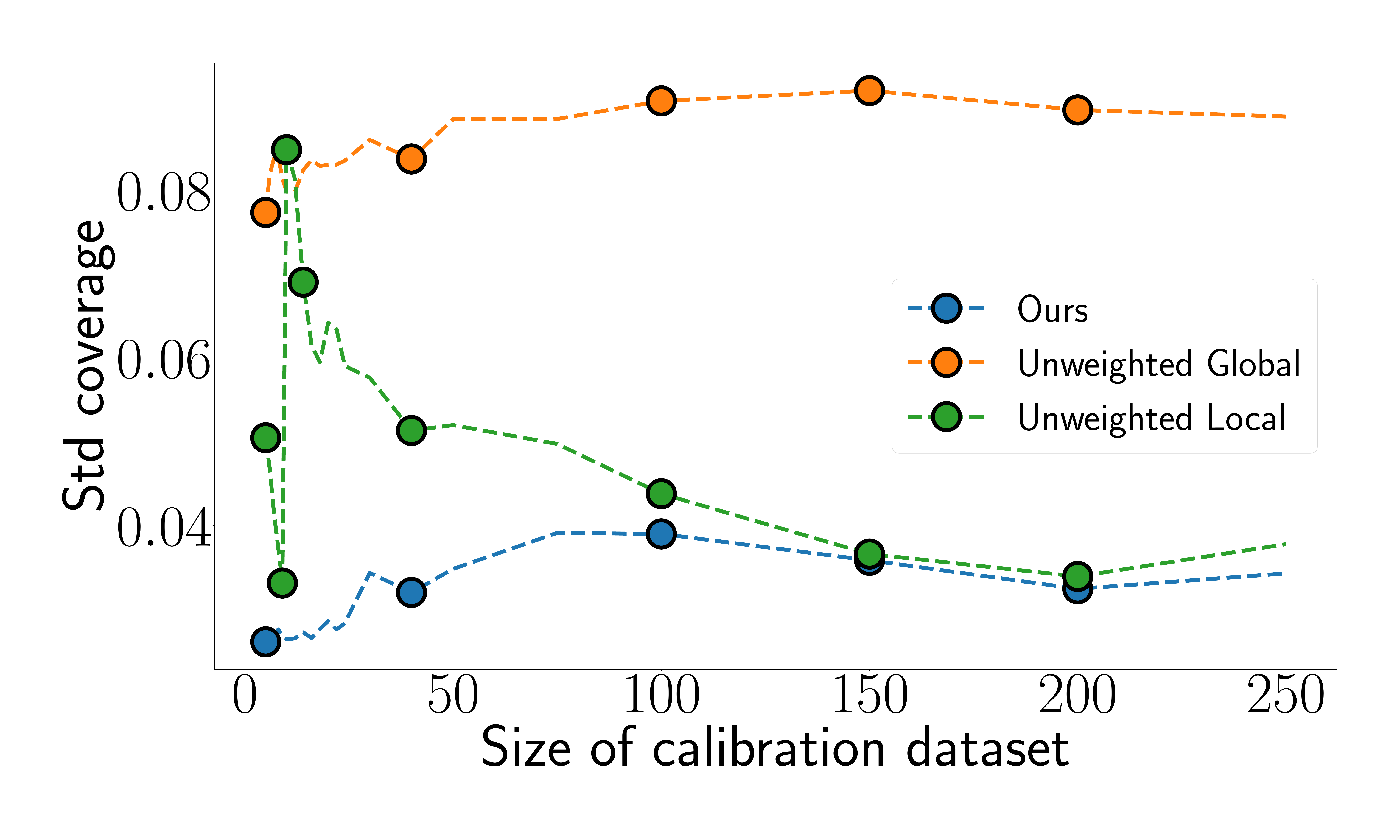}
    \vskip-10pt
    \caption{} \label{fig:cifar100_avg_std}
  \end{subfigure}
  \vskip-5pt
  \caption{CIFAR-100 experimental results with extended axis on the size of calibration dataset. (a) Mean empirical coverage changes as a function of the calibration dataset size. (b) Standard deviation of empirical coverage as function of the calibration dataset size.}
  \vskip-10pt
\end{figure*}

\subsection{Comparison with Tibshirani's weights}

\citet{tibshirani2019conformal}'s algorithm relies on importance weights derived from the density of $(Z_{\sigma(1)},\ldots,Z_{\sigma(N+1)})$ calculated over all possible permutations $\sigma$, that in general case requires the evaluation of sums of $N!$ terms, where $N$ represents the total number of calibration points. Tibshirani's expression simplifies in very special cases, such as when the distribution of the test point is different from the distribution of the calibration data, which are i.i.d. (an example discussed in \citet{tibshirani2019conformal}).
Our proposed conformal prediction set \textbf{aligns with Tibshirani's when the calibration data are i.i.d.}. However, our main contribution is that our prediction set \textbf{remains valid even for the non-i.i.d. calibration data}, i.e. in the presence of label and covariate shifts \textbf{inside} the calibration data. To show this, we provide conditional coverage guarantees and avoid the intractability of combinatorial Tibshirani's CP set.

\citet{plassier2023conformal} method consists in subsampling the calibration data to generate i.i.d. samples from the mixture distribution, and then using the ``simplified'' Tibshirani's weights. However, the subsampling cost leads to an increase in variance (and also the theory becomes more convoluted compared to our paper).

We conducted additional experiments comparing the work of \citet{plassier2023conformal} and \citet{tibshirani2019conformal}. For Tibshirani's method, the only feasible approach is to sample a certain number of permutations and compute approximate weights based on them. Unfortunately, it leads to the very high variance of the weights which led us even to the bias in the mean coverage ($82.48$ on the ``Mix'' calibration data in our experiment on domain adaptation). For Plassier's method we observed a coverage of $92.28\pm5.97\%$ which is perfectly aligned with expectation to have increased variance compared to our method; see \Cref{fig:synth_data}.
When replicating the CIFAR10 experiment, sampling the permutations led us to even worse results and we only present the comparison with Plassier's method which again gives an expected increase in variance (see figures below).
\begin{figure}[h!]
  \centering
  \includegraphics[width=0.45\textwidth]{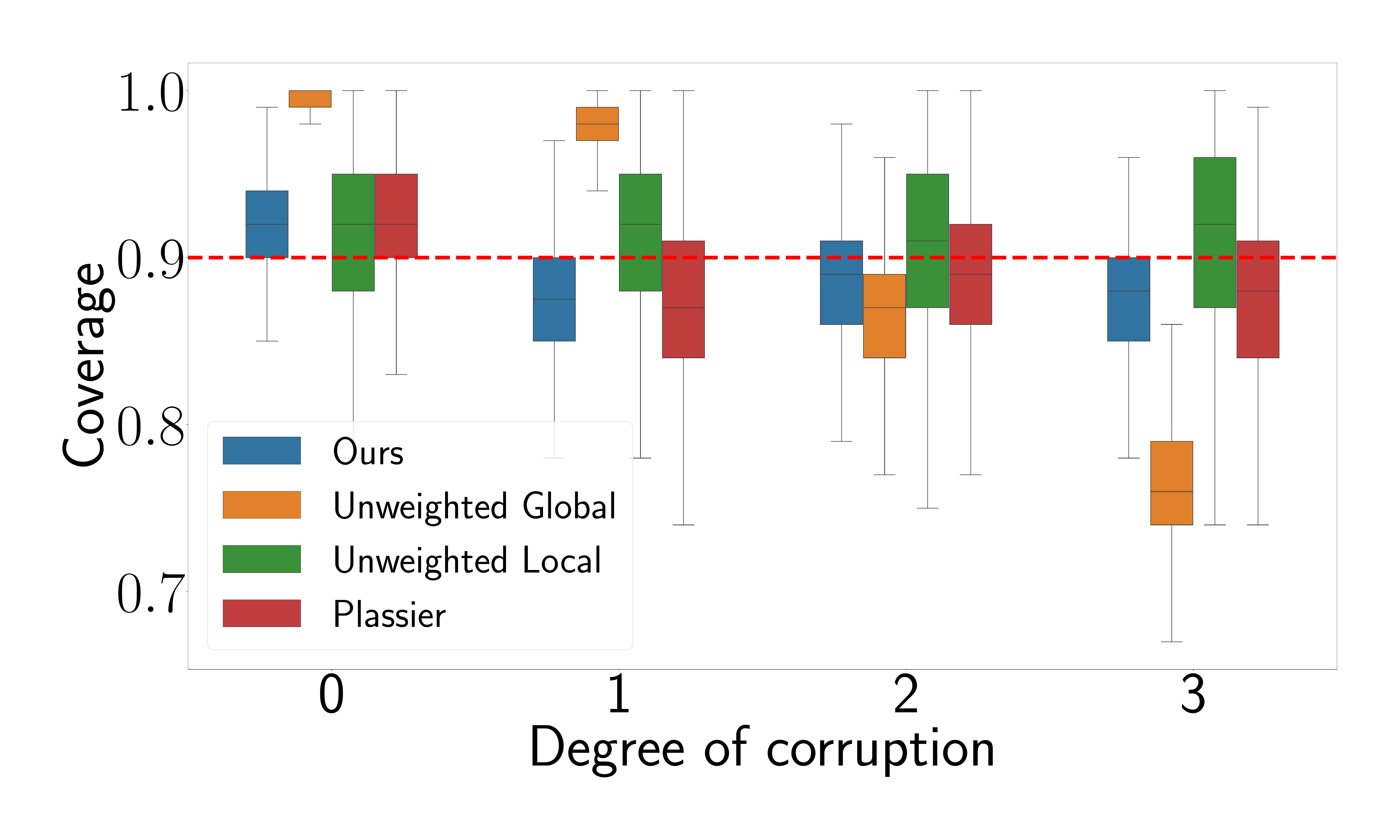}
  \includegraphics[width=0.45\textwidth]{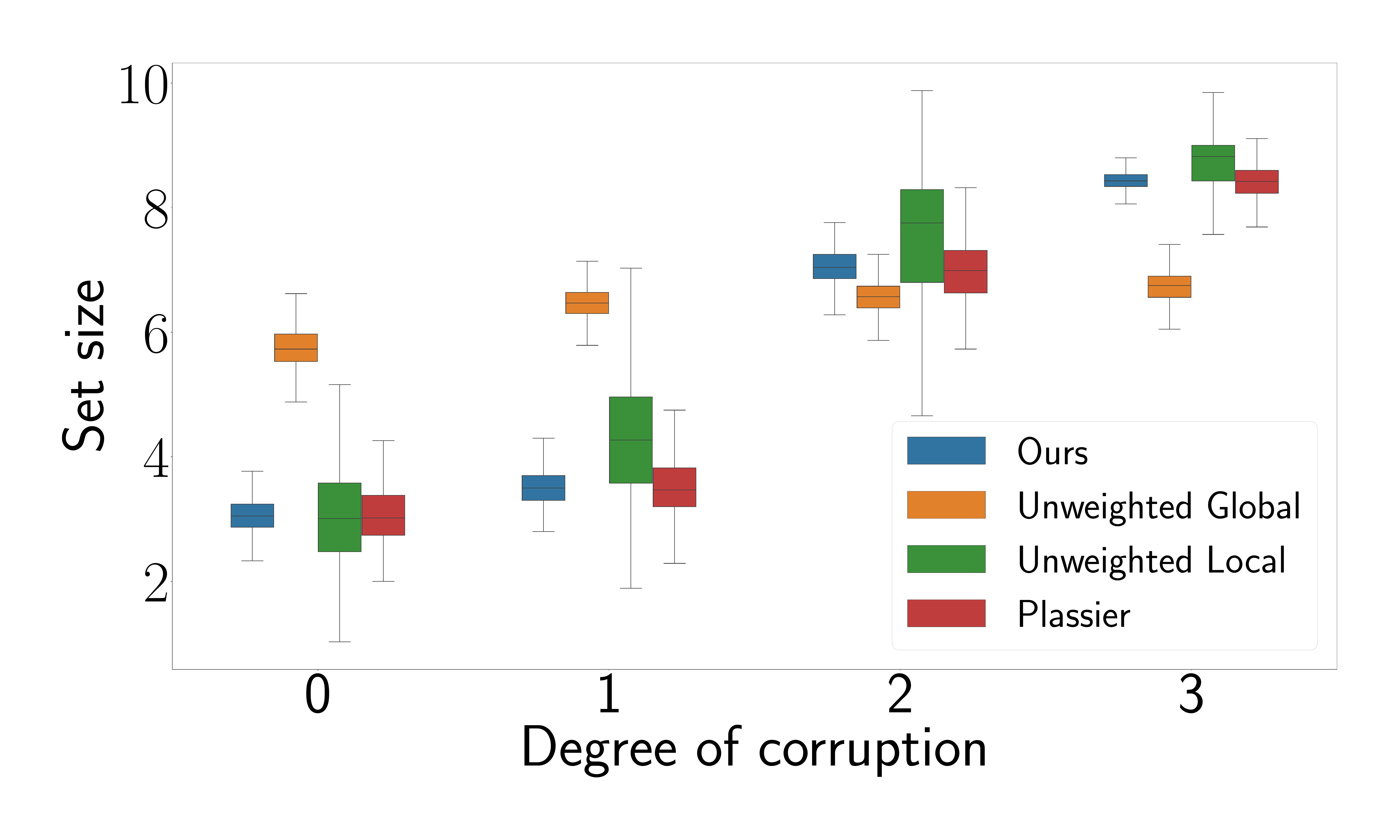}
  \caption{Coverage comparison on CIFAR-10. (Left) Empirical coverage in function of the data corruption level. (Right) Average set size of conformal
  prediction sets as a function of data corruption level.}
\end{figure}

\end{document}